\documentclass{article} 
\usepackage{arXiv_preamble}

\title{Mildly Overparameterized ReLU Networks on Orthogonal Data: Incremental Learning and \\ Implicit Bias}

\author{
{\normalfont James Town} \\
University of Warwick \\
\texttt{james.p.town@warwick.ac.uk}
\and
Etienne Boursier \\
INRIA \\
LMO, Universit{\'e} Paris-Saclay \\
\texttt{etienne.boursier@inria.fr}
\and
Ben Lewis \\
University of Warwick \\
\texttt{ben.lewis@warwick.ac.uk}
\and
Matthias Englert \\
University of Warwick \\
\texttt{m.englert@warwick.ac.uk}
\and
Ranko Lazi\'c\footnotemark[\value{footnote}] \\
University of Warwick \\
\texttt{r.s.lazic@warwick.ac.uk}
}

\begin{document}
\setcounter{tocdepth}{3}

\doparttoc 
\faketableofcontents 
\mtcsettitle{parttoc}{Table of contents}

\maketitle

\begin{abstract}
\looseness=-1
The successful training of neural networks hinges on the use of first order optimization methods, yet the theoretical characterization of these methods remains incomplete. This is especially true in settings with mild overparameterization. In this work, we study the gradient flow dynamics of two-layer ReLU networks from small initialization with orthogonal training data. We prove the limiting flow converges to a saddle-to-saddle jump process as the initialization scale tends to zero, revealing an incremental learning phenomenon in which a new neuron activates at each saddle. This analysis recovers the known result of \citet{dana2025convergence} that the network interpolates the training data with high probability as soon as $m \gtrsim \log(n)$, where $m$ is the network width and $n$ is the number of training samples. This incremental process characterization also allows us to derive a novel implicit bias result: the learned interpolator has a squared $\ell_2$-norm scaling as $\sqrt{n}$, which is within a constant factor of the minimal $\ell_2$-norm interpolator. More broadly, our work provides the first rigorous proof of an incremental learning process for ReLU networks, whilst suggesting mildly overparameterized networks can converge to interpolating solutions whose complexity is of the same order as that of the optimal interpolator.
\end{abstract}

\section{Introduction}

Despite the remarkable empirical successes of deep learning \citep{lipman2023flow,zagoruyko2016WideRestNet,Cheng2016recommender}, a precise theoretical understanding of the optimization dynamics of neural networks trained by first-order methods remains elusive. The \textit{Neural Tangent Kernel} (NTK) framework offered an initial answer, establishing convergence to global minima for overparameterized networks \citep{jacot2018neural,du2018gradient,arora2019fine}. However, NTK operates under large initialization assumptions that cause network features to remain nearly unchanged during training; a regime aptly termed \textit{lazy} by \citet{chizat2019lazy}, and now understood to be unrepresentative of practice \citep{fort2020deep}.

\looseness=-1
In the practically relevant \textit{small initialization} regime, training induces significant changes in network features, a phenomenon known as \textit{feature learning}. Theoretical analysis of this regime is considerably more challenging, as the resulting dynamics are highly non-convex. General convergence results exist for linear activation architectures \citep{pmlr-v80-arora18a,WoodworthGLMSGS20} and, in the infinite-width limit, for nonlinear networks \citep{chizat2018global,wojtowytsch2020convergence}. These results, however, require either smooth activations or infinite data \citep{Holzmuller2022inconsistency,BoursierF25}, and thus do not apply to the canonical case of two-layer ReLU networks.

\looseness=-1
Beyond convergence, a central motivation for studying the small initialization regime is its \textit{implicit bias}: gradient flow is known to favor sparse or low-complexity solutions that generalize well. For classification, two-layer networks trained to interpolation have been shown to converge to min-norm max-margin estimators \citep{LyuL20,chizat2020implicit,JiT20}, which coincide with sparse, narrow networks \citep{SafranVL22,boursierF23}. Regression is however arguably more representative of practical training dynamics, as characterizing the implicit bias requires describing the entire trajectory rather than only its asymptotic endpoint after interpolation is achieved. In the regression setting, diagonal linear networks favor sparse estimators \citep{WoodworthGLMSGS20}, while fully connected linear networks favor low-rank solutions \citep{saxe2014exact,arora2019implicit}. For two-layer ReLU networks, implicit bias in regression remains largely open, and constitutes the focus of this work.

\looseness=-1
Progress has been made by restricting to structured data settings. Among these, orthogonal data has emerged as a natural and tractable setting: \citet{BoursierPF22} established convergence of gradient flow to a minimal $\ell_2$-norm interpolator, but under an initialization condition requiring exponential overparameterization in the number of training samples. 
\citet{dana2025convergence} later showed that mild overparameterization suffices to guarantee convergence to an interpolating solution. However, their analysis leaves the norm of the reached interpolator uncharacterized, and understanding this implicit bias in the mildly overparameterized regime remained open until the present work.

\paragraph{Contributions.}
We prove a precise characterization of the implicit bias of gradient flow when training two-layer ReLU networks on orthogonal data in the mildly overparameterized regime, via a complete analysis of the training trajectory. 
Interestingly, our analysis reveals an \textit{incremental learning} phenomenon, illustrated in \Cref{fig:illustration_intro}: as initialization scale vanishes, the loss alternates between long idle plateaus and abrupt decreases, each corresponding to the activation of a new neuron. While incremental learning has been characterized for linear networks, our work provides the first such characterization for networks with nonlinear activations. An informal statement of our main result follows, where $m$ refers to the width of the network and $n$ the number of training samples.

\begin{figure}
\centering
\subfigure{%
\begin{minipage}[t]{7.6cm}
\centering
\includegraphics{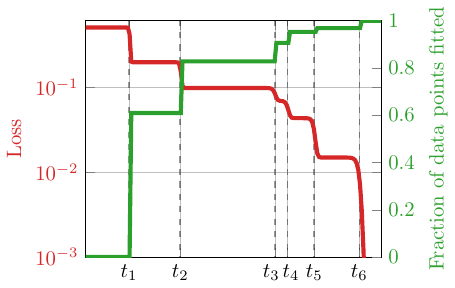}
\end{minipage}
}
\hfill
\subfigure{%
\begin{minipage}[t]{6.1cm}
\centering
\includegraphics{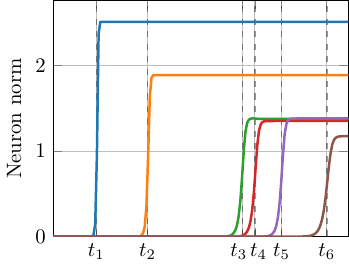}
\end{minipage}
}
\vspace{-1.5em}\caption{\label{fig:illustration_intro}We train a two-layer ReLU network of six neurons on 64 orthonormal data points in 64 dimensions. The initialization scale is small (experimental details in \Cref{sec:expe}). Left: During training, the loss (red) shows long periods of little change with sudden drops in between. Those drops correspond to an additional set of data points being fitted (green). The drops in the loss happen at the times independently predicted by \cref{alg:s.t.s}, which are indicated by the dashed lines. Right: The norms of the hidden layer of the six neurons during the same training.}
\end{figure}

\begin{thm}[Informal]\label{thm:informal}
In the regression setting with orthogonal training data, as initialization scale vanishes, and with high probability over the network initialization, the gradient flow trajectory over a two-layer ReLU network with $\log(n) \lesssim m \ll \exp\left(n \log (2)/2\right)$ converges to a piecewise-constant process: the iterates move successively from one saddle of the training loss to another, with each transition corresponding to the creation of a new non-zero neuron. The sequence of visited saddles and transition times can be computed via \Cref{alg:s.t.s} without any actual training of the network.
\end{thm}
\looseness=-1
Building on this algorithmic description of the limit dynamics, we derive an upper bound on the parameter $\ell_2$-norm of the solution found by gradient flow, showing it is within a constant factor of the minimal $\ell_2$-norm interpolator.

\subsection{Related work} \label{sec:related}

\paragraph{Incremental Learning.} \looseness=-1 \citet{arpit2017closer,kalimeris2019sgd} empirically demonstrated that standard optimization of deep learning models follows a \textit{simplicity bias}, whereby learned representations grow increasingly complex over the course of training. Separately, several works have shown that this complexity does not accumulate continuously, but rather incrementally: new \textit{data patterns} are absorbed one at a time \citep{rahaman2019spectral,charton2024learning}, giving rise to a characteristic training curve marked by long stagnation plateaus and sharp drops in the loss. The resulting training dynamics can thus be understood as a succession of jumps between saddle points of the loss landscape --- a phenomenon sometimes referred to as \textit{saddle-to-saddle dynamics} \citep{JacotGSHG21}.

\looseness=-1
Theoretically characterizing this incremental behavior is challenging because it is specific to the small initialization regime, where networks operate in the rich regime and the resulting dynamics are highly non-convex. For two-layer networks, \citet{kunin2025alternating} offered high-level intuitions for this regime: during each plateau, a dormant neuron is activated in a manner shaped by the local geometry near the saddle point, after which the activated neurons undergo rapid updates as training transitions to the next plateau. However, these intuitions fall short of a complete and rigorous characterization, which has so far been confined to specific linear activation models: for diagonal linear networks, \citet{berthier22,PesmeF23} precisely characterized the incremental learning dynamics, showing that the estimator's support grows one coordinate at a time; for linear networks and matrix sensing, \citet{gidel2019rank,JinLLDL23} established, under favorable data or initialization conditions, that the learned matrix undergoes incremental rank increases throughout training.

Our work advances this line of research by providing the first theoretical characterization of incremental learning under gradient flow for networks with non-linear activations, specifically two-layer ReLU networks trained on orthogonal data. In a related vein, \citet{abbe2023sgd} also studied incremental learning in two-layer networks with non-linear activations; however, their analysis is tailored to a specific activation function and optimization algorithm that depart from standard practice.

\paragraph{Training two-layer ReLU networks on orthogonal data.} \looseness=-1
Since a theoretical analysis of two-layer ReLU network training remains out of reach for general data distributions, many works have focused on specific toy settings, such as XOR data \citep{glasgow2024sgd}, positively correlated data \citep{wang2023understanding,chistikov2023learning,MinMV24}, or linearly separable data \citep{PhuongL21,lyu2021gradient}. Among these, we focus on the orthogonal data setting, first studied by \citet{BoursierPF22}, who showed that under small initialization, the parameters converge to a minimal $\ell_2$-norm interpolator. However, their analysis requires an initialization condition that is only satisfied in a heavily overparameterized regime, where the number of neurons~$m$ is exponential in the number of training points~$n$. \citet{dana2025convergence} later relaxed this assumption, showing that mild overparameterization $m \gtrsim \log(n)$ suffices to guarantee convergence to a global minimum of the training loss.

Yet the question of implicit bias in the mildly overparameterized regime remained open, as \citet{dana2025convergence} do not characterize the nature of the resulting estimator. Specifically, one may ask: \textit{does the network still converge to the minimal $\ell_2$-norm interpolator in the mildly overparameterized regime?} \citet{boursier2025benignity} answered this negatively, showing that the minimal $\ell_2$-norm interpolator is not reached with high probability as soon as $m \lesssim 2^{n}$. Taken together, these results leave a gap precisely in the mildly overparameterized regime $\log(n) \lesssim m \lesssim 2^{n}$, which our work fills by characterizing the $\ell_2$-norm of the interpolator obtained in this setting.

\paragraph{More parameters lead to simpler models.}
The generalization of overparameterized models remains an active area of research in deep learning theory. A phenomenon of particular interest is \textit{double descent} \citep{BelkinHMM19, Nakkiran2020Deep}, whereby overparameterization beyond the interpolation threshold leads to smaller test error. A likely driver of this behavior is that, as the number of parameters grows, the function represented at convergence becomes simpler \citep{wilson2025position}, exhibiting for example smaller norm \citep{BelkinHMM19}, fewer effective parameters \citep{maddox2020rethinkingparametercountingdeep}, or lower Kolmogorov complexity \citep{goldblum2024position}. 
Numerous theoretical analyses of this phenomenon have been proposed in recent years, but remain confined to linear regression on random features \citep{liao2020random, belkin2020two, hastie2022surprises, mei2022generalization,bach2024high}. Our work establishes such a tendency for two-layer ReLU networks trained on orthogonal data: mildly overparameterized networks converge to interpolators with slightly larger parameter $\ell_2$-norm than those found in the heavily overparameterized regime.

\section{Setting}\label{setup}

We consider a two-layer neural network with the ReLU activation, whose output is given by
\begin{equation}
    h_\theta(x) \coloneqq
  \textstyle\sum_{j = 1}^m a_j \, \sigma(w_j^\top x),
\end{equation} 
where $m$~is the network width, $\sigma(z) \coloneqq \max \{0, z\}$ is the ReLU function, and $\theta = (a, W)$ are the network parameters, which consist of output weights $a \in \mathbb{R}^m$ and hidden neurons $W^\top = (w_j)_{j = 1}^m \in \mathbb{R}^{d \times m}$. We consider a training dataset $(x_i, y_i)_{i = 1}^n$, which consists of inputs $x_i \in \mathbb{R}^d$ and outputs $y_i \in \mathbb{R}$, satisfying the following assumption.

\begin{ass}
\label{ass:data}
The training inputs are such that $x_i^\top x_{i'} = \iind{i = i'}$ for all $i, i' \in [n]$. 
\end{ass}

Although restrictive, the orthogonality assumption is a useful test bed for analyzing training dynamics, which can even sometimes be relaxed to general random high dimensional data \citep{dana2025convergence}. On the other hand, the norm constraint on the inputs implied by \Cref{ass:data} is solely used for sake of simplicity and could be easily alleviated. 
Training aims at minimizing the mean squared error loss over the training data,
\begin{equation}
    \mathcal{L}(\theta) \coloneqq \frac{1}{2 n} \textstyle\sum_{i = 1}^n (h_\theta(x_i) -y_i)^2.
\end{equation}
For analytical purpose, we consider optimizing this loss function by gradient flow, which approximates (stochastic) gradient descent as the learning rate vanishes to $0$:
\begin{equation}
    \dot{\theta}(t) \in -\partial \mathcal{L}(\theta(t)) \text{ for almost all } t \geq 0.\label{eq:gr.fl}
\end{equation}

Here, $\partial$~denotes the \citet{clarke1975generalized} subdifferential, which allows to define the gradient flow beyond differentiable functions. We direct the reader to \Cref{app:clarke} for further details on the Clarke subdifferential. Note that \Cref{ass:data,eq:gr.fl} imply that
\begin{equation} \label{eq:individual_alignment}
    \dot{w}_j^\top(t) \, x_i \in \frac{1}{n} (y_i - h_{\theta(t)}(x_i)) \, a_j(t) \, \partial \sigma(w_j^\top(t) \, x_i) .
\end{equation}
Hence, for all $i \in [n]$, all $j \in [m]$, and almost all $t \geq 0$, if a hidden neuron component $w_j^\top(t) \, x_i$ is negative then it must be constant for all $t' \geq t$.  Empirically, the same is true with ``negative'' replaced by ``non-positive'',\footnote{E.g., in PyTorch, the derivative of the ReLU non-linearity at~$0$ is computed as~$0$.} so \Cref{ass:dead} below restricts our attention to realistic training trajectories, which have that property.
 \begin{ass}
\label{ass:dead}
For all $i \in [n]$ and all $j \in [m]$, if $w_j^\top(t) \, x_i = 0$ for some $t \geq 0$, then $w_j^\top(t') \, x_i = 0$ for all $t' > t$.
\end{ass}
Note that whilst it is guaranteed that a solution of \Cref{eq:gr.fl} exists locally, it may not be unique. With \cref{ass:dead}, along every solution of the subdifferential inclusion \cref{eq:gr.fl}, the left-hand side $\dot{\theta}(t)$ exists --- possibly except at the finitely many time points that are the minima of non-empty sets $\{t \geq 0 \;\vert\; w_j^\top(t) \, x_i = 0\}$ --- and is equal to the element of the right-hand side $-\partial L(\theta(t))$ that is obtained by using only~$0$ as the subderivative of the ReLU function~$\sigma$ at~$0$.  Hence for simplicity of presentation, we fix the latter choice, i.e., in the remainder of the paper we work with $\iind{z > 0}$ instead of $\partial \sigma(z)$ and we may use the gradient notation to mean the subgradient obtained with that choice. \citet{BoursierPF22} also proved that any other constant definition of the subgradient fails to yield global solutions of \Cref{eq:gr.fl} in presence of orthogonal data. Finally, we note that \citet{bertoin2021numerical} suggest that this fixed choice tends to be most efficient in practice. 
We consider balanced initialization of the parameters, which is specified by an initialization scale $\alpha$.
\begin{ass}
\label{ass:init}
For all $j \in [m]$ we have $a_j(0) = \alpha \, s_j$ and $w_j(0) = \alpha \, u_j$, where
$s_j
 \stackrel{\text{\textnormal{iid}}}{\sim}
 \mathcal{U}(\{\pm 1\})$
and
$u_j
 \stackrel{\text{\textnormal{iid}}}{\sim}
 \mathcal{U}(\mathbb{S}^{d - 1})$.
\end{ass}
Whilst the assumption of strictly balanced initialization is partly to simplify the analysis, the work of \citet{AzulayMNWSGS21} suggests that balanced initializations encourage bias towards the rich regime. \citet{AzulayMNWSGS21} also note that the commonly used Xavier initialization \citep{GlorotB10} is approximately balanced. 
The balanced initialization simplifies the analysis of the training dynamics, as this weight balancedness will be preserved along training, thanks to the following result.
\begin{prop}[{\citealt[Theorem~2.1]{DuHL18}}]
\label{pr:bal}
For all $j \in [m]$ and all $t \geq 0$, we have $|a_j(t)| = \|w_j(t)\|$ and $\sgn(a_j(t)) = s_j$.
\end{prop}
\Cref{pr:bal} also implies that the output layer signs remain constant during training. This allows us to treat separately the data associated with positive or negative labels, as well as the neurons with positive or negative output signs. To this end, we denote for $s \in \{\pm 1\}$, the set $I_s \coloneqq \{i \in [n] \;\vert\; \sgn(y_i) = s\} \label{Is_defn}$ of labels with sign $s$.

\paragraph{Neuron dynamics.} Denote the normalized hidden-layer weights by $ \overline{w}_j := w_j / \lVert w_j \rVert $ (in general we use bars to denote normalized vectors and $\lVert \cdot \rVert$ to denote the $\ell_2$-norm) and define the dynamical vector $\fD_j = - \frac{1}{n} \sum_{i \mid w_j^\top x_i > 0} (h_{\theta}(x_i) - y_i) x_i \label{fd_defn}$. \Cref{eq:gr.fl,pr:bal,ass:dead} imply the radial and tangential dynamics of neurons are given for almost all $t\geq 0$ by
\begin{align}
    \dot{a}_j(t) = \fD_j^\top w_j \quad \text{and} \quad
    \dot{\overline{w}}_j(t) = s_j \left( \text{I}_d -  \overline{w}_j \overline{w}_j^\top \right) \fD_j \label{eq:gf_tangential}.
\end{align}
\paragraph{Acceleration of time.} Our goal is to study the limiting dynamics of the parameters trajectory as the initialization scale $\alpha$ goes to $0$. When $\alpha\to 0$, escaping saddle points of the trajectory --- starting with the initial one where all parameters are zero --- takes an arbitrarily long time. Similarly to \citet{berthier22,PesmeF23}, obtaining a non-trivial limit dynamics requires a proper rescaling of time, which is given by the time acceleration $t \mapsto \log(1/\alpha) t$. We denote the accelerated parameters by $\widetilde{\theta}^\alpha(t) = \theta( \log(1/\alpha) t)$. In the following, we omit the superscript as the acceleration by a factor of $\log(1/\alpha)$ is implicit in the tilde notation. The flow defining this rescaled process then satisfies $\dot{\widetilde{\theta}}(t) \in - \log(1/\alpha) \partial \mathcal{L}(\widetilde{\theta}(t))$ for almost all $t \geq 0$.

\section{Limit process}

In this section, we provide intuitions for deriving the limit process of the training trajectory as $\alpha\to 0$. We here only provide high level insights, and proving the exact convergence results claimed here requires much more detailed and intricate analyses, that are given in the complete proof of \cref{th:s.t.s}.

We build the saddles visited by the limit process and the corresponding jump times recursively. 
Assume that $\widetilde{\theta}(t)\overset{\alpha\to 0}{\longrightarrow} \theta^\circ(t)$ for almost every $t \in [0, t_k]$, where $t_k$ is a jump time, and that the limit process has reached the saddle point $\theta^{(k)}$ after the $k$-th jump, i.e., $\theta^\circ(t_k) = \theta^{(k)}$.  The limit process $\theta^\circ(t)$ will then be constant, equal to the saddle $\theta^{(k)}$ on an interval $[t_k, t_{k+1})$, and will jump to a new saddle $\theta^{(k+1)}$ at a jump time $t_{k+1}$ to be determined. 
The initial saddle point is given by $\theta^{(0)}=\mathbf{0}$. Moreover, we will also show recursively that the limit process jumps from a saddle point to another by modifying a single neuron $(w_{j}^\circ, a^\circ_j)$, which is zero before jump time $t_{k+1}$, and will reach its final, non-zero value after the jump. Also, all the training points along which the neuron $j$ is active will be fitted after the jump. During that jump, the values of all other neurons remain unchanged. In consequence, the saddle point $\theta^{(k)}$ visited by the limit process includes two types of neurons:
\begin{itemize}
    \item unfitted neurons, which are equal to zero, i.e., $(w_j^\circ, a_j^\circ) = \mathbf{0}$;
    \item fitted neurons, which perfectly fit the training points along which they are active: \\ $\iind{w_j^{\circ \ \top} x_i > 0} (h_{\theta^{(k)}}(x_i) - y_i)^2 = 0$ for any $i\in[n]$. These neurons will not move further during the remainder of training.
\end{itemize}

Describing the limit process then calls for describing how and when it jumps from a saddle point $\theta^{(k)}$ to the next one $\theta^{(k+1)}$. For that and following the heuristics developed by \citet{MaennelBG18,kunin2025alternating}, we will decompose the tangential and radial dynamics of the parameters, using the reparameterization $(\widetilde{\overline{w}}_j,\widetilde{\ell}_j)$ such that $\widetilde{w}_j = \alpha^{\widetilde{\ell}_j}\widetilde{\overline{w}}_j$. This reparameterization captures the right quantities: while $\widetilde{w}_j$ is nearly zero for unfitted neurons, its renormalized counterpart $\widetilde{\overline{w}}_j$ tracks the tangential movement of the neuron, which is key to understanding the dynamics in the vicinity of each saddle. 
The dynamics satisfied by these parameters can then be written as follows:
\begin{equation}
    \frac{\df }{\df t} \widetilde{\overline{w}}_j= \log(1/\alpha) s_j  \left( \text{I}_d - \widetilde{\overline{w}}_j \widetilde{\overline{w}}_j^\top \right)\fD_j  ; \qquad \frac{\df }{\df t} \widetilde{\ell}_j = s_j \widetilde{\overline{w}}_j^\top \fD_j
\end{equation}
Around a saddle point $\theta^{(k)}$, note that $\fD_j\approx 0$ for already fitted neurons, which thus do not move anymore. 
For unfitted neurons, the tangential movement can be seen as gradient ascent/descent of some alignment restricted on the sphere, similarly to observations made by \citet{BoursierPF22, MinMV24, kumar2024directional}. Because of its $\log(1/\alpha)$ scaling, this tangential movement thus occurs instantly as $\alpha\to 0$, such that there exist some vectors $D_j^{(k)}$ and for $t\in (t_k, t_{k+1})$
\begin{equation}
    \widetilde{\overline{w}}_j(t) \overset{\alpha\to 0}{\longrightarrow} s_j \overline{D}_j^{(k)}.
\end{equation}
These vectors $\overline{D}_j^{(k)}$ correspond to the (normalized) critical points reached by the gradient ascent/descent mentioned above. In the case of orthogonal data, they are simple to compute and are actually given for unfitted neurons by
\label{Djk_defn}
\begin{equation}
D_j^{(k)} = \frac{1}{n}\sum_{i\in S_U^{(k)}} \iind{w_j(0)^\top x_i>0} \iind{\sgn{y_i}=s_j} y_i x_i, 
\end{equation}
where $S_U^{(k)} \label{Suk_defn}$ is the set of training points that are not yet fitted by the network (a rigorous definition of all quantities is given in \Cref{alg:s.t.s} below). 
Since we have $\widetilde{\overline{w}}_j(t) \approx s_j \overline{D}_j^{(k)}$, note that the time derivative of the norm exponent then satisfies on $(t_k, t_{k+1})$:
\begin{equation}
    \frac{\df}{\df t}\widetilde{\ell}_j \approx \left\|D_j^{(k)}\right\|_2.
\end{equation}
Consequently, $\widetilde{\ell}_j$ converges to a piecewise affine process $\ell_j^\circ$ that starts at $-1$, owing to our initialization choice, and is defined recursively by:
\label{ellj_defn}
\begin{equation}
    \ell_j^\circ(0) = -1; \qquad \ell_j^\circ(t)= \ell_j^\circ(t_k) + (t-t_k)\left\| D_j^{(k)} \right\|_2 \quad \text{for } t\in[t_k, t_{k+1}].
\end{equation}
Moreover, the next jump happens as soon as one unfitted neuron becomes non-zero (in the limit $\alpha\to 0$). With our reparameterization, this occurs as soon as $\ell_j^\circ(t)\geq 0$ for some unfitted neuron. 
The jump time $t_{k+1}$ can then be deduced as the first time for which the exponent $\ell_j^\circ$ of some unfitted neuron becomes zero, and the corresponding neuron is the one to be updated at jump time $t_{k+1}$. In other words, we have the following:
\begin{itemize}
    \item $j_\star^{(k)} = \argmax_{j}  \ell_j^{\circ}(t_k)/\left\|D_j^{(k)}\right\|_2$ is the unfitted neuron that will grow in norm; \label{jstar_defn}
    \item $t_{k+1}= t_k - \ell_{j_\star^{(k)}}^{\circ}(t_k)/\left\|D_{j_\star^{(k)}}^{(k)}\right\|_2$ is the jump time at which $w_{j_\star^{(k)}}^\circ$ becomes non-zero. \label{tk_defn}
\end{itemize}
At jump time $t_{k+1}$, the process then leaves the saddle $\theta^{(k)}$, and due to the time acceleration $\log(1/\alpha)$, it instantaneously reaches a new critical point of the training loss. 
Moreover, during this instantaneous jump, only the neuron $(w_{j_\star^{(k)}}^\circ, a_{j_\star^{(k)}}^\circ)$ of the limit process undergoes a change: it jumps from $\mathbf{0}$ to the value minimizing the new training loss over this single neuron, which is given by
\begin{equation}
    w_{j_\star^{(k)}}^\circ(t) = s_j \frac{n D_{j_\star^{(k)}}^{(k)}}{\sqrt{\left\|n D_{j_\star^{(k)}}^{(k)}\right\|_2}}; \quad  a_{j_\star^{(k)}}^\circ(t) = s_j \sqrt{\left\|n D_{j_\star^{(k)}}^{(k)}\right\|_2} \qquad \text{for }t\geq t_{k+1}.
\end{equation}

\begin{algorithm}[t]
\caption{Construction of the limit process.}
\label{alg:s.t.s}
\begin{algorithmic}[1]
\State \textbf{Input:} data $(x_i,y_i)_{i\in[n]}$ and initial weights $(w_j(0), a_j(0))_{j\in[m]}$.
\State \textit{\color{commentgreen} $\triangleright$ Initialization}
\State $p \gets 0$;\quad  $t_0=0$; \quad $N_U^{(0)} = [m]$; \quad $S_U^{(0)} = [n]$
\State $S_j = \lbrace i \in [n] \mid w_j(0)^\top x_i > 0 \text{ and } \sgn(y_i)=\sgn(a_j(0))\rbrace$; \ and \
 $\ell^\circ_j(0)=-1$ for any $j\in[m]$
\State \textit{\color{commentgreen} $\triangleright$ Recursive definition of the jumps}
\While{$\exists j\in N_U^{(p)}, S_U^{(p)}\cap S_j \neq \emptyset$,}
    \State for any $j\in[m]$, \quad $D_j^{(p)} = \begin{cases}\frac{1}{n}\sum_{i\in S_U^{(p)}\cap S_j}  y_i x_i \quad&\text{ if }j\in N_U^{(p)} \\
     D_j^{(p-1)} &\text { otherwise} \end{cases}$
    \State \label{line:argmin} Let $j_\star^{(p)} \in \argmax_{j\in N_U^{(p)}}  \ell_j^{\circ}(t_p)/\left\|D_j^{(p)}\right\|_2$
    \Comment{neuron to activate}
    \State $t_{p+1} = t_p - \ell_{j_\star^{(p)}}^{\circ}(t_p)/\left\|D_{j_\star^{(p)}}^{(p)}\right\|_2 $
    \Comment{jump time}
    \State $\ell^\circ_j(t_{p+1}) = \min \left( 0,\ell_j^\circ(t_p) + (t_{p+1} -t_p)\left\|D_j^{(p)}\right\|_2 \right) $ for any $j\in[m]$
    \Comment{norm exponent}
    \State  $N_U^{(p+1)} = N_U^{(p)} \setminus \{ j_\star^{(p)}\}$ \ and  \ $S_U^{(p+1)} =  S_U^{(p)}\setminus S_{j_\star^{(p)}} \label{Suk_Nuk_defn} \nonumber $
    \Comment{unfitted neurons and data}
    \State $p\gets p+1$;
\EndWhile
\State \textit{\color{commentgreen} $\triangleright$ Definition of limit process}
\State Let $t_{p+1}=+\infty$ by convention
\State For any $k\in[p]$ and $t\in[t_k, t_{k+1})$, define $\theta^\circ(t)=(w_j^\circ(t), a_j^\circ(t))_{j\in[m]}$ as
    \begin{align}
  &  \text{for any } j\in N_U^{(k)}, \qquad (w_j^\circ(t), a_j^\circ(t)) = \mathbf{0};\\
   & \text{for any } j\not\in N_U^{(k)}, \qquad    w_{j}^\circ(t) = s_j \frac{n D_{j}^{(k)}}{\sqrt{\left\|n D_{j}^{(k)}\right\|_2}} \quad \text{and} \quad  a_{j}^\circ(t) = s_j \sqrt{\left\|n D_{j}^{(k)}\right\|_2}. 
\end{align}
\vspace{0.5em}
\State \textbf{Output:} limit process $(\theta^\circ(t))_{t\geq 0}$ and jump times $(t_0, \ldots, t_{p+1})$.
\end{algorithmic}
\end{algorithm}

Thanks to these high level insights, one can define formally the limit process $(\theta^\circ(t))_{t\geq 0}$ in \cref{alg:s.t.s} below. 
Importantly, the \texttt{while} loop in \cref{alg:s.t.s} eventually stops after a finite number of iterations~$p$, such that $S_U^{(p)} \cap S_j = \emptyset$ for all $j\in N_U^{(p)}$. In the following, \Cref{ass:A.j.star} will guarantee that we actually stop when $S_U^{(p)} = \emptyset$, which corresponds to the training loss being zero.

\section{Saddle-to-saddle dynamics} \label{sec:convergence}
\label{s:dyn}

In this section, we present our main result. We first require some assumptions to ensure that the network interpolates at convergence, and avoid degenerate cases in the definition of the limit process. 


\begin{ass}\label{ass:A.j.star}
Let $A \in \{0,1\}^{n\times m}$ be the matrix such that $A_{i,j} =\iind{w_j^\top(0) \, x_i > 0 \text{ and } i \in I_{s_j}} \label{Aij_defn}$. 
\begin{enumerate}[(i), leftmargin=30pt, topsep=0pt, parsep=0pt]
\item
\label{ass:A.j.star:A}
Each row of the mask matrix $A$ is non-zero, each column is non-zero, and its columns are pairwise distinct.
\item
\label{ass:A.j.star:j.star}
The training outputs $(y_i)_{i = 1}^n$ are such that they are non-zero, and in each iteration defining the limit process, the $\argmax$ in \cref{alg:s.t.s} \cref{line:argmin} is unique.
\end{enumerate}
\end{ass}

We believe our main result could be extended to the case that columns are not pairwise distinct, at the expense of a more refined analysis, by merging neurons into groups with matching initial activation pattern in a similar manner to \cite{BoursierPF22}: neurons which share activation patterns align and grow in the same directions throughout training, and the sum of their squared norms is equal to the squared norm achieved by a single neuron with this activation pattern. In this way, neurons sharing activation patterns form rank-1 sub-networks, so \Cref{alg:s.t.s} may be applied considering an aggregated collection of neurons with matching initial activation pattern as one neuron. 

The requirement that each row be non-zero ensures loss minimization. The requirement that each column be non-zero ensures that for each neuron, there exists a direction in which it can grow at initialization to reduce the loss, otherwise such neurons remain small throughout training and can be ignored.
\Cref{pr:alg_assumptions} below states that \Cref{ass:A.j.star} holds with high probability in the mild overparameterization regime $\frac{\log(n)}{\log(4/3)} \lesssim m \ll \exp\left(n \log (2)/2\right)$.

\begin{prop} \label{pr:alg_assumptions}Let $n_+ = |I_+|$\, , $n_- = |I_-|$ and consider the initialization scheme of \cref{ass:init}. \begin{enumerate}[(i), topsep=0pt, parsep=0pt]
\item \label{pr:alg_assumptions_1}
 \cref{ass:A.j.star}~\ref{ass:A.j.star:A} holds with probability at least
$1 - n \, \big(\frac{3}{4}\big)^m
    - \frac{m(m+3)}{2} \big(\frac{1}{2}\big)^{\min(n_+,n_-) + 1}$.
\item \label{pr:alg_assumptions_2}
When \cref{ass:A.j.star}~\ref{ass:A.j.star:A} holds, the set of tuples $(y_i)_{i = 1}^n$ excluded by \cref{ass:A.j.star}~\ref{ass:A.j.star:j.star} has Lebesgue measure zero in $\mathbb{R}^n$.
\end{enumerate}
\end{prop}

The following result rigorously formulates the sense in which the gradient flow defined by \cref{eq:gr.fl} converges, as the initialization scale tends to zero, to the limit process $\theta^\circ$.

\begin{thm}
\label{th:s.t.s}
Under \cref{ass:data,,ass:init,,ass:dead,,ass:A.j.star}, as the initialization scale~$\alpha$ tends to~$0$, the accelerated flow~$\widetilde{\theta}^\alpha$ converges to~$\theta^\circ$, uniformly on every compact subset of $\mathbb{R}_{\geq 0} \setminus \{t_0, t_1, \dots, t_p\}$, where $\theta^\circ$ and $(t_0, t_1, \dots, t_p)$ are respectively the limit process and jump times defined in \cref{alg:s.t.s}.
\end{thm}
\looseness=-1
\Cref{th:s.t.s} shows that as the initialization scale decreases, the trajectory of the parameters approximates a jump process between different saddles of the loss function. As each saddle increases the rank of the network by one from the previous saddle, this result shows an example of incremental feature learning. This provides the first complete theoretical characterization of a saddle-to-saddle process in a neural network with realistic assumptions on the architecture and overparameterization. 
\Cref{fig:illustration_intro} illustrates this convergence towards the incremental limit process $\theta^\circ$ and additional experiments are provided in \cref{app:expe}

Instead of \cref{ass:A.j.star}, the analysis of \citet{BoursierPF22} relies on the assumption that there exist at least two neurons $j_+, j_-$ such that $A_{i,j_+}=1$ for all $i\in I_+$ and $A_{i,j_-}=1$ for all $i\in I_-$. In this case, \cref{alg:s.t.s} reduces to just two jumps: neuron $j_+$ grows to fit all positive labels, and neuron $j_-$ grows to fit all negative labels. Moreover, as explained in \cref{app:autonomous}, these two jumps are \textit{non-interfering}: the neurons fitting positive and negative labels evolve independently throughout training. Under mild overparameterization, however, the process is considerably more intricate: fitting positive (resp.\ negative) labels requires multiple neurons and hence multiple jumps, which may interact with one another.

Our proof therefore requires a substantially finer analysis of the dynamics than that of \citet{BoursierPF22}. The two main additional difficulties are the following. First, the norm growth of a given neuron exhibits memory across multiple saddles, which is controlled by the term $\ell^\circ_j$. Second, when a new neuron is activated, it may still be positively correlated with already fitted samples $x_i$, and one must ensure that its activation does not cause the loss on these samples to increase significantly.

\paragraph{Sketch of Proof.} The result is proven by an inductive argument on the number of saddles traversed by the flow. The schematic for the proof is depicted in \Cref{fig:frise} and consists of splitting the dynamics into phases around the jump from saddle~$k$ to saddle~$k+1$, where $\varepsilon$~and~$\varepsilon'$ are parameters of the proof which may be selected arbitrarily small, giving the desired uniform convergence.

\begin{itemize}[leftmargin=10pt, topsep=0pt, itemsep=5pt, parsep=0pt]
    \item \underline{Slow growth phase:} For $\kappa \in (t_k, t_{k+1})$, $\tau_1$ is the first time that $\big\lVert \widetilde{w}_{j_\star^{(k)}} \big\rVert_2 = \alpha^\varepsilon$. We find that for $t \in [\kappa, \tau_1]$ all neurons remain aligned in the directions $D_j^{(k)}$. Hence, $\widetilde{\ell}_j$ increases towards $0$ for all $j \in N_U^{(k)}$, but $\widetilde{\theta}$ does not change significantly and remains near saddle $k$ until any $\kappa'\in (\tau_1, t_{k+1})$.
    \item \underline{Rapid fitting phase:} $\tau_3$ is the first time that $\big\lVert \widetilde{w}_{j_\star^{(k)}} \big\rVert^2_2 = n \big\lVert D_{j_\star^{(k)}}^{(k)} \big\rVert_2 - \alpha^{\varepsilon'/4} $ for a small enough $\varepsilon'$. We show $\tau_3-\tau_1 = \bigO(\varepsilon, \varepsilon')$ through checking $\widetilde{w}_{j_\star^{(k)}}$ continues to grow at an exponential rate $\big\lVert D_{j_\star^{(k)}}^{(k)} \big\rVert_2$ by neither realigning in $\text{span} ( x_i \mid i \in S_F^{(k)} )$. 
    To show this we employ two techniques: first, we split the phase in two, defining $\tau_2$ as the first time  $ \big\lVert \widetilde{w}_{j_\star^{(k)}} \big\rVert_2 = \eta$ for some convenient $\eta$, as having a stronger bound on neuron $j_\star^{(k)}$ helps control the loss on fitted data; second, we estimate how much neuron $j_\star^{(k)}$ realigns in $\text{span} ( x_i \mid i \in S_F^{(k)} )$ in terms of the loss on this data, and then in turn show $w_{j_\star^{(k)}}$ remains aligned in direction $D_{j_\star^{(k)}}^{(k)}$ on $[\tau_1, \tau_3]$.
    \item \underline{Realignment and slow growth:} The unfitted neurons $j \in N_U^{(k+1)}$ respond to the change in gradients due to a new set of data being fitted by realigning rapidly in direction $D_j^{(k+1)}$ by some time $\tau_4$, after which time this alignment remains stable until any time $\kappa'' \in (t_{k+1}, t_{k+2})$.
\end{itemize} 

\begin{figure}
    \centering
\begin{tikzpicture}[xscale=1.375]\usetikzlibrary{patterns}
  \def\h{1}
  \def\b{0.725}
  \def\o{0.25}
  {\small
  \draw[thick,blue!80!black,fill=blue!90!black,opacity=0.15]
    (0.75,0) rectangle (3.0,\h);
  \node at (1.875,0.5*\h) {\small Slow growth};
    \draw[thick,orange!80!black,fill=green!90!black,opacity=0.15]
    (3.0,0) rectangle (5.0,\h);
    \node at (4.0,0.5*\h) {\small Fast growth};
    \draw[thick,pink!80!black,fill=pink!90!black,opacity=0.3]
    (5.0,0) rectangle (6.75,\h);
    \node at (5.875,0.5*\h) {\small Realignment};
    \draw[thick,blue!80!black,fill=blue!90!black,opacity=0.15]
    (6.75,0) rectangle (8.5,\h);
    \node at (7.65,0.5*\h) {\small Slow growth};
  
   \draw[->,line width=2] (0,0) -- (9.5,0) node[below right]{} ;
    \draw (0.5 cm,2pt) -- (0.5 cm,-2pt) node[anchor=north] {{\color{black}$t_k$}};
   \draw (0.75 cm,2pt) -- (0.75 cm,-2pt) node[anchor=north] {{\color{black}$\kappa\mathstrut$}};
   \draw (2.25 cm,2pt) -- (2.25 cm,-2pt) node[anchor=north] {{\color{black}$\kappa'$}};
   \draw (3 cm,2pt) -- (3 cm,-2pt) node[anchor=north] {$\tau_1$};
   \draw (3.75 cm,2pt) -- (3.75 cm,-2pt) node[anchor=north] {$\tau_2$};
   \draw (4.25 cm,2pt) -- (4.25 cm,-2pt) node[anchor=north] {{\color{black}$t_{k+1}$}};
   \draw (5.0 cm,2pt) -- (5.0 cm,-2pt) node[anchor=north] {$\tau_3$};
   \draw (6.75 cm,2pt) -- (6.75 cm,-2pt) node[anchor=north] {$\tau_4$};
   \draw (8.5 cm,2pt) -- (8.5 cm,-2pt) node[anchor=north] {{\color{black}$\kappa''$}};
   \draw (9.0 cm,2pt) -- (9.0 cm,-2pt) node[anchor=north]  {$t_{k+2}$};
   \draw[<->,line width=1.5, color=black] (0.75 cm, 1.2*\h cm) -- (3.0 cm, 1.2*\h cm) node[above, align = center, color=black]  {\hspace{-3.0cm} \small $\Theta(1)$};
   \draw[<->,line width=1.5, color=black] (3.0 cm, 1.2*\h cm) -- (3.75 cm, 1.2*\h cm) node[above, align = center, color=black]  {\hspace{-1.2cm} \small $\bigO(\varepsilon)$};
   \draw[<->,line width=1.5, color=black] (3.75 cm, 1.2*\h cm) -- (5.0 cm, 1.2*\h cm) node[above, align = center, color=black]  {\hspace{-1.75cm} \small $\bigO(\varepsilon')$};
   \draw[<->,line width=1.5, color=black] (5.0 cm, 1.2*\h cm) -- (6.75 cm, 1.2*\h cm) node[above, align = center, color=black]  {\hspace{-2.5cm} \small $\bigO(\varepsilon)$};
   \draw[<->,line width=1.5, color=black] (6.75 cm, 1.2*\h cm) -- (8.5 cm, 1.2*\h cm) node[above, align = center, color=black]  {\hspace{-2.375cm} \small $\Theta(1)$};
   }

\end{tikzpicture}
    \caption{Timeline of the training dynamics about the jump from saddle $k$ to saddle $k+1$. The quantities $\varepsilon$ and $\varepsilon'$ are  parameters of the proof of \Cref{th:s.t.s}, $\tau_1$ is the beginning of the fast growth of neuron $j_\star^{(k)}$, $\tau_3$ is the time by which neuron $j_\star^{(k)}$ has grown to its final norm, and $\tau_4$ is the time by which neurons in the set $N_U^{(k+1)}$ have realigned. The quantities $\kappa, \kappa'$, and $\kappa''$ can be viewed as fixed for the proof of convergence.}
    \label{fig:frise}
\end{figure}
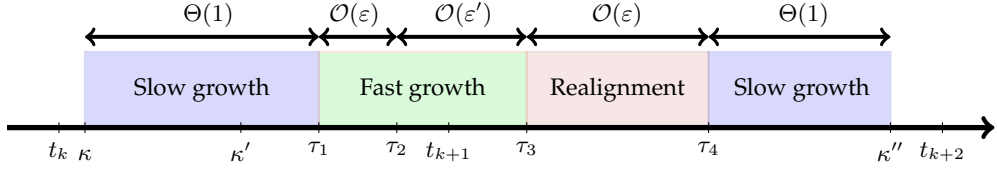

\section{Implicit bias} \label{sec:implicit_bias}
\label{s:i.b}

\citet{BoursierPF22} proved that gradient flow converges to a minimum $\ell_2$-norm interpolator $\theta_\opt$ as the initialization scale tends to zero, which is equivalent to the sum of two neurons that respectively fit the positive and negative labels. However, this norm optimality relies on a specific initialization condition, that assumes that at least one neuron is active along all positively labeled data and another neuron is active along all negatively labeled data. Such a condition only realistically holds in a heavily overparameterized regime with $m\gtrsim 2^n$.  
\citet{boursier2025benignity} then proved in a later work that with high probability in the mild overparameterization regime ($m\lesssim 2^n$), gradient flow does not converge towards a minimal $\ell_2$-norm interpolator as the initialization scales goes to $0$. 
We yet prove in this section that with mild overparameterization, \Cref{alg:s.t.s} produces with high probability an interpolator that is only larger in $\ell_2$-norm than the optimal one by a constant multiplicative factor. 


We denote in the following by~$\theta_{\pred}$ the limiting parameters $\theta^\circ(t)$ for any $t\geq t_p$---the definition is independent of the choice of  $t>t_p$, thanks to \Cref{th:s.t.s}---and by~$\theta_{\opt}$ any minimum $\ell_2$-norm interpolator. The squared $\ell_2$-norms of these parameters are given by

\begin{align} \label{eq:optnorm}
    \frac{1}{2}\lVert \theta_{\opt} \rVert_2^2 = \sqrt{\sum_{i \, \mid \, y_i>0} y_i^2} +\sqrt{\sum_{i \, \mid \, y_i<0} y_i^2}\qquad \text{and} \qquad \frac{1}{2}\lVert \theta_{\pred} \rVert_2^2 = \sum_{j \in [m]} n \left\lVert D_j^{(p)} \right\rVert_2. 
\end{align}
where the last jump of the limit process $p$ and $D_j^{(p)}$ are all defined in \Cref{alg:s.t.s}.
\Cref{thm:implicit_bias} below upper bounds the norm of the learned interpolator in the mild overparameterization regime as the initialization scale goes to~$0$.

\begin{thm} \label{thm:implicit_bias}
Let $n_+ = |I_+|$, $n_- = |I_-|$, and $\npm=\min(n_+,n_-)$.  Under \cref{ass:data,ass:init}, if $\log(m) = \lito(\npm^{1/8}/\log(\npm))$, then the final output of \Cref{alg:s.t.s} satisfies
    \begin{equation}
       \frac{1}{2}\lVert \theta_{\pred} \rVert_2^2\leq 5 \left( \sqrt{n_+} + \sqrt{n_-} \right)\max_{i\in[n]} |y_i|,
    \end{equation}
    with probability at least $1 - \exp(-\Omega(\npm^{1/8}))$ over the network initialization.
\end{thm}
\Cref{thm:implicit_bias} implies that, in the presence of uniformly bounded labels, the achieved interpolator in the mild overparameterization regime as $\alpha\to 0$ scales as $\sqrt{n}$. Although this interpolator is not minimal in norm, it recovers the $\sqrt{n}$ scaling of $\lVert \theta_\opt \rVert_2^2$, thanks to \Cref{eq:optnorm}. Also note that this type of bound could be easily extended to random labels, replacing the $\max_{i\in[n]} |y_i|$ term by $\max_{i\in[n]} \sqrt{\mathbb{E}[y_i^2]}$ for, e.g., sub-Gaussian labels. \Cref{thm:implicit_bias} does not require any lower bound on~$m$ as \Cref{alg:s.t.s} is still defined in that case. However, $m$ must be of order at least $\log(n)$ to ensure that the achieved estimator interpolates the data, which is guaranteed by \Cref{ass:A.j.star:A} \Cref{ass:A.j.star}.

\medskip

The proof of \Cref{thm:implicit_bias} reflects the incremental nature of \Cref{alg:s.t.s}. Each newly added positive (resp.\ negative) neuron approximately fits half of the remaining positive (resp.\ negative) labels. Consequently, after about $\log(n)$ increments, nearly all data points are fitted. 
The proof relies on a union bound over all subsets of the first $\log(n)$ neurons used to construct the estimator. This approach does not depend on the specific order induced by \Cref{alg:s.t.s}. As a result, the bound in \Cref{thm:implicit_bias} does not capture how the norm of the interpolator decreases as the number of neurons $m$ grows. Establishing such a result remains open, as it would require a more refined analysis that accounts for dependencies between the successive increments in \Cref{alg:s.t.s}.

\section{Conclusion}

In this work, we have provided a complete analysis of the gradient flow dynamics of mildly overparameterized two-layer ReLU networks from small initialization in the special case of orthogonal data. In doing so, we have extended existing convergence results and shown an incremental learning phenomenon.



There are several avenues for future research. 
It would be valuable to extend the proof of saddle-to-saddle dynamics to more general data distributions. Possibilities could include weakly interacting data \citep{dana2025convergence} or clustered data with orthogonal centers \citep{min2025gradient}. Such work could shed light on incremental learning in more realistic settings. 

In \Cref{thm:implicit_bias} on the implicit bias, there is no reference to the width except in the restriction $\log(m) = \lito(\npm^{1/8}/\log(\npm))$. Future work could seek to precisely quantify the dependence of the learned norm on the network width as one approaches the regime $m = \exp(n)$. Doing so would require a more intricate analysis of the dependence of each stage of \Cref{alg:s.t.s} on those that preceded it, which could be of independent interest in analyzing similar sequential algorithms.


\subsubsection*{Acknowledgments and Disclosure of Funding}

The authors are thankful to Nikolaos Zygouras for insightful discussions. The authors acknowledge the Engineering and Physical Sciences Research Council (EPSRC) and DIMAP research centre at the University of Warwick for partial support. J. Town is supported by the EPSRC through the Mathematics of Systems II Centre for Doctoral Training at the University of Warwick (reference EP/S022244/1). B. Lewis is supported by the EPSRC studentship 2927289. The authors acknowledge the use of the Batch Compute System in the Department of Computer Science at the University of Warwick, and associated support services, in the completion of this work.

\bibliographystyle{plainnat}
{\small \bibliography{main.bib}}

\clearpage
\appendix
\addcontentsline{toc}{section}{Appendix} 
\part{Appendix} 
\parttoc 

\section{Experiments} \label{app:expe}

\subsection{Experimental details}\label{sec:expe}

\cref{fig:illustration_intro} was generated by training a two-layer ReLU network with six neurons on 64 orthonormal data points in 64 dimension. The labels are chosen as the absolute values of standard Gaussian random variables.
As explained in \cref{app:autonomous}, without the sign restriction on the labels, we get two independent processes: neurons with negative output weights will turn and stay in the subspace orthogonal to the space spanned by the data points with positive labels, and vice versa. The sign restriction on the labels therefore helps to more clearly illustrate one of these processes in the plots, as opposed to a loss curve that is the combination of interleaving these two independent processes.

The network weights are initialized such that the hidden layer neurons have random directions and $\ell_2$ norm $e^{-500}$ and the output weights are $e^{-500}$. To avoid numerical issues, we maintain a scaled version of the weights such that they are closer to~$1$ (i.e., initially scaled by $e^{500}$) and separately track the logarithm of the scale of the weights (initially $-500$). During training, the stored weights are kept near $1$ while the logarithm of the scale is adjusted. The loss is the mean squared error and we train the network with full batch gradient descent and a learning rate of 0.01 until the loss is less than $10^{-20}$. The plots in \cref{fig:illustration_intro} start at epoch $\approx 0.9\cdot t_1$,
For the purposes of the illustration, we consider a data point fitted when the output of the network for that point is at least 50\% of the corresponding label. Due to the exponential growth of neuron norms, once the output of a network for a data point reaches such a significant proportion of the label, the output of the network at the data point very rapidly closes in on the label value. Therefore, the exact choice of the threshold (here 50\%) is somewhat arbitrary but also immaterial in the sense that it does not greatly influence the general picture.

\subsection{Additional experiments}

\begin{figure}
\centering
\includegraphics{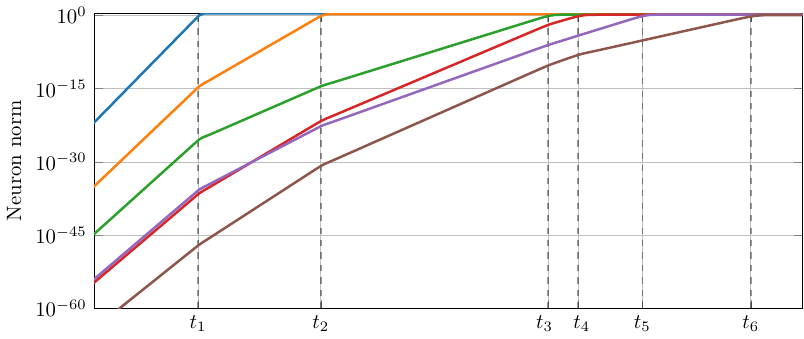}
\caption{\label{fig:illustration_intro_log}The same as the right plot of \cref{fig:illustration_intro} but with the $y$-axis on a logarithmic scale. The neuron norms increase exponentially. When new data points are fitted, the rate of this exponential increase reduces for the remaining neurons that were active on at least one of those newly fitted data points.}
\end{figure}

For the same training as in \cref{fig:illustration_intro}, \cref{fig:illustration_intro_log} shows how the $\ell_2$~norms of the hidden layer neurons change over time, but with a logarithmic $y$-axis. This reveals that the norms grow at an exponential rate. At the jump times, when a neuron fits the data on which it is active, that neuron stops growing. Other neurons continue to grow. However, if they were active on some of the same data points as the newly fitted neuron, they are now growing at a slower exponential rate than before the jump time. Therefore, the logarithms of the neuron norms are piecewise linear curves with decreasing slopes. The slopes of these linear segments as well as the jump times are predicted by \cref{alg:s.t.s}. Specifically, between jump times $t_p$~and~$t_{p+1}$, the slope for neuron~$j$ is proportional to $\left\|D_j^{(p)}\right\|_2$.

\begin{figure}
\centering
\subfigure[Near the start]{%
\includegraphics[
    width=0.44\linewidth,
    height=5.3cm,
    keepaspectratio
]{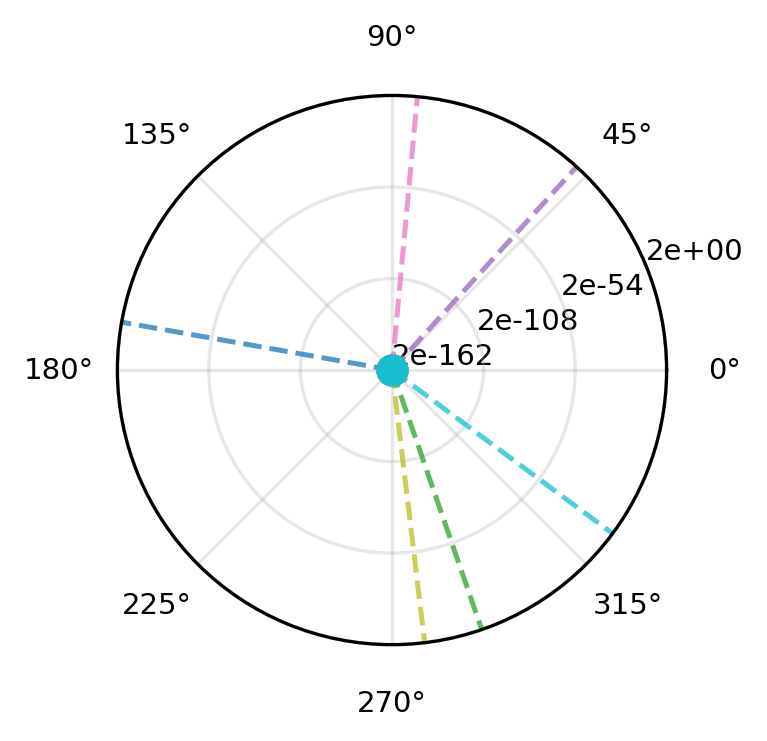}%
}
\hfill
\subfigure[Around time $t_1/2$]{%
\includegraphics[
    width=0.44\linewidth,
    height=5.3cm,
    keepaspectratio
]{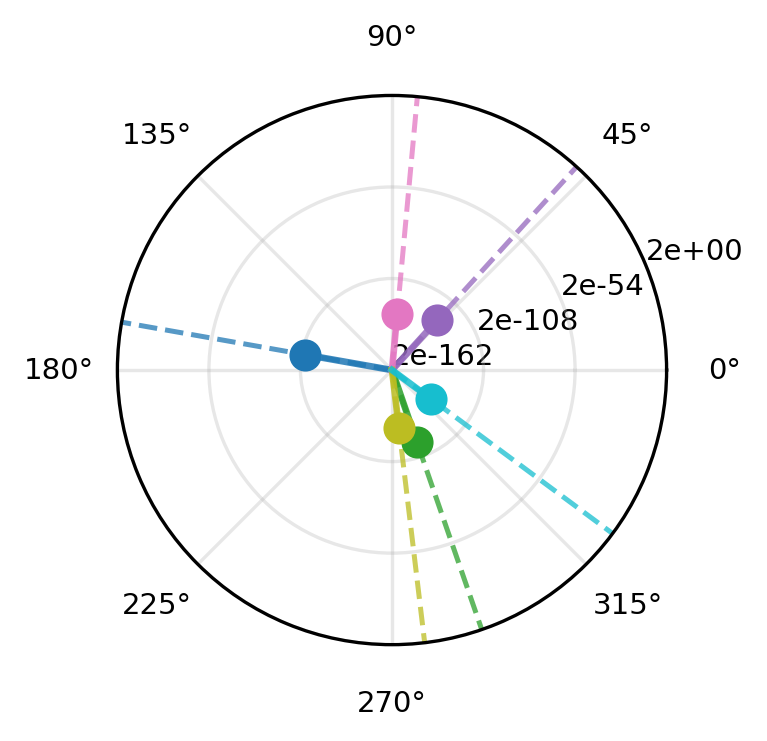}%
}
\par\vspace{0.8em}
\subfigure[Just before time $t_1$]{%
\includegraphics[
    width=0.44\linewidth,
    height=5.3cm,
    keepaspectratio
]{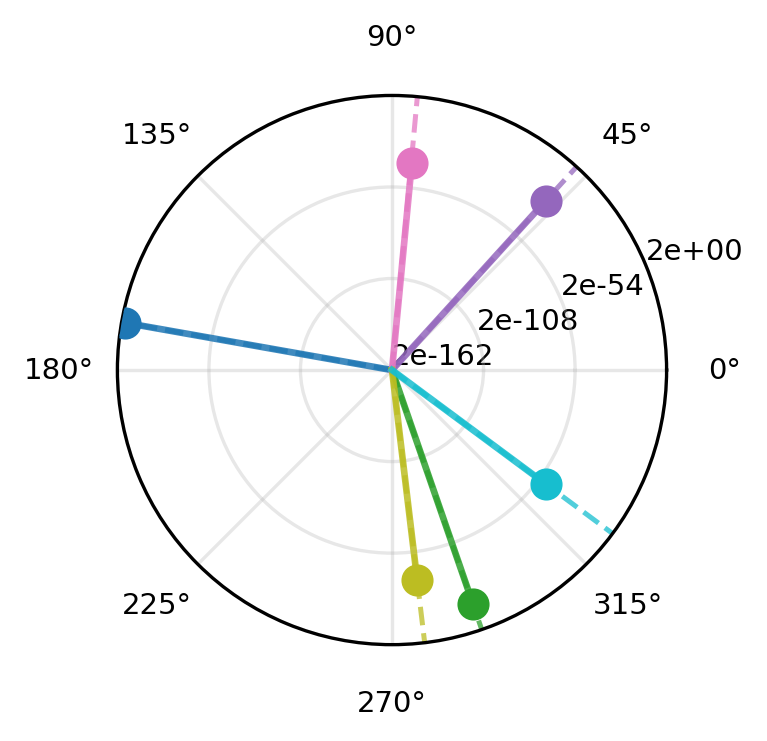}%
}
\hfill
\subfigure[Just after time $t_1$]{%
\includegraphics[
    width=0.44\linewidth,
    height=5.3cm,
    keepaspectratio
]{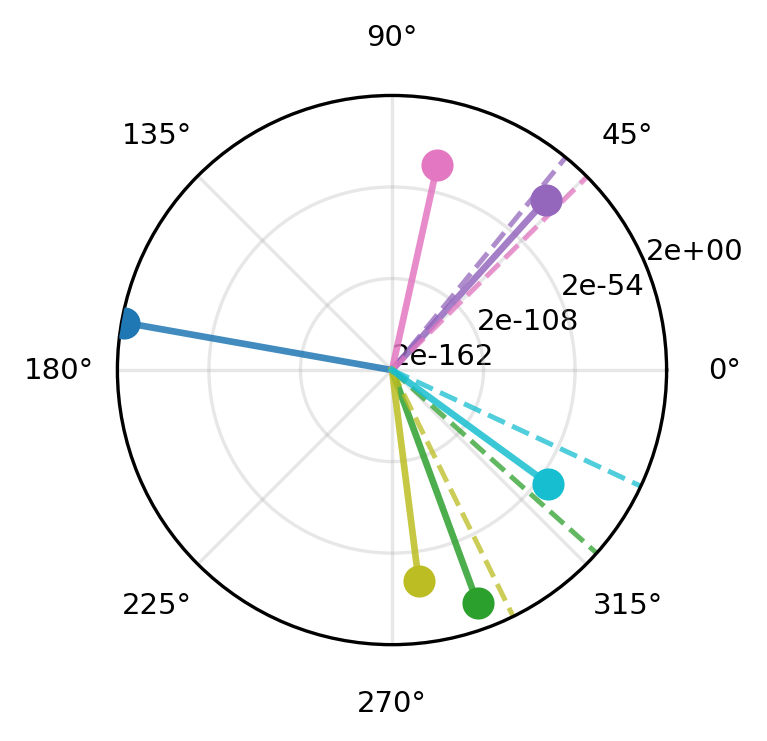}%
}
\par\vspace{0.8em}
\subfigure[A short while after $t_1$]{%
\includegraphics[
    width=0.44\linewidth,
    height=5.3cm,
    keepaspectratio
]{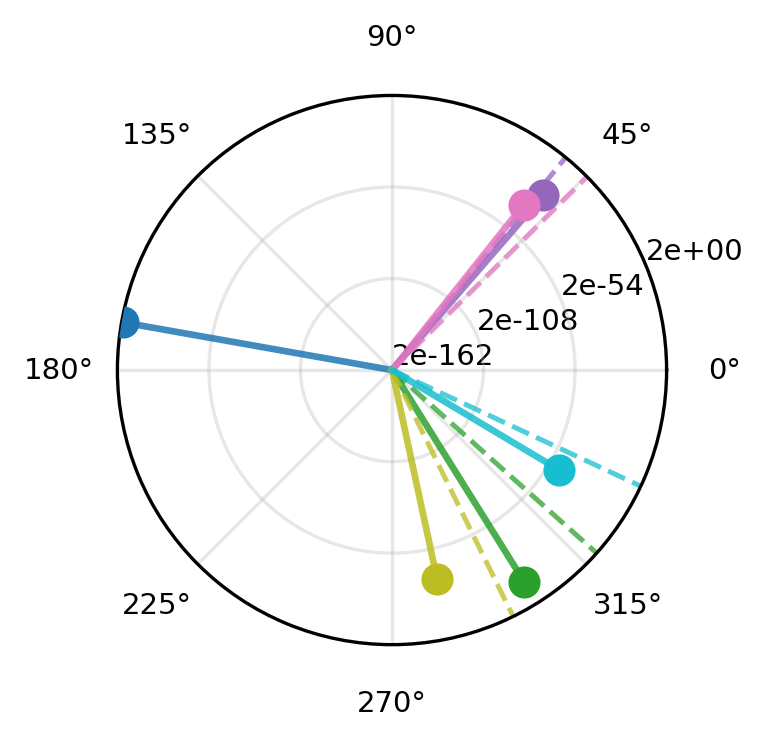}%
}
\hfill
\subfigure[Around time $(t_1+t_2)/2$]{%
\includegraphics[
    width=0.44\linewidth,
    height=5.3cm,
    keepaspectratio
]{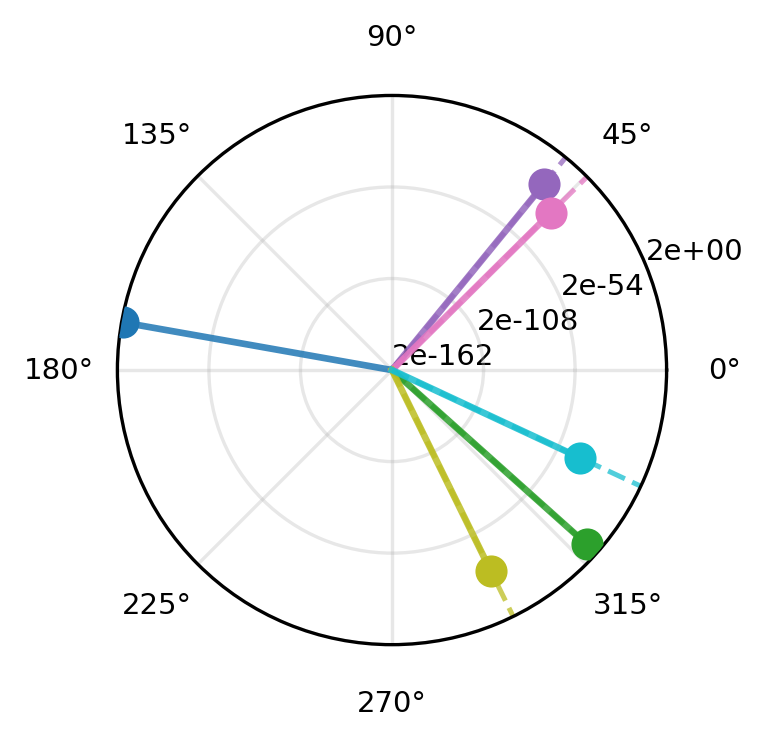}%
}
\caption{\label{fig:polarplots}For the same training as in \cref{fig:illustration_intro}, these polar plots show the norm of each neuron (on a logarithmic scale) as well as the direction of the neuron projected onto a random two-dimensional plane, at different points in time. The projections of the neuron directions from the limit process $\theta^\circ(t)$ predicted by \cref{alg:s.t.s} are indicated by dashed lines.}
\end{figure}

Still for the same training, \cref{fig:polarplots} shows the norm (on a logarithmic scale) and direction (projected onto a random two-dimensional plane) of the six neurons at different points in time. Initially, the norms are very small. The neurons orient themselves in the directions predicted by \cref{alg:s.t.s} and their norm starts to increase exponentially. Just before time $t_1$, one neuron has a significant norm while the remaining neurons have norms that are still much smaller. Just after time $t_1$, the neuron directions predicted by \cref{alg:s.t.s} change and the neurons that did not fit any data yet turn quickly towards their new predicted directions. Time~$t_1$ approximately corresponds to epoch 507,000. Plot~(c) shows the process at epoch 506,000, plot~(d) shows it at epoch 508,000, plot~(e) shows it at epoch 510,000, and plot~(f) shows it at epoch 538,000.

\paragraph{Norm minimization.}

\begin{figure}
    \centering
    \includegraphics{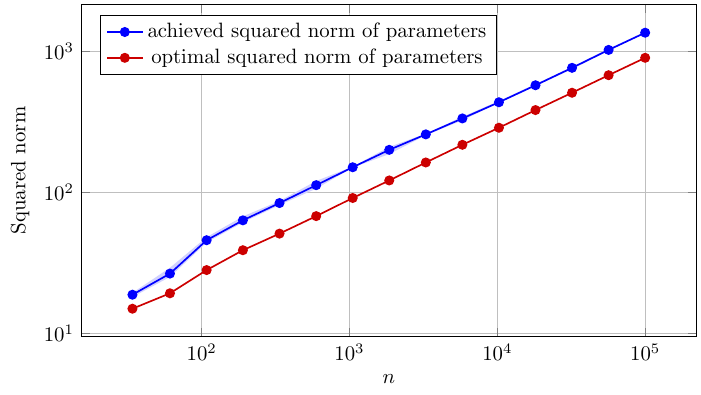}
    \caption{
    We train a two-layer ReLU network with increasing data dimensionality and number of data points~$n$, and plot the square of the norm of of the parameters of the final network as well as that of the optimal interpolator (see \cref{eq:optnorm}). The hidden layer size~$m$ increases logarithmically in~$n$, and we use a balanced initialization of scale $10^{-30}/\sqrt{m}$. Each value plotted is the mean over three runs, with each run using fixed orthogonal inputs and different random labels which are standard Gaussians. We use full batch gradient descent with a learning rate of~$0.5$ for one million epochs or until the loss reaches $10^{-10}$. The area between the minimum and maximum values is shaded, however they are relatively concentrated.}
    \label{fig:n_on_squaredNorm}
\end{figure}

\cref{thm:implicit_bias} gives an upper bound on the squared $\ell_2$~norm of the parameters of the achieved interpolator in the mild overparameterization regime. For uniformly bounded labels, this bound is of the order $\sqrt{n}$ and therefore of the same order as the minimum norm interpolator. \cref{fig:n_on_squaredNorm} experimentally shows the growth of the squared norm of the parameters of the achieved interpolator versus that of the optimal interpolator, depending on the number~$n$ of data points.

Specifically, we train a two-layer ReLU network with full batch gradient descent and a learning rate of 0.5. For different $n$, a network of width $m=\lceil\ln(10000\cdot n)/\ln(4/3)\rceil$ is initialized such that the hidden neurons have random directions and norm $10^{-30}/\sqrt{m}$ and the output neurons have a random sign and absolute value $10^{-30}/\sqrt{m}$. The choice of $m$ ensures that there are a sufficient number of neurons to ensure interpolation in a typical training run. The factor $1/\sqrt{m}$ in the initialization scale ensures that the initial squared norm of the network parameters remains the same regardless of the width of the network. The network is trained on data points corresponding to the $n\times n$ identity matrix and labels are drawn from a standard Gaussian distribution. The identity matrix is chosen for computational and memory efficiency, but we note that the training dynamics are rotation invariant and therefore the behavior for the identity matrix is the same as for any other orthonormal basis.

Training is run for one million epochs but stopped earlier if the loss reaches $10^{-10}$. We record the squared norm of the parameters of the final network. This is repeated three times with new random initial weights. In \cref{fig:n_on_squaredNorm}, we plot the mean values of the squared norms as well the the optimum squared norm $\lVert \theta_{\opt} \rVert_2^2$ (see \cref{eq:optnorm}). We also indicate the min/max deviation over the three runs, but this is barely visible due to relatively low variance of the resulting squared norms.

In the plot, both axes are logarithmic. The straight lines of slope about~$1/2$ therefore indicate that both squared norms grow proportional to~$\sqrt{n}$ and that the achieved squared norm is only by a constant factor larger than the optimal one.

\paragraph{Larger overparameterization.}
\begin{figure}
    \centering
    \includegraphics{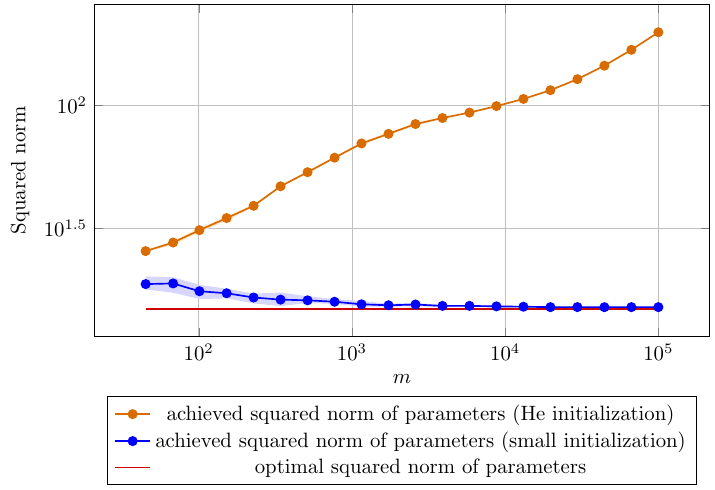}
    \caption{We train two-layer ReLU networks of varying widths $m$ on 32 orthonormal data points in 32 dimensions and labels drawn from the standard Gaussian distribution. We compare two initialization schemes: balanced small initialization where hidden neurons have norm $10^{-30}/\sqrt{m}$ and uniform He initialization~\citep{HeZRS15}, which is the default used in PyTorch to initialize linear layers. We train using full batch gradient descent for one million epochs or until the loss reaches $10^{-10}$. Each value plotted is the mean over three runs with new random initial weights. The area between the minimum and maximum values is shaded, however they are relatively concentrated.}
    \label{fig:m_on_squaredNorm}
\end{figure}
Our next experiment investigates what happens when we increase the overparameterization for two different initialization schemes. The first initialization scheme is the same as before used to produce \cref{fig:n_on_squaredNorm} (i.e., balanced small initialization with hidden layer neuron norms equal to $10^{-30}/\sqrt{m}$). The second scheme is a uniform He initialization which in our network draws hidden layer weights uniformly at random from $[-1/\sqrt{d},1/\sqrt{d}]$ and output layer weights uniformly at random from $[-1/\sqrt{m},1/\sqrt{m}]$. This is the default initialization for linear layers in PyTorch.
The experimental details are the same as before, except that $n$ is now fixed at 32 while $m$ is varied and the learning rate is 1 when using the first initialization scheme and 0.001 when using the second initialization scheme.

Each point in the plot of \cref{fig:m_on_squaredNorm} is the mean of three runs with different random initial weights. We also indicate the min/max deviation over the three runs, but this is barely visible due to relatively low variance of the resulting squared norms.

We observe that with small initialization, the square $\ell_2$~norm of the trained network approaches the theoretical optimum as the width of the network increases. This is in contrast to the uniform He initialization, for which the square $\ell_2$-norm of the trained network keeps increasing.

\paragraph{Compute.} The experiment for \cref{fig:illustration_intro,fig:illustration_intro_log,fig:polarplots} takes only about one minute to generate on a standard desktop CPU.
Generating \cref{fig:n_on_squaredNorm} takes about 3 hours, and generating \cref{fig:m_on_squaredNorm} takes about 10 minutes, both on a consumer GPU.

\section{Technical background} \label{app:clarke}

\subsection{Clarke subdifferential}

We introduce the Clarke subdifferential \citep{clarke1975generalized} and a basic result on the existence of gradient flows using this generalized derivative. 

\begin{defn}

    The Clarke subdifferential of a locally Lipschitz function $f : \mathcal{X} \to \mathbb{R}$ at the point $x \in \mathcal{X}$ is defined as the convex set 

    \begin{equation}
        \partial f(x) = \mathrm{conv} \left\{ \lim_{k \to \infty} \nabla f(x_k) \; \Big\vert x_k \to x, \; f \; \text{is differentiable at} \; x_k  \right\} .
    \end{equation}
    
\end{defn}

One may see \citet[Section 3]{LyuL20} for further details. We present the following result guaranteeing the local existence of solutions.  

\begin{thm}{\citep[Theorem 3, Chapter 2]{aubin1984odes}}
    Let $\mathcal{X}$ be a Hilbert space and $Q \subset \mathbb{R} \times \mathcal{X}$ be an open subset containing $(0, x_0)$. Let $F$ be an upper semicontinuous map from $Q$ into the non-empty closed convex subsets of $\mathcal{X}$. We assume that $(t, x) \mapsto m(F(t, x))$ is locally compact, where for $A \subset \mathcal{X}$ we define $m(A) := \inf \{ \lVert v \rVert \mid v \in A \}$. Then there exist $T>0$ and an absolutely continuous function $x(\cdot)$ defined on $[0, T]$, a solution to the differential inclusion
    \begin{align}
        x'(t) \in F (t, x(t)); && x(0)=x_0.
    \end{align}
\end{thm}

\subsection{Time rescaling} \label{app:time_rescaling}

We provide here some intuition for why it is necessary to rescale time in order for dynamics to occur on a finite timescale as the initialization scale decreases to $0$. Fundamentally, it is because gradients become small in the vicinity of critical points and so the dynamics slow down. Note that near to a critical point we may linearize a dynamical system. In the case that this critical point is the origin, which will already give the required scaling to have non-trivial dynamics, we obtain $\dot{\theta} \approx - \nabla^2 \mathcal{L}(\mathbf{0}) \theta$, where $\nabla^2$ denotes the Hessian. The solution to the linearized dynamics is given by $\theta(t) = \theta(0) \exp( - \nabla^2 \mathcal{L}(\mathbf{0}) t)$. We can obtain a lower bound on the change in $\theta$ by projecting onto the eigenvectors of the Hessian which correspond to positive eigenvalues and assuming the slowest growth rate applies to all these components (i.e., take the smallest positive eigenvalue). Assuming the fastest growth rate in all components (i.e., take the largest positive eigenvalue) will give an upper bound. Concretely, taking norms gives $\lVert \mathbf{P}^\perp_+\theta(0) \rVert \exp( \lambda_+^\text{min} t) \leq \lVert \theta(t) \rVert \leq \lVert \theta(0) \rVert \exp( \lambda^{\text{max}}t)$ where $\lambda_+^\text{min}$ is the smallest positive eigenvalue of $- \nabla^2 \mathcal{L}(\mathbf{0})$, $\lambda^\text{max}$ is the largest eigenvalue, and $\mathbf{P}^\perp_+$ is the orthogonal projection onto the span of eigenvectors of $-\nabla^2 \mathcal{L}(\mathbf{0})$ which correspond to positive eigenvalues. Based on these linearized dynamics, we expect that the escape time, $T_e$, for $\lVert \theta(t) \rVert$ to reach some constant $C$ obeys the bounds

 \begin{equation}
     \frac{1}{\lambda^{\text{max}}} \left( \log(C) + \log \left( \frac{1}{\lVert \mathbf{P}^\perp_+\theta(0) \rVert} \right) \right) \leq T_e \leq \frac{1}{\lambda^{\text{min}}_+} \left( \log(C) +  \log \left( \frac{1}{\lVert \theta(0) \rVert} \right) \right).   
 \end{equation}
 
Supposing that $\lVert \theta(0) \rVert, \lVert \mathbf{P}^\perp_+\theta(0) \rVert \sim \alpha$, we see that the required rescaling of time is $t \mapsto \log(1/\alpha) t$, as claimed, and we denote the accelerated parameters by $\widetilde{\theta}^\alpha(t) = \theta( \log(1/\alpha) t)$. Typically, we omit the superscript as the acceleration by a factor of $\log(1/\alpha)$ is implicit in the tilde notation. As noted in the main body of the paper, the flow defining this rescaled process is simply $\dot{\widetilde{\theta}}(t) \in - \log(1/\alpha) \partial \mathcal{L}(\widetilde{\theta}(t))$ for almost all $t \geq 0$.

\section{Recap of notations}

For ease of reference, \Cref{notation_table} collects notation used throughout the paper, including some symbols that appear only in the appendix. 

\begin{table}
\caption{Definitions and descriptions of named quantities, collected here from the text. Hyperlinks in the Symbol column take the reader to the first point in the paper where a given symbol is introduced.}
\label{notation_table}
\begin{tabular}{P{0.08\textwidth}P{0.44\textwidth}P{0.39\textwidth}} \\ \toprule
\textbf{Symbol} & \textbf{Definition} & \textbf{Description} \\ \midrule

\hyperref[Is_defn]{$I_s$}
& $\{i \in [n] \;\vert\; \sgn(y_i) = s\}$
&  The set of data~$i$ with sign $s$, for $s \in \{ \pm 1\}$. \\ \cmidrule{1-3}

\hyperref[Aij_defn]{$A_{i,j}$}
& $\iind{w_j^\top(0) \, x_i > 0 \text{ and } i \in I_{s_j}}$
&  The set of data~$i$ on which, at initialization, neuron~$j$ is active with correct second-layer weight sign. \\ \cmidrule{1-3}

\hyperref[Js_defn]{$J_s$}
& $\{ j \in [m] \mid \sgn (a_j(0)) = s\}$
&  The set of neurons~$j$ with sign $s$, for $s \in \{ \pm 1\}$. \\ \cmidrule{1-3}

\vspace{2ex}\hyperref[tk_defn]{$t_k$} 
& $\begin{aligned}[b]
    t_0&=0 \\[-1ex] t_{k+1} &= \min_{j \in N_U^{(k)}} \left\{ t_k - \frac{\ell_j^\circ(t_k)}{\left\lVert D_j^{(k)} \right\rVert} \right\}
    \end{aligned}$
& \vspace{2ex}The jump times. \\ \cmidrule{1-3}

\hyperref[ellj_defn]{$\ell_j^\circ(t)$} 
& $\begin{aligned}[b]
   \ell_j^\circ(0) &= -1 \\
   \ell_j^\circ(t) &= \min \left( \ell_j^\circ(t_k) + (t-t_k)\|D_j^{(k)}\|, 0 \right)
   \end{aligned}$
& The exponents of the neurons' norms, where $t \in [t_k, t_{k+1}]$. \\ \cmidrule{1-3}
\vspace{0.6cm}
\hyperref[jstar_defn]{$j_\star^{(k)}$} 
& \vspace{-2ex}$\displaystyle{\argmin_{j \in N_U^{(k)}}} \left\{ t_k - \dfrac{\ell_j^\circ(t_k)}{\left\lVert D_j^{(k)} \right\rVert} \right\}$
&  The neuron which fits when moving from saddle $k$ to saddle $k+1$. \\ \cmidrule{1-3}

\hyperref[Suk_Nuk_defn]{$N_U^{(k)}$} 
& $\begin{aligned}[b]
   N_U^{(0)} &= [m] \\
   N_U^{(k+1)} &= N_U^{(k)} \setminus \{ j^\star \}
   \end{aligned}$
& The set of unfitted neurons at saddle $k$. \\ \cmidrule{1-3}

\hyperref[Nfk_defn]{$N_F^{(k)}$} 
& $\begin{aligned}[b]
   [m] \setminus N_U^{(k)}
   \end{aligned}$
& The set of fitted neurons at saddle $k$. \\ \cmidrule{1-3}

\hyperref[Sjk_defn]{$S_j^{(k)}$} 
& \vspace{-2ex}$\begin{aligned}[b]
   S_j^{(0)} = S_j & = \{ i \in [n] \mid A_{i,j}=1 \} \\
   S_j^{(k+1)} & =
   \begin{cases}
   S_j^{(k)} \setminus S_{j_\star^{(k)}}^{(k)} & \text{if } j \neq j_\star^{(k)} \\[1ex]
   S_j^{(k)}                             & \text{if } j =    j_\star^{(k)}
   \end{cases}
   \end{aligned}$
&  The set of unfitted data on which neuron $j$ is active with correct second-layer weight at saddle $k$, and which has not yet been fitted by any other neuron. \\ \cmidrule{1-3}

\hyperref[Suk_Nuk_defn]{$S_U^{(k)}$} 
& $\begin{aligned}[b]
   S_U^{(0)} &= [n] \\
   S_U^{(k+1)} &= S_U^{(k)} \setminus S_{j^\star}^{(k)}
   \end{aligned}$
& The set of unfitted data at saddle $k$. \\ \cmidrule{1-3}

\hyperref[Sfk_defn]{$S_F^{(k)}$} 
& $\begin{aligned}[b]
   [n] \setminus S_U^{(k)}
   \end{aligned}$
& The set of fitted data at saddle $k$. \\ \cmidrule{1-3}

\hyperref[Djk_defn]{$D_j^{(k)}$} 
& \vspace{-1ex}$D_j^{(k)} = \dfrac{1}{n} \displaystyle{\sum_{i \in S_j^{(k)}}} y_i x_i$
& The label-weighted sum of data in the set $S_j^{(k)}$, which is the direction in which neuron $j$ aligns and grows at saddle~$k$. \\ \cmidrule{1-3}

\hyperref[Ek_defn]{$E^{(k)}$} 
& \vspace{-1ex}$E^{(k)} = -\dfrac{1}{n} \displaystyle{\sum_{i \in S_F^{(k)}}} (h_{\widetilde{\theta}}(x_i) - y_i) x_i$
& A dynamical vector encoding the loss on data which has already been fitted, with $n \left\lVert E^{(k)} \right\rVert^2$ directly corresponding to these terms of the loss. \\ \cmidrule{1-3}

\hyperref[fd_defn]{$\fD_j$} 
& \vspace{-1.5ex}$\fD_j = -\dfrac{1}{n} \displaystyle{\sum_{\substack{i \in [n] \,\mid \\ w_j^\top x_i>0}}} (h_{\widetilde{\theta}}(x_i) - y_i) x_i$
& \vspace{1ex}A dynamical vector encoding the loss on all data on which neuron $j$ is active, with $n \lVert \fD_j \rVert^2$ directly corresponding to these terms of the loss. \\ \bottomrule
\end{tabular}
\end{table}

\section{Proof of \texorpdfstring{\Cref{th:s.t.s}}{Theorem \ref*{th:s.t.s}}}

The proof convergence to the limit process defined in \Cref{alg:s.t.s} is due to a non-asymptotic analysis of the dynamics of all neurons around the jump from saddle $k$ to saddle $k+1$. We use this to construct an inductive argument for convergence, with the number of saddles traversed by the flow serving as the parameter of the induction. Specifically, we will assume the following inductive hypothesis on convergence at saddle $k$:

\begin{hyp}[for $k \in \{ 0, 1, \dots , p\}$] \label{inductive_hypothesis}
     Given any $\kappa \in (t_k, t_{k+1})$, there exists $\delta>0$ such that for all $\xi>0$, there exists $\alpha^*>0$ such that for all $\alpha \in (0,\alpha^*)$:
    
    \begin{enumerate}
        \item  $\forall j \in N_U^{(k)}, \; \left\vert \ell_j^\circ(\kappa) +  \log_{\alpha} ( \lVert \widetilde{w}_j(\kappa) \rVert )\right\vert < \xi .$ \label{IH_unfitted_norms}
        \item $\forall j \in N_F^{(k)}, \; \Big\vert n \left\lVert D_j^{(k)} \right\rVert - \lVert \widetilde{w}_j(\kappa) \rVert^2 \Big\vert < \alpha^{\delta} . $ \label{IH_fitted_norms}
        \item $\forall i \in S_F^{(k)}, \; | h_{\widetilde{\theta}(\kappa)}(x_i) - y_i | < \alpha^{\delta} .$  \label{IH_residual}
        \item $\forall j \in [m], \; \widetilde{\overline{w}}^\top_j(\kappa) \overline{D}^{(k)}_j > 1 - \alpha^{\delta} .$ \label{IH_alignment}
    \end{enumerate} 
\end{hyp}

In the above, and throughout the proof, we use $\lVert \cdot \rVert$ to denote the $\ell_2$-norm. The reader should note that $\delta$ depends on $\kappa$ in \cref{inductive_hypothesis}; in particular we have $\lim_{\kappa \to 0} \delta = 0$. This does not present a problem, since the proof is of convergence on compact subsets of $\mathbb{R}$ which exclude $\left( t_k \right)_{k=1}^p$. Therefore, $\kappa$ may be viewed as a quantity which is fixed at the beginning of the proof depending on the compact set in question. The following lemma states that on compact intervals not containing a jump time, we have the desired uniform convergence to the limit process under \Cref{inductive_hypothesis}.

\begin{lem} \label{lem:network_pre_saddle}
    Suppose that \cref{inductive_hypothesis} holds for some $k \in {0, 1, \dots, p-1}$. Then given any $\kappa, \kappa' > 0$ such that $t_k < \kappa < \kappa' < t_{k+1}$, $\widetilde{\theta}^\alpha$ converges uniformly to $\theta^\circ$ on $[\kappa, \kappa']$.
\end{lem}

The following lemma recovers \Cref{inductive_hypothesis} at saddle $k+1$.

\begin{lem} \label{lem:network_post_saddle}
    If \cref{inductive_hypothesis} holds for some $k \in \{ 0, 1, \dots , p-1\}$, then it holds for $k+1$.
\end{lem}


    
Each of these lemmas, as well as the final proof of \Cref{th:s.t.s}, are provided in \Cref{sect:final_proof_thm_2}. We deal with the final saddle separately in \Cref{final_saddle} because, although the arguments for uniform convergence on closed intervals contained in $(t_{p}, \infty)$ are essentially the same as those for $k<p$, it is simpler to account for this case separately.

\subsection{Basis case and autonomous systems}

We provide here an analysis of the early time dynamics, specifically a proof of the alignment of neurons towards the directions $D_j^{(0)}$ by an arbitrary positive time, where

\begin{equation}
    D_j^{(0)} := \frac{1}{n} \sum_{ i \in [n]} A_{i,j}y_i x_i,
\end{equation}

and we recall that

\begin{equation}
    A_{i,j} = \iind{w_j^\top(0) \, x_i > 0 \text{ and } i \in I_{s_j}} .
\end{equation}

This will amount to a proof of a statement similar to the case $k=0$ for \Cref{inductive_hypothesis}. At this point, it will also be helpful to introduce the vector 

\begin{equation}
    D_j^{(k)} := \frac{1}{n} \sum_{i \in S_j^{(k)}} y_i x_i .
\end{equation}

In the above display, $S_j^{(k)}$ is recursively defined by

\label{Sjk_defn}
\begin{equation} 
\begin{aligned}
    S_j^{(0)} &= \{ i \in [n] \mid A_{i,j}=1 \} \\
   S_j^{(k+1)} & =
   \begin{cases}
   S_j^{(k)} \setminus S_{j_\star^{(k)}}^{(k)} & \text{if } j \neq j_\star^{(k)} \\[1ex]
   S_j^{(k)}                             & \text{if } j =    j_\star^{(k)}
   \end{cases} 
\end{aligned} 
\end{equation}

and denotes the set of data on which neuron $j$ is active at saddle $k$ and on which no neuron in the set 

\label{Nfk_defn}
\begin{equation}
    N_F^{(k)} := [m] \setminus N_U^{(k)} 
\end{equation}

is active. Note that $S_j$ (as introduced in \Cref{alg:s.t.s}) is the same set as $S_j^{(0)}$. In a similar way to how we define the set of neurons which have fitted their data, we can also define the set of fitted data itself by 

\label{Sfk_defn}
\begin{equation}
    S_F^{(k)} := [n] \setminus S_U^{(k)}.
\end{equation}

One may see \Cref{alg:s.t.s} for the construction of the sets $N_U^{(k)}$ and $S_U^{(k)}$. Our arguments will use the equations 

\begin{align}
 \dot{\widetilde{a}}_j &= \log(1/\alpha)  \widetilde{w}_j^\top \fD_j & \dot{\widetilde{w}}_j &= \log(1/\alpha) \widetilde{a}_j \fD_j  , 
\end{align}

where we have

\begin{equation}
    \fD_j = -\frac{1}{n} \sum_{i \in [n] \, \mid \, \widetilde{w}_j^\top x_i > 0}  (h_{\widetilde{\theta}}(x_i) - y_i)x_i. 
\end{equation}

As noted in \cref{setup}, these equations combined with \cref{pr:bal} and \cref{ass:dead} imply that 

\begin{equation}
    \dot{\widetilde{\overline{w}}}_j(t)  = s_j \log(1/\alpha) \left( \fD_j (t) - \widetilde{\overline{w}}_j(t)  \widetilde{\overline{w}}_j^\top(t) \fD_j (t)  \right).
\end{equation}

Using these equations, we will ensure that by an arbitrary time $\kappa>0$, all neurons have aligned towards the directions $D_j^{(0)}$. Furthermore, we will verify by introducing an auxiliary parameter of the proof, $\widehat{\kappa}$, that $-\widetilde{\ell}_j(\kappa) \approx -\ell_j^\circ(\kappa)$ with arbitrary precision. The sketch of the proof is to ensure that by any time $\widehat{\kappa}>0$, the realignment has occurred and then remains stable until time $\kappa$. Therefore, the precision of the exponents of the unfitted neurons is controlled by the arbitrary parameter $\widehat{\kappa}$. To state the lemma below, we define for $s \in \{\pm 1\}$ the set of neurons with output sign given by $s$: 

\label{Js_defn}
\begin{equation}
    J_s := \{ j \in [m] \mid \sgn(a_j(0)) = s \}. 
\end{equation}

\begin{lem} \label{induction_basis_case_lemma}
    For any $\kappa \in \left( 0, \frac{1}{4} \sqrt{\frac{n}{2\mathcal{L}(0)}} \right)$ there exists $\delta>0$ such that for any $\xi>0$ there exists $\alpha^*(\kappa, \xi)>0$ such that for all $\alpha \in (0, \alpha^*)$ and both signs $s \in \{ \pm 1 \}$ the following are true:

    \begin{enumerate}
        \item $\forall j \in J_s, \; s_j \, \widetilde{\overline{w}}_j^\top(\kappa) \overline{D}_j^{(0)} > 1-\alpha^\delta.$ \label{lem1.1}
        \item $\forall j \in J_s, \, \forall t \in (0, \kappa), \; \left\vert \ell_j^\circ(t) +  \log_{\alpha} ( \lVert \widetilde{w}_j \rVert )\right\vert < \xi $ \label{lem1.2}.
        \item $\forall i \in I_{-s}, \forall j \in J_s, \; \widetilde{w}_j^\top(\kappa) x_i \leq 0.$ \label{lem1.3}
    \end{enumerate}
\end{lem}

The monotonicity of the gradient flow means there exists a maximum rate of growth it is possible for any neuron to achieve. Therefore, the upper bound on $\kappa$ appears as a technicality to ensure that even if a neuron were to have grown at this maximum rate from initialization, it could not have exceeded the threshold $\alpha^{1/2}$ which we introduce for the purpose of the proof. This kind of argument will be used several times throughout subsequent analyses (see \Cref{pr:uniform_bound_unfitted} and its application in the proof of \Cref{lemma_phase_1,lemma_phase_2a,lemma_phase_2b,lemma_phase_3}).

\Cref{lem1.3} \Cref{induction_basis_case_lemma} is a particularly important point, because it ensures that neurons deactivate within an arbitrarily short length of time on the data for which $A_{i, j} = 0$. Combining this with \Cref{ass:dead} will allow us to see that the network splits into two sub-networks, consisting only of neurons with positive (resp.\ negative) second-layer weight and which are active on data with only positive (resp.\ negative) labels. This allows us to consider w.l.o.g. the case that $s_j>0$ for all $j \in [m]$ and $\sgn (y_i) > 0$ for all $i \in [n]$. This is explained more precisely after the proof.

We stress now that the statement of \Cref{induction_basis_case_lemma} does not exactly correspond to the statement of \Cref{inductive_hypothesis}, since the interval $\left( 0, \frac{1}{4} \sqrt{\frac{n}{2\mathcal{L}(0)}} \right)$ does not correspond to the interval $(0, t_1)$. However, when uniform convergence to $\theta^\circ$ is proven, we will apply \Cref{lem:network_pre_saddle}, and this statement is sufficient to obtain uniform convergence on arbitrary closed intervals contained in $(0, t_1)$. Concretely, if we were to consider a compact set partially or entirely contained in $\left( \frac{1}{4} \sqrt{\frac{n}{2\mathcal{L}(0)}}, t_1 \right)$, we could obtain convergence by applying \Cref{lem:network_pre_saddle} on some interval whose lower endpoint is contained in $\left( 0, \frac{1}{4} \sqrt{\frac{n}{2\mathcal{L}(0)}} \right)$ and whose upper endpoint is greater than the maximum of the compact set.

\begin{proof}

Define $\tau := \inf \{ t \geq 0 \mid \exists \, j \in [m] \; \mathrm{s.t.} \; \lVert \widetilde{w}_j \rVert \geq \alpha^{1/2} \}$. As a preliminary, we check that $\kappa<\tau$. On $[0, \tau]$,

\begin{align}
    \frac{1}{\log(1/\alpha)}\frac{\df}{\df t} \lVert \widetilde{w}_j \rVert &= - \frac{\lVert \widetilde{w}_j \rVert}{n} \sum_{i \in [n] \, \mid \, \widetilde{w}_j^\top x_i > 0}  (h_{\widetilde{\theta}}(x_i)-y_i) \widetilde{\overline{w}}_j^\top x_i \label{basis_norm_growth_eqn} \\
    &\leq \sqrt{\frac{2\mathcal{L}(0)}{n}} \lVert \widetilde{w}_j \rVert.
\end{align}

This follows from the Cauchy-Schwarz inequality combined with the fact that $\mathcal{L}$ is monotone decreasing, by definition of the gradient flow. Gr{\"o}nwall's inequality gives that on this interval, 

\begin{align}
    \lVert \widetilde{w}_j(t) \rVert &\leq \lVert \widetilde{w}_j(0) \rVert \alpha^{- \sqrt{ \frac{2\mathcal{L}(0)}{n}}  t} \\
    &= \alpha^{1 - \sqrt{ \frac{2\mathcal{L}(0)}{n}}  t}.
\end{align}

where the second inequality is valid for $\alpha$ sufficiently small. Since $\kappa < \frac{1}{4} \sqrt{ \frac{2\mathcal{L}(0)}{n}}  t$, we indeed have $\kappa < \tau$. We may reverse the inequality and consequently obtain for $t \in [0, \tau]$ the lower bound 

\begin{equation}
    \lVert \widetilde{w}_j(t) \rVert > \alpha^{1 + \sqrt{ \frac{2\mathcal{L}(0)}{n}} t}.
\end{equation}

We conclude that for $t \in [0, \kappa]$

\begin{equation}
    \alpha^{1 + \sqrt{ \frac{2\mathcal{L}(0)}{n}} t} < \lVert \widetilde{w}_j(t) \rVert < \alpha^{1 - \sqrt{ \frac{2\mathcal{L}(0)}{n}} t}. \label{basis_case_norms_bound}
\end{equation}

Observe at this point that the signs are symmetric, and so we prove w.l.o.g. the statements for $s_j=1$.  

\begin{subproof}[Proof of \Cref{lem1.1}] For the claimed alignment, we compute that 

\begin{align}
    \frac{1}{\log(1/\alpha)} \frac{\df}{\df t} \widetilde{\overline{w}}_j^\top D_j^{(0)} &\geq \left\lVert D_j^{(0)} \right\rVert^2 - \left( \widetilde{\overline{w}}_j^\top D_j^{(0)} \right)^2 \\
    &\qquad - \frac{1}{n} \sum_{i \, \mid \, \widetilde{w}_j^\top x_i > 0} h_{\widetilde{\theta}}(x_i) \left( x_i^\top D_j^{(0)} - x_i^\top \widetilde{\overline{w}}_j \widetilde{\overline{w}}_j^\top D_j^{(0)} \right) \\
    &> \left\lVert D_j^{(0)} \right\rVert^2 - \left( \widetilde{\overline{w}}_j^\top D_j^{(0)} \right)^2 - 2m \left\lVert D_j^{(0)} \right\rVert \alpha,
\end{align}

where the first inequality is from discarding negatively labeled data on which neuron $j$ is active. This is because such data add a positive contribution for $\alpha$ small enough that $|h_{\widetilde{\theta}}(x_i)| < |y_i|$. We then use the fact that $f'(t)=\log(1/\alpha)(a^2-f(t)^2)$ is solved by $a \tanh(a \log(1/\alpha)t + t_0)$, where $t_0$ is an integration constant. Since the inner product is positive at initialization by definition, $t_0$ is positive. By the fact that $\tanh$ is monotone increasing, we obtain a lower bound by discarding this integration constant. This leaves 

\begin{align}
    \widetilde{\overline{w}}_j^\top D_j^{(0)} &> \left( \left\lVert D_j^{(0)} \right\rVert^2 - 2m\alpha \right)^{1/2} \tanh \left( \log(1/\alpha) \left( \left\lVert D_j^{(0)} \right\rVert^2 - 2m\alpha \right)^{1/2} t \right) \\
    &\geq \left( \left\lVert D_j^{(0)} \right\rVert - \sqrt{2m\alpha} \right) \left( 1 - \alpha^{2 \left( \left\lVert D_j^{(0)} \right\rVert - \sqrt{2m\alpha} \right) t} \right)  ,
\end{align}

where above we applied the inequalities $\sqrt{1-x^2} \geq 1 - x$ and $\tanh(x) \geq 1 - 2 \exp(-2x)$. Note that there exists $\alpha$ sufficiently small that for any $j$, $2\left(\left\lVert D_j^{(0)} \right\rVert - \sqrt{2m\alpha}\right) > \left\lVert D_j^{(0)} \right\rVert $, so for $\alpha$ below this threshold at time $t\in[0,\kappa]$ one has 

\begin{align}
    \widetilde{\overline{w}}_j^\top(t) D_j^{(0)} &> \left( \left\lVert D_j^{(0)} \right\rVert - \sqrt{2m\alpha} \right) \left( 1 - \alpha^{ \left\lVert D_j^{(0)} \right\rVert t} \right) \\
    &> \left\lVert D_j^{(0)} \right\rVert \left( 1 - \alpha^{\left\lVert D_j^{(0)} \right\rVert t} -\alpha^{1/4}  \right), \label{eq:basis_case_alignment_bound_general}
\end{align}
where the second inequality is valid for $\alpha\leq \frac{1}{(2m)^2}$.  
This directly implies that for $\alpha$ small enough and any $\kappa\in \left(0, \frac{1}{4} \sqrt{\frac{n}{2\mathcal{L}(0)}} \right)$, there exists $\delta>0$ such that the alignment of any neuron $j$ with direction $D_j^{(0)}$ is bounded by below by $1-\alpha^\delta$ at time $\kappa$, as required.
Note here that $\delta$ depends on $\kappa$, and $\alpha*$ only depends on the data.
\end{subproof}

\begin{subproof}[Proof of \Cref{lem1.2}] We select some arbitrary $\widehat{\kappa} \in (0, \kappa)$. Since $\widehat{\kappa}<\tau$, we may insert it into \cref{basis_case_norms_bound} to obtain

\begin{equation}
    \alpha^{1 + \sqrt{ \frac{2\mathcal{L}(0)}{n}} \widehat{\kappa}} < \lVert \widetilde{w}_j(\widehat{\kappa}) \rVert < \alpha^{1 - \sqrt{ \frac{2\mathcal{L}(0)}{n}} \widehat{\kappa}}. 
\end{equation}

We express this deviation in terms of the quantity $\ell_j^\circ(\widehat{\kappa})$. We note that $-\ell_j^\circ(\widehat{\kappa}) = 1 - \left\lVert D_j^{(0)} \right\rVert \widehat{\kappa}$ for all $j \in [m]$. Therefore, 

\begin{equation}
    \alpha^{-\ell_j^\circ(\widehat{\kappa}) + \left( \left\lVert D_j^{(0)} \right\rVert + \sqrt{ \frac{2\mathcal{L}(0)}{n}} \right) \widehat{\kappa}} < \lVert \widetilde{w}_j(\widehat{\kappa}) \rVert < \alpha^{-\ell_j^\circ(\widehat{\kappa}) + \left( \left\lVert D_j^{(0)} \right\rVert - \sqrt{ \frac{2\mathcal{L}(0)}{n}} \right) \widehat{\kappa}}. 
\end{equation}

Then, \cref{eq:basis_case_alignment_bound_general} gives that for $\widehat{\kappa}$ sufficiently small, $\alpha$ may be chosen sufficiently small that for any time $t \in [\widehat{\kappa}, \kappa]$

\begin{equation}
     \left\lVert D_j^{(0)} \right\rVert \left( 1 - \alpha^{\left\lVert D_j^{(0)} \right\rVert \widehat{\kappa} / 2} \right) < \widetilde{\overline{w}}_j^\top(t) D_j^{(0)} \leq \left\lVert D_j^{(0)} \right\rVert. \label{basis_widehat_kappa_alignment_bound}
\end{equation}

Note that by definition of $\tau$, $|h_{\widetilde{\theta}}(x_i)| < m \alpha$ for all $i \in [n]$. Putting this back into \cref{basis_norm_growth_eqn}, we obtain as an upper bound that 

\begin{align}
    \lVert \widetilde{w}_j(\kappa) \rVert &< \lVert \widetilde{w}_j(\widehat{\kappa}) \rVert \alpha^{- \left( \left\lVert D_j^{(0)} \right\rVert +  m\alpha \right)(\kappa-\widehat{\kappa}) } \\
    &< \alpha^{-\ell_j^\circ(\kappa) + \left( \left\lVert D_j^{(0)} \right\rVert -  \sqrt{\frac{2\mathcal{L}(0)}{n}} \right) \widehat{\kappa} - m \alpha(\kappa-\widehat{\kappa})} \\
    &< \alpha^{-\ell_j^\circ(\kappa) + \left( \left\lVert D_j^{(0)} \right\rVert - 2  \sqrt{\frac{2\mathcal{L}(0)}{n}} \right) \widehat{\kappa}}.
\end{align}

For the lower bound, we get 

\begin{align}
    \lVert \widetilde{w}_j(\kappa) \rVert &> \lVert \widetilde{w}_j(\widehat{\kappa}) \rVert \exp_\alpha \left( - \left( \left\lVert D_j^{(0)} \right\rVert \left( 1 - \alpha^{\left\lVert D_j^{(0)} \right\rVert \kappa' / 2} \right) -  m\alpha \right)(\kappa-\widehat{\kappa}) \right) \\
    &> \exp_\alpha \left( -\ell_j^\circ(\kappa) + \left( \left\lVert D_j^{(0)} \right\rVert +  \sqrt{\frac{2\mathcal{L}(0)}{n}} \right) \widehat{\kappa} + \left( \left\lVert D_j^{(0)} \right\rVert \alpha^{\left\lVert D_j^{(0)} \right\rVert \kappa' / 2} + m \alpha \right) (\kappa-\widehat{\kappa}) \right) \\
    &> \alpha^{-\ell_j^\circ(\kappa) + \left( \left\lVert D_j^{(0)} \right\rVert + 2  \sqrt{\frac{2\mathcal{L}(0)}{n}} \right) \widehat{\kappa}},
\end{align}

where the final inequality is valid for $\alpha$ sufficiently small. Since $\widehat{\kappa}$ is arbitrary, we conclude that given any $\xi>0$, we may select $\alpha$ sufficiently small that the deviation of the neurons' norms from those predicted by the \cref{alg:s.t.s} does not exceed $\xi$.
\end{subproof}

\begin{subproof}[Proof of \Cref{lem1.3}] We compute that for some $y_i<0$ such that $\widetilde{\overline{w}}^\top_j(0) x_i > 0 $, and for $t \in [0, \kappa]$,

\begin{align}
    \frac{1}{\log(1/\alpha)} \frac{\df}{\df t} \widetilde{\overline{w}}^\top_j x_i  &= \fD_j^\top x_i -\fD_j^\top \widetilde{\overline{w}}_j \widetilde{\overline{w}}^\top_j x_i  \\
    &=\frac{1}{n} (y_i - h_{\widetilde{\theta}}(x_i)) + \frac{1}{n} \sum_{v \, \mid \, \widetilde{w}^\top_j x_i >0} (h_{\widetilde{\theta}}(x_v)-y_v) \widetilde{w}^\top_j x_v \widetilde{\overline{w}}^\top_j x_i\\ \label{deactivation1}
    &\leq \frac{1}{n} y_i -\frac{1}{n}  \sum_{\substack{v \, \mid \, y_v<0, \\ \widetilde{\overline{w}}^\top_j x_v >0}} y_v \widetilde{\overline{w}}^\top_j x_v \widetilde{\overline{w}}^\top_j x_i + 2m\alpha. 
\end{align}

The is obtained by plugging in the bound on the network output. This bound also implies that for $\alpha$ sufficiently small, $\iind{y_v>0}(y_v-h_{\theta^t}(x_v))>0$, which allowed us to discard the sum over positively labeled data. Since neurons are aligned by time $\widehat{\kappa},$ we conclude using \cref{basis_widehat_kappa_alignment_bound} and \cref{alignment_observation} that 

\begin{equation}
    \frac{\df}{\df t} \widetilde{\overline{w}}^\top_j x_i \leq \frac{\log(1/\alpha)}{n} y_i + \sqrt{2} \alpha^{\widehat{\kappa} \left\lVert D_j^{(0)} \right\rVert / 4},
\end{equation}
where this holds for $\alpha$ sufficiently small and any $t \in [\widehat{\kappa}, \kappa]$. Therefore, by time $2 \widehat{\kappa}$ and for $\alpha$ sufficiently small, $\widetilde{\overline{w}}^\top_j x_i = 0$ for all $i \notin I_{s_j}$. 
\end{subproof}
This completes the proof of \Cref{induction_basis_case_lemma}.
\end{proof}

\label{app:autonomous}
From \Cref{induction_basis_case_lemma} and \Cref{ass:dead}, we conclude that given any $\kappa>0$, there exists $\alpha^*$ sufficiently small that $\theta_+ := (a_i,w_j)_{j\in J_+}$ satisfies 

\begin{equation}
    \dot{\theta}_+(t) \in -\partial \mathcal{L}_+(\theta_+(t)) \text{ for almost all } t \geq \kappa,
\end{equation}

where 

\begin{equation}
    \mathcal{L}_+(\theta_+) \coloneqq \frac{1}{2 n} \sum_{i \in I_{+}} (h_{\theta_+}(x_i) -y_i)^2.
\end{equation}

By the symmetry of the signs, an analogous statement is true of $\theta_- := (a_i,w_j)_{j\in J_-}$. Therefore, on the interval $[\kappa, \infty)$ neurons in the set $J_+$ evolve independently from those in the set $J_-$ according to the autonomous systems just described above. Consequently, we can treat the systems $\theta_+$ and $\theta_-$ independently from now on, which we do by assuming w.l.o.g. $s_j = 1$ for all $j \in [m]$ and $\sgn{y_i} = 1$ for all $i \in [n]$. We note that with this assumption, $A_{i,j} = \iind{\widetilde{w}_j^\top(0) x_i > 0}$.

Similar autonomous systems arguments have been used in \citet[][Appendix F.3]{boursier2025simplicity}, but are even simpler here, due to the orthogonal data assumption.

\subsection{Phase 1: Slow growth} \label{phase_1}

In this stage, we verify that up to a time arbitrarily close to the next jump time, $t_{k+1}$, all neurons in the set $N_U^{(k)}$ are still small and aligned in the directions $D_j^{(k)}$, and all neurons in the set $N_F^{(k)}$ have not changed significantly. For $\varepsilon>0$, define the time

\begin{align}
    \tau_1 := \inf \Big\{ t & \geq \kappa \,\big\vert \;  \exists \, j \in [m] \; \mathrm{s.t.} \; \widetilde{\overline{w}}_j^\top \overline{D}_j^{(k)} \leq 1 - \alpha^{\varepsilon/2} \label{phase_1_alignment_defn}  \\
    & \;\;\;\;\;\, \mathrm{or} \; \exists \, i \in S_F^{(k)} \; \mathrm{s.t.} \; | h_{\widetilde{\theta}}(x_i) - y_i | \, \geq \alpha^{\varepsilon} \label{phase_1_residual_defn}\\ 
    & \;\;\;\;\;\, \mathrm{or} \; \exists \, j \in [m] \; \mathrm{s.t.} \;\lVert \widetilde{w}_{j} \rVert > 2 \max_{j \in [m]} \left\lVert D_j^{(k)} \right\rVert \label{phase_1_lipschitz_defn} \\ 
    & \;\;\;\;\;\, \mathrm{or} \; \exists \, j \in N_U^{(k)} \; \mathrm{s.t.} \;\lVert \widetilde{w}_{j} \rVert \geq \alpha^{\varepsilon} \label{phase_1_growth_defn} \Big\} \bigwedge t_{k+1}.
\end{align}

The condition $\lVert \widetilde{w}_{j} \rVert \leq 2 \lVert D_j^{(k)} \rVert$ is simply used to ensure that our analysis is restricted to a domain on which $\widetilde{\theta}$ is Lipschitz continuous, as this will allow us to apply Rademacher's theorem to conclude certain functions thereof are differentiable almost everywhere. Our analysis will actually show that this condition does not break, meaning that for $\alpha$ sufficiently small, the gradient flow trajectory on $[\kappa, \tau_1]$ is restricted to the compact set $\left\{ ( \widetilde{w}_j )_{j \in [m]} \; \big\vert \; \lVert \widetilde{w}_j \rVert \leq 2 \left\lVert D_j^{(k)} \right\rVert \text{ for all } j \in [m]\right\}$.

In the result below and throughout the account of subsequent phases, we use the notation $j^\star$ in place of $j_\star^{(k)}$ for the neuron which grows to constant order when passing from saddle $k$ to saddle $k+1$ to improve readability. Assuming \Cref{inductive_hypothesis}, we will prove in the lemma below that  on the interval $[\kappa, \tau_1]$, the alignment of all neurons remains stable and the loss on fitted data does not increase significantly. We will also show that provided $\varepsilon$ is selected sufficiently small we have $\lVert \widetilde{w}_{j^\star}(\tau_1) \rVert = \alpha^{\varepsilon}$, meaning neuron $j^\star$ is the largest of all of the unfitted neurons at time $\tau_1$. We will also prove that as $\varepsilon$ decreases, $\tau_1$ approaches $t_{k+1}$. 

\begin{lem}\label{lemma_phase_1} 
    Suppose that \cref{inductive_hypothesis} holds. Then for all $\kappa \in (t_k, t_{k+1})$, there exists $\varepsilon^*$ such that for all $\varepsilon \in (0, \varepsilon^*)$ there exists $\alpha^*>0$ such that for all $\alpha \in (0, \alpha^*)$ the following are true:
    \begin{enumerate}
        \item $\forall i \in S_F^{(k)}, \forall t \in [\kappa, \tau_1],\; | h_{\widetilde{\theta}}(x_i) - y_i | \, < \alpha^{\varepsilon}.$ \label{lemma_phase_1_residual}
        \item $ \forall j \in [m], \forall t \in [\kappa, \tau_1], \; \widetilde{\overline{w}}_j^\top \overline{D}_j^{(k)} > 1 - \alpha^{\varepsilon/2} .$ \label{lemma_phase_1_alignment}
        \item $ \forall j \in N_F^{(k)}, \forall t \in [\kappa, \tau_1], \; \sqrt{n\left\lVert D_j^{(k)} \right\rVert} - \alpha^{\varepsilon/4} < \lVert \widetilde{w}_j \rVert < \sqrt{n\left\lVert D_j^{(k)} \right\rVert} + \alpha^{\varepsilon/4} .$ \label{lemma_phase_1_fitted_norms} 
        \item $  t_{k+1}  - \frac{3 \varepsilon}{2 \left\lVert D_{j^\star}^{(k)} \right\rVert}< \tau_1 < t_{k+1} - \frac{\varepsilon}{2 \left\lVert D_{j^\star}^{(k)} \right\rVert} .$ \label{lemma_phase_1_duration} 
        \item $ \lVert \widetilde{w}_{j^\star}(\tau_1) \rVert = \alpha^{\varepsilon}.$ \label{lemma_phase_1_j*_biggest} 
        \item $\forall j \in N_U^{(k)}, \, \forall t \in [\kappa, \tau_1], \; \left\vert \ell_j^\circ(t) +  \log_{\alpha} ( \lVert \widetilde{w}_j \rVert )\right\vert < \frac{\varepsilon}{2}$. \label{lemma_phase_1_unfitted_norms} 
    \end{enumerate}
\end{lem}

\begin{proof}
\begin{subproof}[Proof of \Cref{lemma_phase_1_residual} \Cref{lemma_phase_1}] Let $\delta$ be such that the conditions of \cref{inductive_hypothesis} hold. In order to check that $| h_{\widetilde{\theta}}(x_i) - y_i | < \alpha^{\varepsilon },$ we will bound the growth in norm of the quantity 

\label{Ek_defn}
\begin{equation} 
    E^{(k)} = - \frac{1}{n} \sum_{i \in S_F^{(k)}} \left( h_{\widetilde{\theta}}(x_i) - y_i \right) x_i . 
\end{equation}

To begin this section, it will be useful to note now that

\begin{align}
    \lVert \fD_j \rVert &\leq \frac{1}{n} \sqrt{\left( \sum_{i \in [n]} ( h_{\widetilde{\theta}}(x_i) - y_i)^2 \right)} \\
    &\leq \sqrt{\frac{2 \mathcal{L}(0)}{n}} \label{fd_bound},
\end{align}

where we have applied the Cauchy-Schwarz inequality and the monotonicity of the gradient flow. We will use this inequality repeatedly. We compute that 

\begin{align}
    \frac{1}{2 \log(1/\alpha) } \frac{\df}{\df t} \left\lVert E^{(k)} \right\rVert^2 & =   \frac{1}{\log(1/\alpha) } E^{(k)\ \top} \ \frac{\df}{\df t} E^{(k)}\\
    & = - \frac{1}{n} \sum_{i \in S_F^{(k)}} \sum_{j \in [m]} A_{i, j} \lVert \widetilde{w}_j \rVert^2 \left( \left( x_i^\top E^{(k)} \right)^2 + \fD_j^\top \widetilde{\overline{w}}_j \widetilde{\overline{w}}_j^\top x_i x_i^\top E^{(k)} \right) \label{derivative_of_residual_equation} \\
    &\leq - \frac{1}{n} \sum_{i \in S_F^{(k)}} \sum_{j \in [m]} A_{i, j} \lVert \widetilde{w}_j \rVert^2 \fD_j^\top \widetilde{\overline{w}}_j \widetilde{\overline{w}}_j^\top x_i x_i^\top E^{(k)} \\
    &\leq - \frac{1}{n} \sum_{i \in S_F^{(k)}} \sum_{j \in N_U^{(k)}} A_{i, j} \lVert \widetilde{w}_j \rVert^2 \fD_j^\top \widetilde{\overline{w}}_j \widetilde{\overline{w}}_j^\top x_i x_i^\top E^{(k)}, \label{derivative_of_residual_basic_bound} 
\end{align}

where the final inequality follows from the fact that for any $j \in N_F^{(k)}$, the sum may be written as

\begin{align}
    - \frac{\lVert \widetilde{w}_j \rVert^2 }{n^3} \left( \sum_{v \in [n]} A_{v, j}(h_{\widetilde{\theta}}(x_v)-y_v) \widetilde{\overline{w}}_j^\top x_v \right) \left(  \sum_{i \in S_F^{(k)}} A_{i, j} (h_{\widetilde{\theta}}(x_i)-y_i) \widetilde{\overline{w}}_j^\top x_i \right) \leq 0. \label{residual_sum_over_fitted_negative}
\end{align}

The above inequality follows from noting that, since we are considering some $j \in N_F^{(k)},$ $\{ v \in [n] \mid A_{v, j} = 1\} = \{ i \in S_F^{(k)} \mid A_{i, j} = 1\}$ (such neurons are by definition active only on fitted data) so this is a square. We have therefore only to obtain a bound on the contribution of those $j \in N_U^{(k)}$. Note that for any $j \in [m],$ the fact that $\widetilde{\overline{w}}_j$ is a unit vector combined with \eqref{fd_bound} implies 

\begin{align}
    | \widetilde{\overline{w}}_j^\top \fD_j | &\leq \lVert \fD_j \rVert \leq \sqrt{\frac{2\mathcal{L}(0)}{n}}.
\end{align}

Also, we have for any $j \in N_U^{(k)}$ by \eqref{phase_1_alignment_defn} and \cref{alignment_observation} that $\widetilde{\overline{w}}_j^\top x_i < \sqrt{2} \alpha^{\varepsilon/4}$ on $[\kappa, \tau_1]$. We combine these observations with \eqref{phase_1_growth_defn} to obtain

\begin{align}
    \frac{1}{\log(1/\alpha) } E^{(k)\ \top} \ \frac{\df}{\df t} E^{(k)} &\leq \frac{1}{n} \sum_{i \in S_F^{(k)}} m \sqrt{\frac{4 \mathcal{L}(0)}{n}} \alpha^{9 \varepsilon / 4} \left\vert x_i^\top E^{(k)} \right\vert \\
    &\leq \frac{2 m}{n} \sqrt{ \mathcal{L}(0)} \alpha^{9 \varepsilon / 4} \left\lVert E^{(k)} \right\rVert,
\end{align}

where we have used the Cauchy-Schwarz inequality in the final line. On a compact domain, $\partial \mathcal{L}$ is bounded. Therefore, the function $t \mapsto \widetilde{\theta}(t)$ is Lipschitz continuous on the domain specified in the definition of $\tau_1$ for any fixed $\alpha$. Since the composition of Lipschitz continuous functions is also Lipschitz continuous, $\left\lVert E^{(k)} \right\rVert$ is Lipschitz continuous on $[\kappa, \tau_1]$. Therefore, it is differentiable almost everywhere on $[\kappa, \tau_1]$, and furthermore has bounded subdifferential at every point in $[\kappa, \tau_1]$. In subsequent arguments when we appeal briefly to Rademacher's theorem for differentiability, it is implicit that this argument has been used. This affords us the bound 

\begin{align}
    \frac{\df}{\df t} \left\lVert E^{(k)} \right\rVert &\leq \frac{2m}{n} \sqrt{\mathcal{L}(0)} \log(1/\alpha) \alpha^{9 \varepsilon / 4} \\
    &< \alpha^{2 \varepsilon } ,
\end{align}
for $\alpha$ sufficiently small. Since $| h_{\widetilde{\theta}^{\kappa}}(x_i) - y_i | \leq \alpha^{\delta}$ by \Cref{IH_residual} \Cref{inductive_hypothesis}, we have $\left\lVert E^{(k)}(\kappa) \right\rVert < \alpha^\delta/\sqrt{n}$. Integrating and using this initial condition, we have 

\begin{align}
    \left\lVert E^{(k)} \right\rVert &< \left\lVert E^{(k)}(\kappa) \right\rVert + (t-\kappa) \alpha^{2 \varepsilon} \\
    &< \frac{\alpha^{\delta}}{\sqrt{n}} + (t_{k+1}-\kappa) \alpha^{2 \varepsilon} \\
    &< \alpha^{ 3 \varepsilon / 2},
\end{align}

for $\varepsilon<\frac{\delta}{2}$ and $\alpha\leq \frac{1}{(1+t_{k+1}-t_k)^{2/\varepsilon}}$ . This implies that

\begin{align}
    \max_{i \in S_F^{(k)}} | h_{\widetilde{\theta}}(x_i) - y_i | &< n \alpha^{3 \varepsilon / 2} \\
    &< \alpha^{\varepsilon},
\end{align} 

for all $t \in [\kappa, \tau_1]$, concluding the proof of this item.
\end{subproof}

\begin{subproof}[Proof of \Cref{lemma_phase_1_alignment} \Cref{lemma_phase_1}]

We verify now the stability of alignment of the neurons. We split this into the case $j \in N_F^{(k)}$ and the case $j \in N_U^{(k)}$. For the former, first note that for $j \in N_F^{(k)}$, $\{i \in S_U^{(k)} \mid A_{i, j} = 1\} = \emptyset$ by definition. Consequently,

\begin{align}
    \big\lVert \fD_j \big\rVert < \frac{1}{\sqrt{n}} \alpha^\varepsilon .
\end{align}

Inserting this into the differential equation governing tangential dynamics, we get

\begin{align}
    \frac{1}{\log(1/\alpha)} \frac{\df}{\df t} \widetilde{\overline{w}}_j^\top D_j^{(k)} &= \fD_j^\top D_j^{(k)} - \fD_j^\top \widetilde{\overline{w}}_j \widetilde{\overline{w}}_j^\top D_j^{(k)} \\
    &\geq -2 \left\lVert D_j^{(k)} \right\rVert \big\lVert \fD_j \big\rVert \\
    &> - \frac{2 \left\lVert D_j^{(k)} \right\rVert}{\sqrt{n}} \alpha^{\varepsilon} \label{phase_1_alignment_derivative_bound_fitted},
\end{align} 

We simply integrate, using \Cref{IH_alignment} \Cref{inductive_hypothesis} for the alignment value at time $\kappa$, to obtain 

\begin{align}
    \widetilde{\overline{w}}_j^\top D_j^{(k)} &>  \left\lVert D_j^{(k)} \right\rVert \left(1-\alpha^\delta \right) - \frac{2 \left\lVert D_j^{(k)} \right\rVert}{\sqrt{n}} (t_{k+1}-\kappa) \log(1/\alpha) \alpha^{\varepsilon} \\
    &>  \left\lVert D_j^{(k)} \right\rVert \left(1-\alpha^{\varepsilon/2} \right),
\end{align} 

where the final inequality is again valid for $\varepsilon$ and $\alpha$ sufficiently small. For the case $j \in N_U^{(k)}$, we first use \cref{phase_1_growth_defn} to bound the network output on data $i \in S_U^{(k)}$:

\begin{align}
    h_{\widetilde{\theta}}(x_i) &= \sum_{j \in [m]} \lVert \widetilde{w}_j \rVert^2 \left( \widetilde{\overline{w}}_j^\top x_i \right)_+ \\
    &< m \alpha^{2 \varepsilon}. \label{phase_1_network_bound}
\end{align}

Using this, we estimate the alignment of an arbitrary neuron $j \in N_U^{(k)}$ with direction $D_j^{(k)}$ as follows: 

\begin{align}
    \frac{1}{\log(1/\alpha)}\frac{\df}{\df t}  \widetilde{\overline{w}}^\top_j D_j^{(k)} &= \left\lVert D_j^{(k)} \right\rVert^2 - \left( \widetilde{\overline{w}}^\top_j D_j^{(k)} \right)^2\\
    &\qquad - \frac{1}{n} \sum_{ i \in S_j^{(k)}} h_{\widetilde{\theta}}(x_i) \left( x_i^\top D_j^{(k)} - x^\top_i \widetilde{\overline{w}}_j \widetilde{\overline{w}}^\top_{j} D_j^{(k)} \right) \\ 
    &\qquad - \frac{1}{n} \sum_{i \in S_F^{(k)}} A_{i,j} ( h_{\widetilde{\theta}}(x_i) - y_i) \left( x_i^\top D_j^{(k)} - x^\top_i \widetilde{\overline{w}}_j \widetilde{\overline{w}}^\top_{j} D_j^{(k)} \right)  \\
    &> \left\lVert D_j^{(k)} \right\rVert^2 - \left( \widetilde{\overline{w}}^\top_j D_j^{(k)} \right)^2 - 2 \left\lVert D_j^{(k)} \right\rVert m \alpha^{9 \varepsilon / 4} - 2 \left\lVert D_j^{(k)} \right\rVert \alpha^{5 \varepsilon / 4} \\
    &> \left\lVert D_j^{(k)} \right\rVert^2 - \left( \widetilde{\overline{w}}^\top_j D_j^{(k)} \right)^2 - \alpha^{6 \varepsilon / 5} \label{phase_1_alignment_args}.
\end{align}

The first inequality is an application of \Cref{alignment_corollary} to both sums (given \eqref{phase_1_alignment_defn}) as well as using \eqref{phase_1_network_bound} for the sum over $i \in S_j^{(k)}$ and \eqref{phase_1_residual_defn} for the sum over $i \in S_F^{(k)}$. The second inequality follows for $\alpha$ sufficiently small. 
The above implies that the solutions of the differential equation $f'(t)=b(a^2-f(t)^2)$, where $a^2 = \left\lVert D_j^{(k)} \right\rVert^2 - \alpha^{6 \varepsilon / 5}$ and $b = \log(1/\alpha)$, furnish us with a lower bound on the alignment. Solutions of this differential equation approach a limiting value of $a$ monotonically. Since $a > \left\lVert D_j^{(k)} \right\rVert - \alpha^{3 \varepsilon / 5}$ from the inequality $\sqrt{1-x^2}>1-x$, and \Cref{IH_alignment} \Cref{inductive_hypothesis} allows us to lower bound the alignment value at time $\kappa$, we have that for all $t \in [\kappa, \tau_1]$,

\begin{equation}
    \widetilde{\overline{w}}^\top_j D_j^{(k)} > \left\lVert D_j^{(k)} \right\rVert \min \left( 1 - \alpha^\delta, 1 - \alpha^{3 \varepsilon / 5} \right) > \left\lVert D_j^{(k)} \right\rVert \left( 1 - \alpha^{\varepsilon/2} \right),
\end{equation}

for $\varepsilon$ and $\alpha$ sufficiently small.
\end{subproof}

\begin{subproof}[Proof of \Cref{lemma_phase_1_fitted_norms} \Cref{lemma_phase_1}]

We check that all the neurons in the set $N_F^{(k)}$ cannot change in norm significantly in the first phase. Recall that for $j \in N_F^{(k)}$, $\{i \in S_U^{(k)} \mid A_{i,j} = 1\} = \emptyset$. Therefore, the differential equation governing radial dynamics admits the bound 

\begin{align}
    \frac{1}{\log(1/\alpha)} \frac{\df}{\df t} \lVert \widetilde{w}_j \rVert &= \fD_j^\top \widetilde{w}_j \\
    &\geq - \left\lVert E^{(k)} \right\rVert \lVert \widetilde{w}_j \rVert. 
\end{align}

Applying Gr{\"o}nwall's inequality combined with the initial conditions from \Cref{IH_fitted_norms} \Cref{inductive_hypothesis}, we get that for any $t \in [\kappa, \tau_1]$ and $j \in N_U^{(k)}$,

\begin{align}
    \lVert \widetilde{w}_j \rVert &\geq \lVert \widetilde{w}_j(\kappa)\rVert \exp \left( -\log(1/\alpha) \int_{\kappa}^{\tau_1} \left\lVert E^{(k)} \right\rVert ds \right) \\
    &>\lVert \widetilde{w}_j(\kappa)\rVert \exp \left( -\log(1/\alpha) \frac{\alpha^\varepsilon}{\sqrt{n}} (t_{k+1} - \kappa) \right) \\
    &>\lVert \widetilde{w}_j(\kappa)\rVert \exp \left( - \alpha^{\varepsilon/2} \right) \\
    &\geq \lVert \widetilde{w}_j(\kappa)\rVert \left( 1 - \alpha^{\varepsilon/2} \right) \\
    &> \sqrt{ n \left\lVert D_j^{(k)} \right\rVert - \alpha^{\delta}} \left( 1 - \alpha^{\varepsilon/2} \right) \\
    &> \sqrt{ n \left\lVert D_j^{(k)} \right\rVert} -\alpha^{\varepsilon/4}. \label{phase_1_initial_lower_bound_fitted_norms}
\end{align}

The above inequalities are valid for $\varepsilon$ and $\alpha$ sufficiently small. An analogous argument yields the claimed upper bound 

\begin{equation}
    \lVert \widetilde{w}_j(t)\rVert < \sqrt{ n \left\lVert D_j^{(k)} \right\rVert} + \alpha^{\varepsilon/4}.
\end{equation}

Note that this implies that for $\alpha$ sufficiently small, for all $j \in N_F^{(k)}$, $\lVert \widetilde{w}_j(t)\rVert \leq 2 \left\lVert D_j^{(k)} \right\rVert$. Combining this with \eqref{phase_1_growth_defn}, we see that \eqref{phase_1_lipschitz_defn} does not break at time $\tau_1$.
\end{subproof}

\begin{subproof}[Proof of \Cref{lemma_phase_1_unfitted_norms} \Cref{lemma_phase_1}] To estimate the growth of any neuron $j \in N_U^{(k)}$, we compute that 

\begin{align}
    \frac{1}{\log(1/\alpha)} \frac{\df}{\df t} \lVert \widetilde{w}_j \rVert &= - \frac{\lVert w_j \rVert}{n} \sum_{i \in [n]} A_{i, j} (h_{\widetilde{\theta}}(x_i)-y_i) \widetilde{\overline{w}}_j^\top x_i \\
    &> \left( \left\lVert D_j^{(k)} \right\rVert - \alpha^{\varepsilon / 2} - m \alpha^{2 \varepsilon}\right) \lVert \widetilde{w}_j \rVert - \frac{\lVert w_j \rVert}{n} \sum_{i \notin S_{j}^{(k)}} A_{i,j} (h_{\widetilde{\theta}}(x_i)-y_i) \widetilde{\overline{w}}_j^\top x_i  \\
    &> \left( \left\lVert D_j^{(k)} \right\rVert - \alpha^{\varepsilon / 2} - m \alpha^{2 \varepsilon} \right) \lVert \widetilde{w}_j \rVert - \sqrt{2} \alpha^{5\varepsilon/4} \lVert \widetilde{w}_j \rVert. \label{phase_1_initial_norm_lower_bound}
\end{align}

The first inequality is from the assumption on alignment, \eqref{phase_1_alignment_defn}, combined with the assumed bound on the norms of the unfitted neurons, \eqref{phase_1_growth_defn}. The second is the combination of this assumed alignment with \eqref{alignment_observation} and \eqref{phase_1_residual_defn}. The above implies that for $\alpha$ sufficiently small, 

\[
\frac{\df}{\df t} \lVert \widetilde{w}_j \rVert > \log(1/\alpha) \left( \left\lVert D_j^{(k)} \right\rVert - 2 \alpha^{\varepsilon/2} \right) \lVert \widetilde{w}_j \rVert.
\]

Gr{\"o}nwall's inequality then gives that for $t \in [\kappa, \tau_1]$,

\begin{align}
    \lVert \widetilde{w}_j \rVert &> \lVert \widetilde{w}_j(\tau_1) \rVert \exp \left( \log(1/\alpha) \left( \left\lVert D_j^{(k)} \right\rVert - 2 \alpha^{\varepsilon/2} \right) (t-\kappa) \right) \\
    &= \lVert \widetilde{w}_j(\tau_1) \rVert \alpha^{ - \left\lVert D_j^{(k)} \right\rVert (t - \kappa) + 2\alpha^{\varepsilon/2} (t-\kappa)} \\
    &\overset{(1)}{>} \alpha^{-\ell_j^\circ(\kappa) - \left\lVert D_j^{(k)} \right\rVert (t - \kappa) + \varepsilon/3 + 2\alpha^{\varepsilon/2} (t_{k+1}-\kappa)} \\
    &\overset{(2)}{>} \alpha^{-\ell_j^\circ(\kappa) - \left\lVert D_j^{(k)} \right\rVert (t - \kappa) + \varepsilon/2} \\
    &=  \alpha^{-\ell_j^\circ(t) + \varepsilon/2} \label{phase_1_arbitrary_lower_bound}.
\end{align}

In the above, (1) is from an application of \Cref{IH_unfitted_norms} \Cref{inductive_hypothesis} with the choice $\xi = \varepsilon/3$ and (2) is valid for $\alpha$ sufficiently small. For the upper bound, we compute in a similar manner that for any $j \in N_U^{(k)}$,

\begin{align}
    \frac{1}{\log(1/\alpha)} \frac{\df}{\df t} \lVert \widetilde{w}_j \rVert &= - \frac{\lVert w_j \rVert}{n} \sum_{i \in [n]} A_{i, j} (h_{\widetilde{\theta}}(x_i)-y_i) \widetilde{\overline{w}}_j^\top x_i \\
    &< \left( \left\lVert D_j^{(k)} \right\rVert + \sqrt{2} \alpha^{5\varepsilon/4} \right) \lVert \widetilde{w}_j \rVert.
\end{align}

Again by Gr{\"o}nwall's inequality combined with the initial conditions from \cref{inductive_hypothesis} with $\xi=\varepsilon/3$, we have the bound

\begin{align}
    \lVert \widetilde{w}_j \rVert &< \alpha^{-\ell_j^\circ(\kappa) - \left\lVert D_j^{(k)} \right\rVert (t-\kappa) - \varepsilon/3 - \sqrt{2} \alpha^{5 \varepsilon / 4} (t_{k+1}-\kappa)} \\
    &< \alpha^{ - \ell_j^\circ(t) - \varepsilon/2 }  \label{phase_1_fitting_norm_upper_bound_tau1} . 
\end{align} \qedhere

\end{subproof}

\begin{subproof}[Proof of \Cref{lemma_phase_1_duration,lemma_phase_1_j*_biggest} \Cref{lemma_phase_1}]

Note that for $t, t' \in [t_k, t_{k+1}]$, $-\ell_j^\circ(t) - \left\lVert D_j^{(k)} \right\rVert (t'-t) = - \ell_j^\circ(t')$. Using this fact and that $\lVert \widetilde{w}_j \rVert < \varepsilon$ on $[\kappa, \tau_1)$, we obtain that for all $j \in N_U^{(k)}$, 

\begin{align}
    \varepsilon &< -\ell_j^\circ(\tau_1) + \frac{\varepsilon}{2} \\
    &= -\ell_j^\circ(t_{k+1}) - \left\lVert D_j^{(k)} \right\rVert (\tau_1 - t_{k+1}) + \frac{\varepsilon}{2} .
\end{align}

We have $-\ell_{j^\star}(t_{k+1})=0$, and hence the upper bound $\tau_1$ of

\begin{align}
    \tau_1 &< t_{k+1} - \frac{\varepsilon}{2 \left\lVert D_{j^\star}^{(k)} \right\rVert}. \label{phase_1_min_dist_tk+1}
\end{align}

To verify that neuron $j^\star$ is the largest at time $\tau_1$ note the upper bound on the neurons' norms provided by \eqref{phase_1_fitting_norm_upper_bound_tau1} is monotone increasing in $t$. From this observation followed by inserting the upper bound on $\tau_1$ derived in \cref{phase_1_min_dist_tk+1}, we get

\begin{align}
    \lVert \widetilde{w}_j(\tau_1) \rVert &< \exp_\alpha \left( - \ell_j^\circ(\tau_1) - \frac{\varepsilon}{2} \right) \label{phase_1_unfitted_norm_upper_bound_tau1_simple} \\
    &<\exp_\alpha \left( -\ell_j^\circ(t_{k+1}) + \frac{1}{2} \left( \frac{\left\lVert D_j^{(k)} \right\rVert}{\left\lVert D_{j^\star}^{(k)} \right\rVert} - 1 \right) \varepsilon \right). \label{phase_1_unfitted_norm_upper_bound_tau1}
\end{align}

Note that $-\ell_{j}^\circ(t_{k+1})>0$ for $j \in N_U^{(k)} \setminus \{ j^\star \}$. This implies that for such $j$ the exists $\alpha$ sufficiently small that 

\begin{equation}
    -\ell_j^\circ(t_{k+1}) + \frac{1}{2} \left( \frac{\left\lVert D_j^{(k)} \right\rVert}{\left\lVert D_{j^\star}^{(k)} \right\rVert} - 1 \right) \varepsilon > \varepsilon.
\end{equation}

Hence, $\lVert \widetilde{w}_j(\tau_1) \rVert < \alpha^{\varepsilon}$ for $j \in N_U^{(k)} \setminus \{ j^\star \}$. This proves that no neuron other than $j^\star$ could have $\lVert \widetilde{w}_j(\tau_1) \rVert = \alpha^\varepsilon$ for $\varepsilon$ sufficiently small, and since no other condition breaks we have $\lVert \widetilde{w}_{j^\star}(\tau_1) \rVert = \alpha^{\varepsilon}$. We obtain a lower bound on $\tau_1$ by considering that 

\begin{equation}
     \varepsilon > -\left\lVert D_{j^\star}^{(k)} \right\rVert (\tau_1-t_{k+1}) - \frac{\varepsilon}{2} 
\end{equation}

imposes the condition 

\begin{align}
    \tau_1 &> t_{k+1} - \frac{3 \varepsilon}{2 \left\lVert D_{j^\star}^{(k)} \right\rVert}. \label{phase_1_max_dist_tk+1}
\end{align}

Note that by selecting $\varepsilon$ sufficiently small depending on $\kappa$, we can  indeed ensure that $\tau_1>\kappa$.
\end{subproof}

This completes the proof of \Cref{lemma_phase_1}.
\end{proof}

\subsection{Phase 2a: Growth to constant order} \label{phase_2a}

Next, we deal with the rapid growth of neuron $j^\star$, the account of which is split into two parts. To begin, it will be important to obtain a uniform bound on the norms of neurons in the set $j \in N_U^{(k)} \setminus \{ j^\star \}$, in much the same way that in the proof of \Cref{lemma_phase_1} we had by definition of $\tau_1$ that $\lVert \widetilde{w}_j \rVert < \alpha^{\varepsilon}$ on $[\kappa, \tau_1]$ or in the proof of \Cref{induction_basis_case_lemma} we had $\lVert \widetilde{w}_j \rVert < \alpha^{1/2}$. The idea of the proof is to verify that, given the upper bound on neurons in the set $N_U^{(k)} \setminus \{ j^\star \}$ at time $\tau_1$, if these neurons were to subsequently grow at the fastest possible rate, they would still not have exceeded some constant threshold by a certain time after $\tau_1$.

\begin{prop} \label{pr:uniform_bound_unfitted}
    For all $\kappa \in (t_k, t_{k+1})$, there exists $\varepsilon^*>0$ such that for all $\varepsilon \in (0, \varepsilon^*)$ there exists $\alpha^*>0$ such that for all $\alpha \in (0, \alpha^*)$ the following holds:
    \begin{equation}
        \forall j \in N_U^{(k)} \setminus \{ j^\star \}, \forall t \in \left[ \tau_1, \tau_1 + 2 \sqrt{\frac{n}{2 \mathcal{L}(0)}} c_U \right], \lVert \widetilde{w}_j \rVert < \alpha^{c_U}, \label{phase_2_unfitted_uniformly_small}
    \end{equation}

    where $c_U := - \frac{1}{4} \max_{j' \in N_U^{(k)} \setminus \{ j^\star \} } \ell_{j'}(t_{k+1})$.
\end{prop}

\begin{proof}

Recall that for any $j \in [m]$,

\begin{equation}
    \frac{1}{\log(1/\alpha)} \frac{\df}{\df t} \lVert \widetilde{w}_j \rVert \leq \sqrt{\frac{2\mathcal{L}(0)}{n}} \lVert \widetilde{w}_j \rVert .
\end{equation}

Recall also the upper bound on neurons in the set $N_U^{(k)} \setminus \{ j^\star \}$ at time $\tau_1$ provided by \eqref{phase_1_unfitted_norm_upper_bound_tau1_simple}. Using Gr{\"o}nwall's inequality we obtain from these observations that on the considered interval,

\begin{align}
    \lVert \widetilde{w}_j \rVert &\leq \lVert \widetilde{w}_j(\tau_1) \rVert \exp \left( \sqrt{\frac{2\mathcal{L}(0)}{n}} \log(1/\alpha) (t-\tau_1) \right) \\
    &< \exp_\alpha \left( -\ell_j^\circ(\tau_1) - \frac{\varepsilon}{2} - \sqrt{\frac{2\mathcal{L}(0)}{n}} (t - \tau_1) \right) \\
    &< \exp_\alpha \left( -\ell_j^\circ(t_{k+1}) - \frac{\varepsilon}{2} - \sqrt{\frac{2\mathcal{L}(0)}{n}} (t - \tau_1) \right) \\
    &\leq \exp_\alpha \left( -\ell_j^\circ(t_{k+1}) - \frac{\varepsilon}{2} + \frac{1}{2} \max_{j' \in N_U^{(k)} \setminus \{j^\star\} } \ell_{j'}(t_{k+1}) \right) \\
    &\leq \exp_\alpha \left( - \frac{1}{2} \max_{j' \in N_U^{(k)} \setminus \{ j^\star \} } \ell_j^\circ(t_{k+1}) - \frac{\varepsilon}{2} \right) \\
    &< \exp_\alpha \left( - \frac{1}{4} \max_{j' \in N_U^{(k)} \setminus \{ j^\star \} } \ell_j^\circ(t_{k+1}) \right).
\end{align}

The final inequality is valid for $\varepsilon$ sufficiently small, proving the claim.
\end{proof}

We will throughout the account of several subsequent results assume that $\varepsilon$ is selected small enough that $\varepsilon < c_U$, meaning that, informally, contributions of order $c_U$ can be thought of as negligible relative to those of order $\varepsilon$. We define the quantity 

\begin{equation}
    \eta := \frac{1}{2} \min \left( \frac{\sqrt{n}}{60 \sqrt{2}} \frac{\left\lVert D_{j^\star}^{(k)} \right\rVert}{\sqrt[4]{2 \mathcal{L}(0)}}, \sqrt{n \left\lVert D_{j^\star}^{(k)} \right\rVert} \right).
\end{equation}

This is a constant threshold which we will use to break the account of the growth of neuron $j^\star$ in two. We do so by defining 

\begin{align}
        \tau_2 := \inf \Big\{ t \geq & \tau_1 \; \big\vert \;  \widetilde{\overline{w}}_{j^\star}^\top \overline{D}_{j^\star}^{(k)} \leq 1 - \alpha^{\varepsilon/12} \label{phase_2a_alignment_defn} \\
        & \;\; \mathrm{or} \; \exists \; i \in S_F^{(k)} \; \mathrm{s.t.} \; | h_{\widetilde{\theta}}(x_i) - y_i | \, \geq \alpha^{\varepsilon/6} \label{phase_2a_residual_defn} \\
        & \;\; \mathrm{or} \; \exists \, j \in [m] \; \mathrm{s.t.} \;\lVert \widetilde{w}_{j} \rVert > 2 \max_{j \in [m]} \left\lVert D_j^{(k)} \right\rVert \label{phase_2a_Lipschitz_defn} \\ 
        & \;\; \mathrm{or} \; \lVert \widetilde{w}_{j^\star} \rVert \geq \eta \Big\} \bigwedge  \left( \tau_1 + 2 \sqrt{\frac{n}{2 \mathcal{L}(0)}} c_U  \right). \label{phase_2a_norm_defn} 
\end{align}

We will prove in the lemma below that $\lVert \widetilde{w}_{j^\star}(\tau_2) \rVert = \eta$, meaning that neuron $j^\star$ has grown to constant order. In order to do so, we will verify that the loss on fitted data does not increase significantly during this phase. As a consequence, the alignment of neurons in the set $N_F^{(k)} \cup \{ j^\star \}$ remains stable. 

The use of the constant threshold $\eta$ is that there is a tradeoff between proving the growth conditions and the stability of the loss on fitted data. Specifically, one must ensure that the growth phase lasts long enough that neuron $j^\star$ can grow to constant order. On the other hand, this increases the time interval on which one must control $\left\lVert E^{(k)} \right\rVert$ and increases the coefficients of terms in \eqref{derivative_of_residual_basic_bound}. Introducing this constant threshold, the specific choice of which simply emerges as a consequence of they way we bound $\left\lVert E^{(k)} \right\rVert$, allows us to balance these two competing aims across both parts of the growth phase. We believe that more delicate arguments when estimating the derivative of the quantity $\max_{i \in S_F^{(k)}} \left( \widetilde{\overline{w}}_{j^\star}^\top x_i \right)_+$, which controls the extent to which neuron $j^\star$ could align in the direction of data which has already been fitted, may allow the account of growth to occur in only only one phase. For instance, the bound when passing to \eqref{phase_2_realign_fitted_data_bound} loses a lot of control. 

The final stopping condition given by the time $\tau_1 + 2 \sqrt{\frac{n}{2\mathcal{L}(0)}} c_U $ is required to ensure, using \Cref{pr:uniform_bound_unfitted}, that we have a uniform bound on all of the neurons in the set $N_U^{(k)} \setminus \{ j^\star \}$ on the interval $[\tau_1, \tau_2]$. That we apply this kind of bound is reflective of the fact we do not consider at all the precise alignment or growth rates of the neurons $N_U^{(k)} \setminus \{ j^\star \}$ during the growth phases; we only ensure that the duration of these phases is sufficiently short that these neurons do not change significantly.

\begin{lem}\label{lemma_phase_2a}

    Suppose the \Cref{inductive_hypothesis} holds. Then for all $\kappa \in (t_k, t_{k+1})$, there exists $\varepsilon^*>0$ such that for all $\varepsilon \in (0, \varepsilon^*)$ there exists $\alpha^*>0$ such that for all $\alpha \in (0, \alpha^*)$ the following are true:

    \begin{enumerate}
        \item $ \tau_2 - \tau_1 < \frac{3 \varepsilon}{\left\lVert D_{j^\star}^{(k)} \right\rVert} .$ \label{lemma_phase_2a_duration}
        \item $ \forall j \in N_F^{(k)}, \, \forall t \in [\tau_1, \tau_2], \; \sqrt{n\left\lVert D_j^{(k)} \right\rVert} - \alpha^{\varepsilon/8} < \lVert \widetilde{w}_j \rVert < \sqrt{n\left\lVert D_j^{(k)} \right\rVert} + \alpha^{\varepsilon/8} .$ \label{lemma_phase_2a_fitted_norms}
        \item $\forall i \in S_F^{(k)}, \, \forall t \in [\tau_1, \tau_2], \; \max_{i \in S_F^{(k)}} \left( \widetilde{\overline{w}}_{j^\star}^\top x_i \right)_+ < \alpha^{\varepsilon/6}.$ \label{lemma_phase_2a_realignment_fitted_data}
        \item $ \forall t \in [\tau_1, \tau_2], \; \widetilde{\overline{w}}_{j^\star}^\top \overline{D}_{j^\star}^{(k)} > 1 - \alpha^{\varepsilon/12} .$ \label{lemma_phase_2a_alignment}
        \item $ \forall i \in S_F^{(k)}, \, \forall t \in [\tau_1, \tau_2], \; | h_{\widetilde{\theta}}(x_i) - y_i | \, < \alpha^{\varepsilon/6} .$ \label{lemma_phase_2a_residual} 
        \item $\lVert \widetilde{w}_{j^\star}(\tau_2) \rVert = \eta$. \label{lemma_phase_2a_jstar_norm}
    \end{enumerate}
\end{lem}

We observe at this point that although \Cref{lemma_phase_2a_alignment} \Cref{lemma_phase_2a} would seem to imply a statement on realignment with fitted data, we are required in the proof to obtain \Cref{lemma_phase_2a_realignment_fitted_data} \Cref{lemma_phase_2a} first in order to prove \Cref{lemma_phase_2a_alignment} \Cref{lemma_phase_2a}. This is an interesting result of the fact that a bootstrapping argument, in which one would assume some level of alignment of neuron $j^\star$ in direction $D_{j^\star}^{(k)}$ and then insert this into \eqref{derivative_of_residual_basic_bound}, leads to too weak a control on $\left\lVert E^{(k)} \right\rVert$ to be able to use this to get the assumed alignment back at the end (and vice versa). This circularity was not an issue in the arguments leading to \eqref{phase_1_alignment_args}, because in this setting we had a the growth conditions $\lVert \widetilde{w}_j \rVert \leq \alpha^{\varepsilon}$ for all $j \in N_U^{(k)}$ on the interval $[\kappa, \tau_1]$. The reader may also note that our statement includes no control on the alignment of neurons in the set $N_F^{(k)}$. This is because we can simply control the loss on fitted data and use this when needed to estimate the alignment of such neurons across several phases at the end (this will ultimately be done in \Cref{lemma_phase_4}). We cannot, on the other hand, avoid estimating the norms of such neurons in every phase we describe, since we must prove the flow is restricted to a bounded domain in order to apply Rademacher's theorem. 

\begin{proof}
We consider the items in the statement of \Cref{lemma_phase_2a}.
\begin{subproof}[Proof of \Cref{lemma_phase_2a_duration}] We bound the phase duration above by examining the growth of neuron $j^\star$. We have that on $[\tau_1, \tau_2]$, 


\begin{align}
    & \frac{1}{2 \log(1/\alpha)} \frac{\df}{\df t} \lVert \widetilde{w}_{j^\star} \rVert^2 \\
    &\qquad = \lVert \widetilde{w}_{j^\star} \rVert^2 \left( - \frac{1}{n} \sum_{i \in S^{(k)}_{F}} A_{i, j^\star}(h_{\widetilde{\theta}}(x_i) - y_i) \widetilde{\overline{w}}_{j^\star}^\top x_i - \frac{1}{n} \sum_{i \in S_{j^\star}^{(k)}} (h_{\widetilde{\theta}}(x_i) - y_i) \widetilde{\overline{w}}_{j^\star}^\top x_i \right) \\
    &\qquad> \lVert \widetilde{w}_{j^\star} \rVert^2 \left( \lVert D^{(k)}_{j^\star}  \rVert (1-\alpha^{\varepsilon/12}) - \alpha^{\varepsilon/6} - \frac{1}{n} \sum_{i \in S_{j^\star}^{(k)}} h_{\widetilde{\theta}}(x_i) \widetilde{\overline{w}}_{j^\star}^\top x_i \right) \label{phase_2a_growth_workings_1} \\
    &\qquad= \lVert \widetilde{w}_{j^\star} \rVert^2 \left( \vphantom{\sum_{ \substack{j \neq j^\star , \\ \widetilde{w}_j^\top x_i > 0} } \lVert \widetilde{w}_{j} \rVert^2 \widetilde{\overline{w}}_{j}^\top x_i}  \lVert D^{(k)}_{j^\star}  \rVert (1-\alpha^{\varepsilon/12}) - \alpha^{\varepsilon/6} \right. \\
    & \hspace{3cm}- \left. \frac{1}{n} \sum_{i \in S_{j^\star}^{(k)}} \left( \lVert \widetilde{w}_{j^\star} \rVert^2 \widetilde{\overline{w}}_{j^\star}^\top x_i + \sum_{j \neq j^\star} A_{i,j} \lVert \widetilde{w}_{j} \rVert^2 \widetilde{\overline{w}}_{j}^\top x_i \right) \widetilde{\overline{w}}_{j^\star}^\top x_i \right) \\
    &\qquad \geq \lVert \widetilde{w}_{j^\star} \rVert^2  \left( \lVert D^{(k)}_{j^\star}  \rVert (1-\alpha^{\varepsilon/12}) - \alpha^{\varepsilon/6} - \frac{\lVert \widetilde{w}_{j^\star} \rVert^2}{n} - m \alpha^{2{c_U}} \right) \label{phase_2a_growth_workings_2}\\
    &\qquad> \lVert \widetilde{w}_{j^\star} \rVert^2  \left(\lVert D^{(k)}_{j^\star} \rVert - \alpha^{\varepsilon  / 13} \right) - \frac{\lVert \widetilde{w}_{j^\star} \rVert^4}{n}. \label{phase_2a_norm_growth_lower_bound}
\end{align}

In the above, \eqref{phase_2a_growth_workings_1} is from the definition of $\tau_2$ (where we have not used the alignment in the estimate of the sum over the fitted data as it is not necessary here) and \eqref{phase_2a_growth_workings_2} is from the fact that for any $i \in S_{j^\star}^{(k)}$, $\{ j \in N_F^{(k)} \mid A_{i,j} = 1 \} = \emptyset$. A function satisfying the differential inequality $f'(t) > 2 \log(1/\alpha) a f(t) - 2 \log(1/\alpha) b f(t)^2 $ has a lower bound

\[
f(t) > \frac{a}{b} \frac{\exp \left( 2 \log (1/\alpha) a (t-C)  \right)}{1 + \exp \left( 2 \log (1/\alpha) a (t-C)  \right)} = \frac{a}{b} \frac{1}{\alpha^{2a(t-C)}+1},
\]

where $C$ is an integration constant. In our case, $f=\lVert \widetilde{w}_{j^\star} \rVert^2$, $a = \left\lVert D_{j^\star}^{(k)} \right\rVert - \alpha^{\varepsilon/13}$, and $b = \frac{1}{n}$. We will obtain an upper bound on this constant $C$, since this will in turn provide a lower bound on $f(t)$ and hence an upper bound on $\tau_2$. By \Cref{lemma_phase_1_j*_biggest} \Cref{lemma_phase_1}, $\lVert \widetilde{w}_{j^\star}(\tau_1) \rVert^2 = \alpha^{2 \varepsilon}$. Simple algebra then gives 

\[
C < \tau_1 - \frac{\log \left( \frac{a}{b} \alpha^{-2 \varepsilon} - 1 \right)}{2a \log(\alpha)}.
\]

This implies that for $t \in [\tau_1, \tau_2]$, 

\[
f(t) > \frac{a}{b} \frac{1}{1 +  \exp_\alpha \left( 2a(t-\tau_1) + \frac{\log \left( \frac{a}{b} \alpha^{-2 \varepsilon} - 1 \right)}{\log(\alpha)} \right) }.
\]

Rearranging 

\[
\frac{a}{b} \frac{1}{1 +  \exp_\alpha \left( 2a(t-\tau_1) + \frac{\log \left( \frac{a}{b} \alpha^{-2 \varepsilon} - 1 \right)}{\log(\alpha)} \right)} < \eta^2
\]

gives 

\begin{align}
    t-\tau_1 &< \frac{1}{2a \log(\alpha)} \left( \log \left(\frac{a}{b \eta^2}-1 \right) - \log \left( \frac{a}{b} \alpha^{-2\varepsilon} - 1 \right) \right) \\
    &< \frac{1}{2a \log(\alpha)} \left( \log \left(\frac{a}{b \eta^2}-1 \right) - \log \left( \frac{a}{b} \alpha^{-2\varepsilon} \right) \right) \\
    &< \frac{3 \varepsilon}{\left\lVert D_{j^\star}^{(k)} \right\rVert} , \label{phase_2a_duration_bound}
\end{align}

valid provided $\alpha$ is taken sufficiently small owing to the fact we take $\eta< \sqrt{ n \left\lVert D_{j^\star}^{(k)} \right\rVert } $ fixed.
\end{subproof}

\begin{subproof}[Proof of \Cref{lemma_phase_2a_realignment_fitted_data,lemma_phase_2a_residual}] We derive an upper bound on the alignment of neuron $j^\star$ with data in the set $S_F^{(k)}$ in terms of $\left\lVert E^{(k)} \right\rVert$. We have for any $i \in S_F^{(k)}$ 

\begin{align}
    \frac{1}{\log(1/\alpha)} \frac{\df}{\df t} \widetilde{\overline{w}}_{j^\star}^\top x_i &= -\frac{1}{n} A_{i, j^\star} (h_{\widetilde{\theta}}(x_i)-y_i)  \\
    &\qquad + \frac{1}{n} \sum_{v \in [n]} A_{v, j^\star} (h_{\widetilde{\theta}}(x_v)-y_v) \widetilde{\overline{w}}_{j^\star}^\top x_v \widetilde{\overline{w}}_{j^\star}^\top x_i. \label{phase_2a_derivate_max_realignment}
\end{align}

For the first term, if $i \in S_F^{(k)}$, then

\begin{equation}
    -\frac{1}{n}(h_{\widetilde{\theta}}(x_i)-y_i) \leq \left\lVert E^{(k)} \right\rVert . \label{phase_2a_single_realign_term}
\end{equation}

For the sum, we obtain separate estimates over unfitted and fitted data. For $v \in S_{j^\star}^{(k)},$ recall that by construction of the set $N_F^{(k)}$, $\{ j \in N_F^{(k)} \mid \widetilde{w}_j^\top D_{j^\star}^{(k)} >0 \} = \emptyset$. Therefore, the only contribution to the network output on such data is from neurons in the set $N_U^{(k)}$. This leads to the inequalities

\begin{align}
    h_{\widetilde{\theta}}(x_v) - y_v &= \lVert \widetilde{w}_{j^\star} \rVert^2 \widetilde{\overline{w}}_{j^\star}^\top x_i + \sum_{j \in N_U^{(k)} \setminus \{ j^\star \} } \lVert \widetilde{w}_j \rVert^2 \widetilde{\overline{w}}_j^\top x_i - y_v \\
    &< \lVert \widetilde{w}_{j^\star} \rVert^2 \widetilde{\overline{w}}_{j^\star}^\top x_i + m \alpha^{2{c_U}} - y_v  \\
    &< \eta^2 \left( \frac{y_v}{n \left\lVert D_{j^\star}^{(k)} \right\rVert} + \sqrt{2} \alpha^{\varepsilon/24} \right) + m \alpha^{2{c_U}} \\
    &< y_v \left( \frac{\eta^2}{n \left\lVert D_{j^\star}^{(k)} \right\rVert} - 1 \right) + \alpha^{\varepsilon / 25}, \label{phase_2_residual_upper_bound_network_output}
\end{align}

where (1) is from \Cref{pr:uniform_bound_unfitted}, (2) is an application of \Cref{alignment_observation} to \eqref{phase_2a_alignment_defn}, and (3) is valid for $\alpha$ sufficiently small. Using the same application of \Cref{alignment_observation} to \eqref{phase_2a_alignment_defn}, we obtain

\begin{align}
    &\frac{1}{n} \sum_{v \in S_{j^\star}^{(k)}} (h_{\widetilde{\theta}}(x_v)-y_v) \widetilde{\overline{w}}_{j^\star}^\top x_v \\
    &\qquad< \frac{1}{n} \sum_{v \in S_{j^\star}^{(k)}} \left( y_v \left( \frac{\eta^2}{n \left\lVert D_{j^\star}^{(k)} \right\rVert} - 1 \right) + \alpha^{\varepsilon / 25} \right) \left( \frac{y_v}{n \left\lVert D_{j^\star}^{(k)} \right\rVert} - \sqrt{2}\alpha^{\varepsilon/24} \right) \\
    &\qquad< \sum_{v \in S_{j^\star}^{(k)}} \frac{y_v^2}{n^2 \left\lVert D_{j^\star}^{(k)} \right\rVert}  \left( \frac{\eta^2}{n \left\lVert D_{j^\star}^{(k)} \right\rVert} - 1 \right) + \alpha^{\varepsilon / 26} \\
    &\qquad= \left\lVert D_{j^\star}^{(k)} \right\rVert  \left( \frac{\eta^2}{n \left\lVert D_{j^\star}^{(k)} \right\rVert} - 1 \right) + \alpha^{\varepsilon / 26}, \label{phase_2_realign_bound_unfitted}
\end{align}

where the second inequality is valid for $\alpha$ sufficiently small. For $v \in S_F^{(k)},$ we have the obvious inequality $\widetilde{\overline{w}}_{j^\star}^\top x_v \leq \max_{v \in S_F^{(k)}} \widetilde{\overline{w}}_{j^\star}^\top x_v$. Using the Cauchy-Schwarz inequality gives 

\begin{equation} \label{phase_2_realign_bound_fitted}
    \frac{1}{n} \sum_{ v \in S_F^{(k)} } A_{v, j^\star} (h_{\widetilde{\theta}}(x_v)-y_v) \widetilde{\overline{w}}_{j^\star}^\top x_v \leq \sqrt{n} \left\lVert E^{(k)} \right\rVert \max_{v \in S_F^{(k)}} \widetilde{\overline{w}}_{j^\star}^\top x_v.
\end{equation}

Inserting \eqref{phase_2a_single_realign_term}, \eqref{phase_2_realign_bound_unfitted}, and \eqref{phase_2_realign_bound_fitted} into \eqref{phase_2a_derivate_max_realignment}, we obtain that for $[\tau_1, \tau_2]$ and $i \in S_F^{(k)}$,

\begin{align}
     \frac{1}{\log(1/\alpha)} \frac{\df}{\df t} \widetilde{\overline{w}}_{j^\star}^\top x_i \leq \left\lVert E^{(k)} \right\rVert &+ \left( \left\lVert D_{j^\star}^{(k)} \right\rVert \left( \frac{\eta^2}{n \left\lVert D_{j^\star}^{(k)} \right\rVert} - 1 \right) + \alpha^{\varepsilon / 26} \right) \widetilde{\overline{w}}_{j^\star}^\top x_i \\
     &+ \sqrt{n} \left\lVert E^{(k)} \right\rVert \max_{v \in S_F^{(k)}} \widetilde{\overline{w}}_{j^\star}^\top x_v \widetilde{\overline{w}}_{j^\star}^\top x_i. 
\end{align}

Since the above inequality holds for any $i \in S_F^{(k)},$ we have that, when differentiable, the continuous function $f(t):= \max_{i \in S_F^{(k)}} \left( \widetilde{\overline{w}}_{j^\star}^\top x_i \right)_+$ obeys the differential inequality 

\begin{equation}
    \frac{1}{\log(1/\alpha)} f'(t) < \left\lVert E^{(k)} \right\rVert \left( 1 + \sqrt{n} f(t)^2 \right) - \left( \left\lVert D_{j^\star}^{(k)} \right\rVert \left( 1 - \frac{\eta^2}{n \left\lVert D_{j^\star}^{(k)} \right\rVert} \right) + \alpha^{\varepsilon / 26} \right)  f(t). \label{phase_2a_realign_explanation}
\end{equation}

Rademacher's theorem applies here also, however we must be more careful since the function $w \mapsto \overline{w}$ is not Lipschitz continuous (in fact, it is not even continuous at the origin). However, it is Lipschitz continuous on any set bounded away from the origin. Note that the arguments in the proof of \Cref{lemma_phase_2a_duration} \Cref{lemma_phase_2a} show that neuron $j^\star$ is increasing in norm monotonically on $[\tau_1, \tau_2]$. Combining this observation with \Cref{lemma_phase_1_j*_biggest} \Cref{lemma_phase_1}, we have that $\alpha^{\varepsilon} < \lVert \widetilde{w}_{j^\star} \rVert < \eta$, a domain on which $w \mapsto \overline{w}$ is indeed Lipschitz continuous for $\alpha$ fixed. As the boundedness of the subdifferential on compact domains (note \eqref{phase_2a_Lipschitz_defn} in the definition of $\tau_2$) implies $\widetilde{\theta}$ is Lipschitz in $t$ and the composition of Lipschitz continuous functions is Lipschitz continuous (we note $\max$ is 1-Lipschitz), we finally get the required differentiability. We use \eqref{phase_2a_alignment_defn} and \Cref{alignment_observation} to see that $f(t)^2 < 2 \alpha^{\varepsilon/12}$ on $[\tau_1, \tau_2]$. Plugging this back in we have for $\alpha$ sufficiently small

\begin{equation}
    \frac{1}{\log(1/\alpha)} f'(t) \leq \left\lVert E^{(k)} \right\rVert \left( 1 + 2 \sqrt{n} \alpha^{\varepsilon/12} \right) - \left( \left\lVert D_{j^\star}^{(k)} \right\rVert \left( 1 - \frac{\eta^2}{n \left\lVert D_{j^\star}^{(k)} \right\rVert} \right) + \alpha^{\varepsilon / 26} \right)  f(t).
\end{equation}

We have the differential inequality $f'(t) < a \left\lVert E^{(k)} \right\rVert \log (1/\alpha) - b \log(1/\alpha) f(t) $, where $a =  \left( 1 + 2 \sqrt{n} \alpha^{\varepsilon/12} \right)$ and $b = \left\lVert D_{j^\star}^{(k)} \right\rVert \left( 1 - \frac{\eta^2}{n \left\lVert D_{j^\star}^{(k)} \right\rVert} \right) + \alpha^{\varepsilon / 26}$. Multiplying both sides by $\exp(b \log(1/\alpha)t)$, we get $\left( \exp(b \log(1/\alpha) t) f(t) \right)' < a \log(1/\alpha) \left\lVert E^{(k)} \right\rVert \exp(b \log(1/\alpha) t)$. Integration leads to

\begin{align} 
    f(t) &< \alpha^{b(t-\tau_1)}f(\tau_1) + a \log(1/\alpha) \alpha^{bt} \int_{\tau_1}^t \alpha^{-bs} \left\lVert E^{(k)} \right\rVert ds \\
    &< \sqrt{2} \alpha^{\varepsilon/4} + a \log(1/\alpha) \alpha^{bt} \int_{\tau_1}^t \alpha^{-bs} \left\lVert E^{(k)} \right\rVert ds \label{phase_2_realign_fitted_data_bound_stronger} \\
    &< \sqrt{2} \alpha^{\varepsilon/4} + a \log(1/\alpha) \int_{\tau_1}^t \left\lVert E^{(k)} \right\rVert ds \label{phase_2_realign_fitted_data_bound},
\end{align}

where the second inequality comes from an application of \Cref{alignment_observation} to the initial condition from \Cref{lemma_phase_1_alignment} \Cref{lemma_phase_1} and the fact $b > 0$. We use this bound to derive a differential equation governing the growth of the residual. We must again estimate \eqref{derivative_of_residual_basic_bound}, which we recall states

\begin{equation}
    - \frac{1}{\log(1/\alpha)} \frac{\df}{\df t} \left\lVert E^{(k)} \right\rVert^2 \leq - \frac{1}{n} \sum_{i \in S_F^{(k)}} \sum_{j \in N_U^{(k)}} A_{i,j} \lVert \widetilde{w}_j \rVert^2 \fD_j^\top \widetilde{\overline{w}}_j \widetilde{\overline{w}}_j^\top x_i x_i^\top E^{(k)}.
\end{equation}

We consider separately the cases $j \in N_U^{(k)} \setminus j^\star$ and $j=j^\star$. Beginning with the former, \Cref{pr:uniform_bound_unfitted} implies $\lVert \widetilde{w}_j \rVert^2 < \alpha^{2{c_U}}$. We simply reuse \eqref{fd_bound} combined with the Cauchy-Schwarz inequality to obtain 

\begin{equation} \label{phase_2_ode_unfitted_bound}
    - \frac{1}{n} \sum_{i \in S_F^{(k)}} \sum_{j \in N_U^{(k)} \setminus \{ j^\star \} } A_{i,j} \lVert \widetilde{w}_j \rVert^2  \fD_j^\top \widetilde{\overline{w}}_j \widetilde{\overline{w}}_j^\top x_i x_i^\top E^{(k)} \leq \frac{m}{n} \sqrt{2 \mathcal{L}(0)} \alpha^{2{c_U}} \left\lVert E^{(k)} \right\rVert . 
\end{equation} 

Note that we have simply estimated $| \widetilde{\overline{w}}_j^\top x_i | \leq 1$, since we have imposed no assumption on the alignment of neurons $j \in N_U^{(k)}$ during this phase. For the contribution of the term corresponding to $j=j^\star$, we have

\begin{equation} \label{phase_2_ode_j*_bound}
    -\frac{1}{n} \sum_{i \in S_{F}^{(k)}} A_{i, j^\star} \lVert \widetilde{w}_{j^\star} \rVert^2  \fD_{j^\star}^\top \widetilde{\overline{w}}_{j^\star} \widetilde{\overline{w}}_{j^\star}^\top x_i x_i^\top E^{(k)} < \frac{1}{n} \sqrt{2 \mathcal{L}(0)} \eta^2 f(t) \left\lVert E^{(k)} \right\rVert,
\end{equation}

where we have applied the bound \eqref{phase_2a_norm_defn} on the norm of neuron $j^\star$ in the definition of $\tau_2$ and \eqref{fd_bound} combined with the Cauchy-Schwarz inequality. Substituting in \eqref{phase_2_realign_fitted_data_bound}, we obtain 

\begin{align}
    -\frac{1}{n} \sum_{i \in S_{F}^{(k)}} A_{i,j^\star} & \lVert \widetilde{w}_{j^\star} \rVert^2  \fD_{j^\star}^\top \widetilde{\overline{w}}_{j^\star} \widetilde{\overline{w}}_{j^\star}^\top x_i x_i^\top E^{(k)} \\
    &\leq \frac{1}{n} \sqrt{2 \mathcal{L}(0)} \eta^2 \left( \sqrt{2}\alpha^{\varepsilon/4} + a \log(1/\alpha) \int_{\tau_1}^t \left\lVert E^{(k)} \right\rVert ds \right) \left\lVert E^{(k)} \right\rVert \label{phase_2a_ode_j_star_bound}
\end{align}

Inserting \eqref{phase_2_ode_unfitted_bound} and \eqref{phase_2a_ode_j_star_bound} into \eqref{derivative_of_residual_basic_bound}, we have

\begin{align}
    \frac{\df}{\df t} \left\lVert E^{(k)} \right\rVert &\leq \frac{1}{n} \sqrt{2 \mathcal{L}(0)} \eta^2 \log(1/\alpha) \left( \sqrt{2}\alpha^{\varepsilon/4} + a \log(1/\alpha) \int_{\tau_1}^t \left\lVert E^{(k)} \right\rVert ds \right) \\
    & \hspace{6cm} + \frac{m}{n} \sqrt{2 \mathcal{L}(0)}  \log(1/\alpha) \alpha^{2{c_U}} \\
    &<  \frac{2}{n} \sqrt{2 \mathcal{L}(0)} \eta^2 \log(1/\alpha) \alpha^{\varepsilon/4} + \frac{2}{n} \sqrt{2 \mathcal{L}(0)} \eta^2 \log(1/\alpha)^2 \int_{\tau_1}^t \left\lVert E^{(k)} \right\rVert ds,
\end{align}

where the second inequality is valid for $\alpha$ sufficiently small by increasing the coefficient of the $\alpha^{\varepsilon/4}$ term (recall here that ${c_U}>\varepsilon$) and by using the estimate $a = 1+ 2 \sqrt{n} \alpha^{\varepsilon/12} < 2$ for $\alpha$ sufficiently small. Differentiability almost everywhere is again due to Rademacher's theorem. This is a differential inequality of the form $F''<pF+q$, where $F = \int_{\tau_1}^t \left\lVert E^{(k)} \right\rVert ds$. Applying \cref{diff_ineq_lemma}, 
\begin{align}
    F'(t) &< \left( \frac{q}{\sqrt{p}} + F'(\tau_1) \right) \exp \left( \sqrt{p} (t-\tau_1) \right) \\
    &= \left( \sqrt{\frac{2}{n}} \sqrt[4]{2 \mathcal{L}(0)}  \eta \alpha^{\varepsilon/4} + \alpha^{\varepsilon} \right) \exp_\alpha \left( - \sqrt{\frac{2}{n}} \sqrt[4]{2 \mathcal{L}(0)} \eta (t-\tau_1) \right) \\
    &\overset{(1)}{<} 2\sqrt{\frac{2}{n}}  \sqrt[4]{2 \mathcal{L}(0)}  \eta \alpha^{\varepsilon/4} \exp_\alpha \left( - \sqrt{\frac{2}{n}} \sqrt[4]{2 \mathcal{L}(0)} \eta (t-\tau_1) \right) \\
    &\overset{(2)}{<} \frac{4}{\sqrt{n}} \sqrt[4]{2 \mathcal{L}(0)}  \eta \alpha^{\varepsilon/4} \exp_\alpha \left( - \sqrt{\frac{2}{n}} \sqrt[4]{2 \mathcal{L}(0)} \eta \frac{3 \varepsilon}{\left\lVert D_{j^\star}^{(k)} \right\rVert} \right) \\
    &\overset{(3)}{<} \alpha^{\varepsilon/5} \label{phase_2a_residual_final_bound}.
\end{align}

In these workings, (1) is valid for $\alpha$ sufficiently small, (2) follows from inserting the upper bound on $\tau_2-\tau_1$ derived in \Cref{lemma_phase_2a_duration} \Cref{lemma_phase_2a}, and (3) follows from the fact $\eta <  \frac{\sqrt{n}}{60 \sqrt{2}} \frac{\left\lVert D_{j^\star}^{(k)} \right\rVert}{\sqrt[4]{2 \mathcal{L}(0)}} $. 

We are now able to plug \eqref{phase_2a_residual_final_bound} back into \eqref{phase_2_realign_fitted_data_bound_stronger} to estimate

\begin{align}
    \max_{i \in S_F^{(k)}} A_{i,j^\star} \widetilde{\overline{w}}_{j^\star}^\top x_i &< \sqrt{2}\alpha^{\varepsilon/4} + a \log(1/\alpha) \alpha^{\varepsilon/5} \int_{\tau_1}^t \alpha^{b(t-s)} ds \\
    &= \sqrt{2}\alpha^{\varepsilon/4} + \frac{a}{b} \alpha^{\varepsilon/5} \left( 1 - \alpha^{b (t-\tau_1)} \right) \\
    &< c_1 \alpha^{\varepsilon/5}, \label{phase_2_realign_fitted_data_final_bound}
\end{align}

for some $c_1 > 0$. \Cref{lemma_phase_2a_realignment_fitted_data} \Cref{lemma_phase_2a} then follows for $\alpha$ sufficiently small.
\end{subproof}

\begin{subproof}[Proof of \Cref{lemma_phase_2a_alignment}] Next, we verify the stability of alignment of neuron $j^\star$. We have 


\begin{align} \label{phase_2a_j*_stability_equation}
    \frac{1}{\log(1/\alpha)}\frac{\df}{\df t}  \widetilde{\overline{w}}^\top_{j^\star} D_{j^\star}^{(k)} &= \left\lVert D_j^{(k)} \right\rVert^2 - \left( \widetilde{\overline{w}}^\top_j D_j^{(k)} \right)^2 \\
    &\qquad - \frac{1}{n} \sum_{i \in S_j^{(k)}} h_{\widetilde{\theta}}(x_i) \left( x_i^\top D_j^{(k)} - x^\top_i \widetilde{\overline{w}}_j \widetilde{\overline{w}}^\top_{j} D_j^{(k)} \right) \\ 
    &\qquad - \frac{1}{n} \sum_{i \in S_F^{(k)}} ( h_{\widetilde{\theta}}(x_i) - y_i) \left( x_i^\top D_j^{(k)} - x^\top_i \widetilde{\overline{w}}_j \widetilde{\overline{w}}^\top_{j} D_j^{(k)} \right)\\
    &> \left\lVert D_{j^\star}^{(k)} \right\rVert^2 - \left( \widetilde{\overline{w}}^\top_{j^\star} D_{j^\star}^{(k)} \right)^2 \\
    &\qquad - \frac{1}{n} \sum_{i \in S_{j^\star}^{(k)}} h_{\widetilde{\theta}}(x_i) \left( x_i^\top D_{j^\star}^{(k)} - x^\top_i \widetilde{\overline{w}}_{j^\star} \widetilde{\overline{w}}^\top_{j^\star} D_{j^\star}^{(k)} \right)\\
    &\qquad - 2 \left\lVert D_{j^\star}^{(k)} \right\rVert \alpha^{5 \varepsilon  / 24},
\end{align}


This bound comes from \eqref{phase_2a_residual_defn} combined with an application of \cref{alignment_corollary} to \eqref{phase_2a_alignment_defn}. 
We recall that for $i \in S_{j^\star}^{(k)}$, $h_{\widetilde{\theta}}(x_i) \leq \lVert \widetilde{w}_{j^\star} \rVert^2 \widetilde{\overline{w}}_{j^\star}^\top x_i + m\alpha^{2{c_U}}$. The growth to constant order of the norm of neuron $j^\star$ means that the previously used bootstrapping argument, where \cref{alignment_corollary} was combined with a uniform growth bound, is no longer possible. Instead, from \eqref{phase_2_realign_fitted_data_final_bound}, we have that if $i \notin S_{j}^{(k)}$, then on the interval $[\tau_1, \tau_2]$, $\widetilde{\overline{w}}_{j^\star}^\top x_i \leq c_1 \alpha^{\varepsilon/5}$.  We apply this to see


\begin{align}
    &\frac{1}{n} \sum_{i \in S_{j^\star}^{(k)}} h_{\widetilde{\theta}}(x_i) \left( x_i^\top D_{j^\star}^{(k)} - x^\top_i \widetilde{\overline{w}}_{j^\star} \widetilde{\overline{w}}^\top_{j^\star} D_{j^\star}^{(k)} \right) \\
    &= \frac{1}{n} \sum_{i \in S_{j^\star}^{(k)}} \left( \lVert \widetilde{w}_{j^\star} \rVert^2 \widetilde{\overline{w}}_{j^\star}^\top x_i + \sum_{j \neq j^\star} A_{i,j} \lVert \widetilde{w}_j \rVert^2 \widetilde{\overline{w}}_j^\top x_i \right) \left( x_i^\top D_{j^\star}^{(k)} - x^\top_i \widetilde{\overline{w}}_{j^\star} \widetilde{\overline{w}}^\top_{j^\star} D_{j^\star}^{(k)} \right) \\
    &\leq \frac{1}{n} \lVert \widetilde{w}_{j^\star} \rVert^2 \widetilde{\overline{w}}_{j^\star}^\top D_{j^\star}^{(k)} \left( 1 - \sum_{i \in S_{j^\star}^{(k)}} \left( \widetilde{\overline{w}}_{j^\star}^\top x_i \right)^2 \right) + 2 m \left\lVert D_{j^\star}^{(k)} \right\rVert \alpha^{2{c_U}} \\
    &= \frac{1}{n} \lVert \widetilde{w}_{j^\star} \rVert^2 \widetilde{\overline{w}}_{j^\star}^\top D_{j^\star}^{(k)} \left( \sum_{i \notin S_{j^\star}^{(k)}} \left( \widetilde{\overline{w}}_{j^\star}^\top x_i \right)^2 \right) + 2 m \left\lVert D_{j^\star}^{(k)} \right\rVert \alpha^{2{c_U}} \\
    &\leq c_1^2 \eta^2 \left\lVert D_{j^\star}^{(k)} \right\rVert \alpha^{2 \varepsilon/5} + 2 m \left\lVert D_{j^\star}^{(k)} \right\rVert \alpha^{2{c_U}}. \label{phase_2a_j*_stability_network_bound}
\end{align}


Putting \eqref{phase_2a_j*_stability_network_bound} into \eqref{phase_2a_j*_stability_equation}, we have for $t \in [\tau_1, \tau_2]$ and for $\alpha$ sufficiently small that

\begin{align}
    \frac{1}{\log(1/\alpha)}\frac{\df}{\df t}  \widetilde{\overline{w}}^\top_{j^\star} D_{j^\star}^{(k)} &> \left\lVert D_{j^\star}^{(k)} \right\rVert^2 - \left( \widetilde{\overline{w}}^\top_{j^\star} D_{j^\star}^{(k)} \right)^2 \\
    &\qquad - c_1^2 \eta^2 \left\lVert D_{j^\star}^{(k)} \right\rVert \alpha^{2\varepsilon/5} - 2 m \left\lVert D_{j^\star}^{(k)} \right\rVert \alpha^{2{c_U}} -  2 \left\lVert D_{j^\star}^{(k)} \right\rVert \alpha^{5 \varepsilon  / 24} \\
    &> \left\lVert D_{j^\star}^{(k)} \right\rVert^2 - \left( \widetilde{\overline{w}}^\top_{j^\star} D_{j^\star}^{(k)} \right)^2 - \alpha^{9 \varepsilon / 48}.
\end{align}

We can now use arguments similar to those employed in the proof of \Cref{lemma_phase_1_alignment} \Cref{lemma_phase_1}. From the initial conditions given by \cref{lemma_phase_1_alignment} \Cref{lemma_phase_1} and the inequality $\sqrt{1-x^2}>1-x$, we obtain that on $[\tau_1, \tau_2]$ the alignment is lower bounded a function converging monotonically to $\min \left( \alpha^{\varepsilon/2}, \alpha^{9 \varepsilon/96} \right)$. It then follows that for $\alpha$ sufficiently small,

\[
\widetilde{\overline{w}}^\top_{j^\star} D_{j^\star}^{(k)} > \left\lVert D_{j^\star}^{(k)} \right\rVert ( 1 - \alpha^{\varepsilon  / 12}). 
\]





\end{subproof}

\begin{subproof}{Proof of \Cref{lemma_phase_2a_fitted_norms}} We verify that the norms of neurons in the set $N_F^{(k)}$ did not change significantly on the interval $[\tau_1, \tau_2]$, which will lead us to conclude in combination with \Cref{pr:uniform_bound_unfitted} that \eqref{phase_2a_Lipschitz_defn} did not fail on this interval. \Cref{lemma_phase_1_fitted_norms} \Cref{lemma_phase_1} provides upper and lower bounds on the norms of neurons in the set $N_F^{(k)}$ at time $\tau_1$. In light of these initial conditions and the bound on the loss on data in the set $S_F^{(k)}$ given in \eqref{phase_2a_residual_defn}, identical workings to those yielding \eqref{phase_1_initial_lower_bound_fitted_norms} give

\begin{equation}
    \sqrt{n \left\lVert D_j^{(k)} \right\rVert } - \alpha^{\varepsilon/8} < \lVert \widetilde{w}_j \rVert < \sqrt{n \left\lVert D_j^{(k)} \right\rVert } + \alpha^{\varepsilon/8}
\end{equation}

for all $t \in [\tau_1, \tau_2]$ for $\alpha$ sufficiently small.

\end{subproof}

\begin{subproof}{Proof of \Cref{lemma_phase_2a_jstar_norm}} We have concluded that no condition except for that on the norm of neuron $j^\star$ could have broken at time $\tau_2$, therefore $\lVert \widetilde{w}_{j^\star}(\tau_2) \rVert = \eta$, as claimed.
    
\end{subproof}

This completes the proof of \Cref{lemma_phase_2a}.
\end{proof}

\subsection{Phase 2b: Fitting of data}

Define for $\varepsilon' > 0$ the time 
    
    \begin{align}
        \tau_3 := \inf \Big\{ t \geq & \tau_2 \; \big\vert \; \widetilde{\overline{w}}_{j^\star}^\top \overline{D}_{j^\star}^{(k)} \leq 1 - \alpha^{\varepsilon'/2} \label{phase_2b_alignment_defn} \\
        & \;\; \mathrm{or} \; \exists \; i \in S_F^{(k)} \; \mathrm{s.t.} \; | h_{\widetilde{\theta}}(x_i) - y_i | \, \geq \alpha^{\varepsilon'} \label{phase_2b_residual_defn} \\  
        & \;\; \mathrm{or} \; \exists \, j \in N_F^{(k)} \; \mathrm{s.t.} \;\lVert \widetilde{w}_{j} \rVert > 2 \max_{j \in [m]} \left\lVert D_j^{(k)} \right\rVert \label{phase_2b_lipschitz_defn} \\
        & \;\; \mathrm{or} \; \lVert \widetilde{w}_{j^\star} \rVert^2 \geq n \left\lVert D_{j^\star}^{(k)} \right\rVert - \alpha^{\varepsilon' / 4} \Big\} \bigwedge  \left( \tau_1 + 2 \sqrt{\frac{n}{2\mathcal{L}(0)}} c_U  \right) \label{phase_2b_norm_defn}.
    \end{align}

The lemma below provides an account of the phase where neuron $j^\star$ fits the data on which is it active. We prove that the alignment of neuron $ j^\star$ remains stable, the loss on fitted data does not increase significantly, and by the end of the phase neuron $j^\star$ has grown very close to its final norm as predicted by \Cref{alg:s.t.s}. The proof for this phase is almost identical to that for Phase 2a.
    
\begin{lem} \label{lemma_phase_2b}
    
    Suppose that \Cref{inductive_hypothesis} holds. Then for all $\kappa \in (t_k, t_{k+1})$, there exists $\varepsilon^*>0$ such that for all $\varepsilon \in (0, \varepsilon^*)$, there exists $\varepsilon'^*>0$ such that for all $\varepsilon' \in (0, \varepsilon'^*)$ there exists $\alpha^*$ such that for $\alpha \in (0,\alpha^*)$ the following are true:

    \begin{enumerate}
        \item $ \tau_3 - \tau_2 < \frac{\varepsilon'}{\left\lVert D_{j^\star}^{(k)} \right\rVert} .$ \label{lemma_phase_2b_duration}
        \item $ \forall j \in N_F^{(k)}, \forall t \in [\tau_2, \tau_3], \; \sqrt{n\left\lVert D_j^{(k)} \right\rVert} - \alpha^{\varepsilon'/2} < \lVert \widetilde{w}_j \rVert < \sqrt{n\left\lVert D_j^{(k)} \right\rVert} + \alpha^{\varepsilon'/2} .$ \label{lemma_phase_2b_fitted_norms}
        \item $ \forall i \in S_F^{(k)}, \forall t \in [\tau_2, \tau_3], \; | h_{\widetilde{\theta}}(x_i) - y_i | \, < \alpha^{\varepsilon'} .$ \label{lemma_phase_2b_residual} 
        \item $\forall i \in S_F^{(k)}, \forall t \in [\tau_2, \tau_3], \; \max_{i \in S_F^{(k)}} \left( \widetilde{\overline{w}}_{j^\star}^\top x_i \right)_+ < \alpha^{5 \varepsilon' / 4}$. \label{lemma_phase_2b_realign_fitted_data}
        \item $ \forall t \in [\tau_2, \tau_3], \; \widetilde{\overline{w}}_{j^\star}^\top \overline{D}_{j^\star}^{(k)} > 1 - \alpha^{\varepsilon'} .$ \label{lemma_phase_2b_alignment}
        \item $\lVert \widetilde{w}_{j^\star}(\tau_3) \rVert^2 = n \left\lVert D_{j^\star}^{(k)} \right\rVert - \alpha^{\varepsilon'/4}$. \label{lemma_phase_2b_j_star_norm}
    \end{enumerate}
\end{lem}


\begin{proof}

We work out a useful preliminary bound to begin this section. We upper bound the network output on $i \in S_{j^\star}^{(k)}$:  

\begin{align}
    h_{\widetilde{\theta}}(x_i) &\overset{(1)}{<} \widetilde{\overline{w}}_{j^\star}^\top x_i \lVert \widetilde{w}_{j^\star} \rVert^2 + m \alpha^{2{c_U}} \\
    &\overset{(2)}{<} \left( \frac{y_i}{n \left\lVert D_{j^\star}^{(k)} \right\rVert} + \sqrt{2}\alpha^{\varepsilon'/2} \right) \left( n \left\lVert D_{j^\star}^{(k)} \right\rVert - \alpha^{\varepsilon'/4} \right) +m \alpha^{2{c_U}} \\
    &\overset{(3)}{<} y_i - \alpha^{\varepsilon'/3}. \label{phase_2b_network_upper_bound}
\end{align}

Here, (1) is from \cref{pr:uniform_bound_unfitted}, (2) is from combining \eqref{phase_2b_norm_defn} with an application of \Cref{alignment_observation} to \eqref{phase_2b_alignment_defn}, and (3) is valid for $\alpha$ sufficiently small. 

\begin{subproof}[Proof of \Cref{lemma_phase_2b_duration}] We estimate the growth of neuron $j^\star$. Similar workings to those which allowed us to obtain \eqref{phase_2a_norm_growth_lower_bound} in the proof of \Cref{lemma_phase_2a_duration} \Cref{lemma_phase_2a} allow us to conclude that for $\alpha$ sufficiently small, on $[\tau_2, \tau_3]$ we have 
\begin{align}
    \frac{1}{2 \log(1/\alpha)} \frac{\df}{\df t} \lVert \widetilde{w}_{j^\star} \rVert^2 &> \lVert \widetilde{w}_{j^\star} \rVert^2  \left( \left\lVert D_{j^\star}^{(k)} \right\rVert (1-\alpha^{\varepsilon'}) - \alpha^{\varepsilon'} - \frac{\lVert \widetilde{w}_{j^\star} \rVert^2}{n} - m \alpha^{2{c_U}} \right) \\
    &> \lVert \widetilde{w}_{j^\star} \rVert^2  \left( \left\lVert D_{j^\star}^{(k)} \right\rVert - \alpha^{\varepsilon' / 2} \right) - \frac{\lVert \widetilde{w}_{j^\star} \rVert^4}{n}.
\end{align}

As before, solutions of differential inequality $f'(t) = 2 \log(1/\alpha) a f(t) - 2 \log(1/\alpha) b f(t)^2 $ for $t \in [\tau_2, \tau_3]$ have the lower bound

\begin{equation}
    f(t)>\frac{a}{b} \frac{1}{\alpha^{2a(t-C)}+1}.
\end{equation}

From \Cref{lemma_phase_2a}, we have $\lVert \widetilde{w}_{j^\star}(\tau_2) \rVert^2 = \eta^2$, and determining the integration constant yields

\begin{align}
    f(t) &>\frac{a}{b} \frac{1}{1 +  \exp_\alpha \left( 2a(t-\tau_2) + \frac{\log \left( \frac{a}{b} \eta^{-2} - 1 \right)}{\log(\alpha)} \right) } .
\end{align}

Therefore, 

\begin{align}
    f \left( \tau_2 + \frac{\varepsilon'}{\left\lVert D_{j^\star}^{(k)} \right\rVert} \right) &\overset{(1)}{>} \frac{a}{b} \frac{1}{1 +  \exp_\alpha \left( \frac{2 \varepsilon' }{\left\lVert D_{j^\star}^{(k)} \right\rVert} \left( \left\lVert D_{j^\star}^{(k)} \right\rVert - \alpha^{\varepsilon'/2} \right) + \frac{\log \left( \frac{a}{b} \eta^{-2} - 1 \right)}{\log(\alpha)} \right) }  \\
    &\overset{(1)}{>}\frac{a}{b} \frac{1}{1 +  \alpha^{\varepsilon'} } \\
    &> \frac{a}{b} (1-2 \alpha^{\varepsilon'}) \\
    &\overset{(2)}{>} n \left\lVert D_{j^\star}^{(k)} \right\rVert - \alpha^{\varepsilon' / 4}, \label{phase_2b_j*_norm_final_bound}
\end{align}

which contradicts the definition of $\tau_3$. Inequality (1) comes from noting $\frac{a}{b} \eta^{-2} - 1$ is bounded away from 0 as $\alpha$ decreases owing to the fact $\eta$ is fixed and inequality (2) comes from estimating $\frac{a}{b}$. Note that each of the inequalities above is only valid for $\alpha$ sufficiently small, depending on $\varepsilon$ and the data. Necessarily, we have $\tau_3 < \tau_2 + \frac{\varepsilon'}{\left\lVert D_{j^\star}^{(k)} \right\rVert}$.
\end{subproof}





\begin{subproof}[Proof of \Cref{lemma_phase_2b_residual,lemma_phase_2b_realign_fitted_data}] In this section it will again be important to derive a bound on the realignment of neuron $j^\star$ with fitted data in terms of the residual on fitted data. We recall that for $i \in S_F^{(k)}$,

\begin{align}
    \frac{1}{\log(1/\alpha)} \frac{\df}{\df t} \widetilde{\overline{w}}_{j^\star}^\top x_i &= -\frac{1}{n}(h_{\widetilde{\theta}}(x_i)-y_i) \iind{ \widetilde{w}_{j^\star}^\top x_i>0} \\
    &\qquad + \frac{1}{n} \sum_{v \in [n]} A_{v, j} (h_{\widetilde{\theta}}(x_v)-y_v) \widetilde{\overline{w}}_{j^\star}^\top x_v \widetilde{\overline{w}}_{j^\star}^\top x_i \\
    &\leq \left\lVert E^{(k)} \right\rVert \left( 1 + \sqrt{n} \max_{v \in S_F^{(k)}} \widetilde{\overline{w}}_{j^\star}^\top x_v \widetilde{\overline{w}}_{j^\star}^\top x_i) \right),
\end{align}

where we have discarded the sum over $v \in S_{j^\star}^{(k)}$ since \eqref{phase_2b_network_upper_bound} gives that this adds a negative contribution. Since this holds for all $i \in S_F^{(k)}$ and for all $t \in [\tau_2, \tau_3],$ we have that  $f(t):= \max_{i \in S_F^{(k)}} \left( \widetilde{\overline{w}}_{j^\star}^\top x_i \right)_+$ obeys the inequality 

\begin{equation}
    \frac{1}{\log(1/\alpha)} f'(t) \leq \left\lVert E^{(k)} \right\rVert \left( 1 + 2\sqrt{n} \alpha^{\varepsilon'} \right), \label{phase_2b_realign_explanation}
\end{equation}

where again we have estimated $f(t)^2$ by plugging in the assumed bound on alignment. Differentiability almost everywhere follows from Rademacher's theorem. Integration gives 

\begin{align} \label{phase_2b_realign_fitted_data_bound}
     f(t) &< \sqrt{2} \alpha^{\varepsilon/24} + \left( 1 + 2 \sqrt{n} \alpha^{\varepsilon'} \right) \log(1/\alpha) \int_{\tau_2}^{t} \left\lVert E^{(k)} \right\rVert ds . 
\end{align}

We proceed to estimate \eqref{derivative_of_residual_basic_bound}, which we recall for convenience is 

\begin{equation}
    \frac{1}{2 \log(1/\alpha) } \frac{\df}{\df t} \left\lVert E^{(k)} \right\rVert^2 \leq - \frac{1}{n} \sum_{i \in S_F^{(k)}} \sum_{j \in N_U^{(k)}} A_{i,j} \lVert \widetilde{w}_j \rVert^2 \fD_j^\top \widetilde{\overline{w}}_j \widetilde{\overline{w}}_j^\top x_i x_i^\top E^{(k)}.
\end{equation}
We can again discard the sum over $j \in N_F^{(k)}$, leaving only the cases $j=j^\star$ and $j \in N_U^{(k)} \setminus \{ j^\star \}$. For the latter, we apply \Cref{pr:uniform_bound_unfitted} in a similar manner to the previous phase to get

\begin{equation} \label{phase_2b_ode_unfitted_bound}
    - \frac{1}{n} \sum_{i \in S_F^{(k)}} \sum_{ j \in N_U^{(k)} \setminus \{ j^\star \} } A_{i,j} \lVert \widetilde{w}_j \rVert^2  \fD_j^\top \widetilde{\overline{w}}_j \widetilde{\overline{w}}_j^\top x_i x_i^\top E^{(k)} \leq \frac{m}{n} \sqrt{2 \mathcal{L}(0)}  \alpha^{2{c_U}} \left\lVert E^{(k)} \right\rVert .
\end{equation} 

For the case $j=j^\star$, we obtain the estimate


\begin{align}
    - \frac{1}{n} \sum_{i \in S_F^{(k)}} A_{i,j^\star} & \lVert \widetilde{w}_{j^\star} \rVert^2 \fD_{j^\star}^\top \widetilde{\overline{w}}_{j^\star} \widetilde{\overline{w}}_{j^\star}^\top x_i x_i^\top E^{(k)} \\
    &\leq \left( n \left\lVert D_{j^\star}^{(k)} \right\rVert - \alpha^{\varepsilon'/4} \right) \sqrt{\frac{2 \mathcal{L}(0)}{n}} \frac{1}{n} \sum_{i \in S_F^{(k)}} A_{i,j^\star} \widetilde{\overline{w}}_{j^\star}^\top x_i x_i^\top E^{(k)}  \\
    &\leq \left\lVert D_{j^\star}^{(k)} \right\rVert \sqrt{2\mathcal{L}(0)} f(t) \left\lVert E^{(k)} \right\rVert . \label{phase_2b_ode_j*_bound}
\end{align}

where in the final line we apply the Cauchy-Schwarz inequality. Inserting \eqref{phase_2b_ode_j*_bound} and \eqref{phase_2b_ode_unfitted_bound} into \eqref{derivative_of_residual_basic_bound}, we get  


\begin{align}
    \frac{\df}{\df t} \left\lVert E^{(k)} \right\rVert &\leq \left\lVert D_{j^\star}^{(k)} \right\rVert \sqrt{2\mathcal{L}(0)} \log(1/\alpha) \left( \sqrt{2} \alpha^{\varepsilon/24} + \left( 1 + 2 \sqrt{n} \alpha^{\varepsilon'} \right) \log(1/\alpha) \int_{\tau_2}^{t} \left\lVert E^{(k)} \right\rVert ds \right) \\
    & \hspace{6cm} + \frac{m}{n} \sqrt{2\mathcal{L}(0)} \log(1/\alpha) \alpha^{2{c_U}} \\
    &< 2 \left\lVert D_{j^\star}^{(k)} \right\rVert \sqrt{2\mathcal{L}(0)} \log(1/\alpha) \alpha^{\varepsilon/24} + 2 \left\lVert D_{j^\star}^{(k)} \right\rVert \sqrt{2\mathcal{L}(0)} \log(1/\alpha)^2 \int_{\tau_2}^{t} \left\lVert E^{(k)} \right\rVert ds ,
\end{align}


for $\alpha$ sufficiently small. Again, we have a differential inequality of the form $F'' < pF + q$, for which we employ \cref{diff_ineq_lemma}:

\begin{align}
    F'(t) &<  \left( \frac{q}{\sqrt{p}} + F'(\tau_2) \right) \exp \left( \sqrt{p} (t-\tau_2) \right) \\
    &< \left( \left( 2 \left\lVert D_{j^\star}^{(k)} \right\rVert \sqrt{2\mathcal{L}(0)} \right)^{1/2} \alpha^{\varepsilon/24} + \alpha^{\varepsilon/6} \right) \exp_\alpha \left( - \left( 2 \left\lVert D_{j^\star}^{(k)} \right\rVert \sqrt{2\mathcal{L}(0)} \right)^{1/2} (t-\tau_2) \right) \\
    &< \alpha^{\varepsilon/25} \exp_\alpha \left( - \left( 2 \left\lVert D_{j^\star}^{(k)} \right\rVert \sqrt{2\mathcal{L}(0)} \right)^{1/2} \frac{\varepsilon'}{\left\lVert D_{j^\star}^{(k)} \right\rVert} \right) \\
    &\overset{(2)}{<} \alpha^{3 \varepsilon' / 2}, \label{phase_2b_residual_final_bound}
\end{align}

provided $\varepsilon'$ is selected sufficiently small depending on $\varepsilon$. We may now plug \eqref{phase_2b_residual_final_bound} back into \eqref{phase_2b_realign_fitted_data_bound} to get that for $t \in [\tau_2, \tau_3]$, 

\begin{align}
    \max_{i \in S_F^{(k)}} A_{i,j^\star} \widetilde{\overline{w}}_{j^\star}^\top x_i &< \sqrt{2}\alpha^{\varepsilon/24} + \left( 1 + 2 \sqrt{n} \alpha^{\varepsilon'} \right) \log(1/\alpha) \int_{\tau_2}^{t} \left\lVert E^{(k)} \right\rVert ds \\
    &< \sqrt{2}\alpha^{\varepsilon/24} + \left( 1 + 2 \sqrt{n} \alpha^{\varepsilon'} \right) \log(1/\alpha) \alpha^{3 \varepsilon' / 2} \frac{\varepsilon'}{\left\lVert D_{j^\star}^{(k)} \right\rVert} \\
    &< \alpha^{5 \varepsilon' / 4}, \label{phase_2b_realign_fitted_data_final_bound}
\end{align}

where the final inequality is valid for $\varepsilon'$ sufficiently small depending on $\varepsilon$ and for $\alpha$ sufficiently small.
\end{subproof}


\begin{subproof}[Proof of \Cref{lemma_phase_2b_alignment}] To estimate the alignment of neuron $j^\star$, we have from \eqref{phase_2b_realign_fitted_data_final_bound} that if $i \in S_F^{(k)}$ then $\widetilde{\overline{w}}_{j^\star}^\top x_i <  \alpha^{5 \varepsilon' / 4}$. Combined with \eqref{phase_2b_residual_final_bound}, we have


\begin{align}
    \frac{1}{\log(1/\alpha)}\frac{\df}{\df t}  \widetilde{\overline{w}}^\top_{j^\star} D_{j^\star}^{(k)} &= \left\lVert D_{j^\star}^{(k)} \right\rVert^2 - \left( \widetilde{\overline{w}}^\top_{j^\star} D_{j^\star}^{(k)} \right)^2 \\
    &\qquad - \frac{1}{n} \sum_{i \in S_{j^\star}^{(k)}} h_{\widetilde{\theta}}(x_i) \left( x_i^\top D_{j^\star}^{(k)} - x^\top_i \widetilde{\overline{w}}_{j^\star} \widetilde{\overline{w}}^\top_{j^\star} D_{j^\star}^{(k)} \right) \\ 
    &\qquad - \frac{1}{n} \sum_{i \in S_F^{(k)}} ( h_{\widetilde{\theta}}(x_i) - y_i) \left( x_i^\top D_{j^\star}^{(k)} - x^\top_i \widetilde{\overline{w}}_{j^\star} \widetilde{\overline{w}}^\top_{j^\star} D_{j^\star}^{(k)} \right)  \\
    &> \left\lVert D_{j^\star}^{(k)} \right\rVert^2 - \left( \widetilde{\overline{w}}^\top_{j^\star} D_{j^\star}^{(k)} \right)^2 \\
    &\qquad - \frac{1}{n} \sum_{i \in S_{j^\star}^{(k)}} h_{\widetilde{\theta}}(x_i) \left( x_i^\top D_{j^\star}^{(k)} - x^\top_i \widetilde{\overline{w}}_{j^\star} \widetilde{\overline{w}}^\top_{j^\star} D_{j^\star}^{(k)} \right) \\
    &\qquad - 2 \left\lVert D_{j^\star}^{(k)} \right\rVert \alpha^{11 \varepsilon' / 4}. \label{phase_2b_j*_stability_equation}
\end{align}

As in the previous section, we apply the bound on the norms of all neurons $j \in N_U^{(k)} \setminus \{j^\star\}$ to get


\begin{align}
    &\frac{1}{n} \sum_{i \in S_{j^\star}^{(k)}} h_{\widetilde{\theta}}(x_i) \left( x_i^\top D_{j^\star}^{(k)} - x^\top_i \widetilde{\overline{w}}_{j^\star} \widetilde{\overline{w}}^\top_{j^\star} D_{j^\star}^{(k)} \right) \\
    &= \frac{1}{n} \sum_{i \in S_{j^\star}^{(k)}} \left( \lVert \widetilde{w}_{j^\star} \rVert^2 \widetilde{\overline{w}}_{j^\star}^\top x_i + \sum_{j \neq j^\star} A_{i, j} \lVert \widetilde{w}_j \rVert^2 \widetilde{\overline{w}}_j^\top x_i \right) \left( x_i^\top D_{j^\star}^{(k)} - x^\top_i \widetilde{\overline{w}}_{j^\star} \widetilde{\overline{w}}^\top_{j^\star} D_{j^\star}^{(k)} \right) \\
    &< \frac{1}{n} \lVert \widetilde{w}_{j^\star} \rVert^2 \widetilde{\overline{w}}_{j^\star}^\top D_{j^\star}^{(k)} \left( 1 - \sum_{i \in S_{j^\star}^{(k)}} \left( \widetilde{\overline{w}}_{j^\star}^\top x_i \right)^2 \right) + 2 m \left\lVert D_{j^\star}^{(k)} \right\rVert \alpha^{2{c_U}} \\
    &= \frac{1}{n} \lVert \widetilde{w}_{j^\star} \rVert^2 \widetilde{\overline{w}}_{j^\star}^\top D_{j^\star}^{(k)} \left( \sum_{i \notin S_{j^\star}^{(k)}} \left( \widetilde{\overline{w}}_{j^\star}^\top x_i \right)^2 \right) + 2 m \left\lVert D_{j^\star}^{(k)} \right\rVert \alpha^{2{c_U}} \\
    &<  n \left\lVert D_{j^\star}^{(k)} \right\rVert^2 \alpha^{5 \varepsilon' / 2} + 2 m \left\lVert D_{j^\star}^{(k)} \right\rVert \alpha^{2{c_U}}, \label{phase_2b_j*_stability_network_bound}
\end{align}


where in the final inequality we used \eqref{phase_2b_realign_fitted_data_final_bound} and \eqref{phase_2b_norm_defn}. Putting \eqref{phase_2b_j*_stability_network_bound} into \eqref{phase_2b_j*_stability_equation}, we have for $t \in [\tau_2, \tau_3]$ and for $\alpha$ sufficiently small that

\[
\frac{1}{\log(1/\alpha)}\frac{\df}{\df t}  \widetilde{\overline{w}}^\top_{j^\star} D_{j^\star}^{(k)} > \left\lVert D_{j^\star}^{(k)} \right\rVert^2 - \left( \widetilde{\overline{w}}^\top_{j^\star} D_{j^\star}^{(k)} \right)^2 - \alpha^{2 \varepsilon'  }.
\]

From the initial conditions provided by \Cref{lemma_phase_2a_alignment} combined with the inequality $\sqrt{1-x^2}>1-x$, we obtain for $t \in [\tau_2, \tau_3]$

\begin{equation} \label{phase_2b_j*_alignment_bound_final}
    \widetilde{\overline{w}}^\top_{j^\star} D_{j^\star}^{(k)} > \left\lVert D_{j^\star}^{(k)} \right\rVert \left( 1 - \alpha^{\varepsilon'} \right), 
\end{equation}

which implies the required alignment for $\alpha$ sufficiently small.

\end{subproof}

\begin{subproof}{Proof of \Cref{lemma_phase_2a_fitted_norms}} We verify that the norms of neurons in the set $N_F^{(k)}$ did not change significantly on the interval $[\tau_1, \tau_2]$, which will lead us to conclude in combination with \Cref{pr:uniform_bound_unfitted} that \eqref{phase_2b_lipschitz_defn} did not fail on this interval. \Cref{lemma_phase_2a_fitted_norms} \Cref{lemma_phase_2a} provides upper and lower bounds on the norms of neurons in the set $N_F^{(k)}$ at time $\tau_2$. In light of these initial conditions and the bound on the loss on data in the set $S_F^{(k)}$ given in \eqref{phase_2b_residual_defn}, identical workings to those yielding \eqref{phase_1_initial_lower_bound_fitted_norms} give

\begin{equation}
    \sqrt{n \left\lVert D_j^{(k)} \right\rVert } - \alpha^{\varepsilon'/2} < \lVert \widetilde{w}_j \rVert < \sqrt{n \left\lVert D_j^{(k)} \right\rVert } + \alpha^{\varepsilon'/2}
\end{equation}

for all $t \in [\tau_2, \tau_3]$ for $\alpha$ sufficiently small.
    
\end{subproof}

\begin{subproof}[Proof of \Cref{lemma_phase_2b_j_star_norm}]

For this point, we simply note that \eqref{phase_2b_j*_norm_final_bound} in the proof of \Cref{lemma_phase_2b_duration} \Cref{lemma_phase_2b} gives that if no other condition in the definition of $\tau_3$ breaks then it must be that $\lVert \widetilde{w}_{j^\star} (\tau_3) \rVert = n \left\lVert D_{j^\star}^{(k)} \right\rVert - \alpha^{\varepsilon'/4}$.
\end{subproof}

This completes the proof of \Cref{lemma_phase_2b}.
\end{proof}

\subsection{Phase 3: Realignment and slow growth}

Define the time 
    
\begin{align}
    \tau_4 := \inf \Big\{ t \geq & \tau_3 \, \big\vert \; \exists \; i \in S_F^{(k+1)} \; \mathrm{s.t.} \; | h_{\widetilde{\theta}}(x_i) - y_i | \, \geq \alpha^{\varepsilon'/6} \label{phase_3_residual_defn} \\
    & \;\; \mathrm{or} \; \exists \, j \in N_F^{(k)} \; \mathrm{s.t.} \;\lVert \widetilde{w}_{j} \rVert > 2 \max_{j \in [m]} \left\lVert D_j^{(k)} \right\rVert \Big\}  \bigwedge  \left( \tau_3 + \frac{3 \varepsilon}{2 \left\lVert D_{j^\star}^{(k)} \right\rVert} \right) \label{phase_3_lipschitz_defn} \\ & \hspace{6.435cm} \bigwedge \left( \tau_1 +2 \sqrt{\frac{n}{2\mathcal{L}(0)}} c_U \right).
\end{align}

This phase is that during which neurons in the set $N_U^{(k)} \setminus \{ j^\star \} = N_U^{(k+1)}$ realign in directions $D_j^{(k+1)}$ to respond to the fact that neuron $j^\star$ has grown to reduce the loss on the data on which it is active. We will find that $\tau_4 = \tau_3+\frac{3 \varepsilon}{2 \left\lVert D_{j^\star}^{(k)} \right\rVert}$, and combining this with the bounds on the duration of phases 2a and 2b, we will see that the norms of the neurons in the set $N_U^{(k+1)}$ did not change significantly on the interval $[\tau_1, \tau_4]$. This choice is motivated by the lower bound on $\tau_1$ afforded to us by \eqref{phase_1_max_dist_tk+1}, as we can then guarantee that $\tau_4>t_{k+1}$. The reader may note that in the lemma below, we do not track the alignment of neuron $j^\star$, the reason being that it is now a member of the set $S_F^{(k+1)}$. As noted before \Cref{lemma_phase_2a}, we can avoid laboriously tracking the alignment of all of these neurons throughout each phase, because this is not necessary to ensure the loss on fitted data remains small. Once this has been done, we will be able to estimate the alignment across several phases at once, using the alignment at time $\tau_3$ as the initial condition for neuron $j^\star$ and the alignment at time $\tau_1$ as the initial condition for neurons in the set $N_F^{(k)}$. 

\begin{lem} \label{lemma_phase_3}    
    Suppose that \Cref{inductive_hypothesis} holds. Then for all $\kappa \in (t_{k}, t_{k+1})$, there exists $\varepsilon^*>0$ such that for all $\varepsilon \in (0, \varepsilon^*)$ there exists $\varepsilon'^*$ such that for all $\varepsilon' \in (0, \varepsilon'^*)$ then there exists $\alpha^*>0$ such that for all $\alpha \in (0, \alpha^*)$ the following are true:


    \begin{enumerate}
        \item $t_{k+1} < \tau_4  < t_{k+1} +  \frac{4 \varepsilon}{ \left\lVert D_{j^\star}^{(k)} \right\rVert}  + \frac{\varepsilon'}{\left\lVert D_{j^\star}^{(k)} \right\rVert}$. \label{lemma_phase_3_duration}
        \item $\forall i \in S_F^{(k+1)}, \; \forall t \in [\tau_3, \tau_4], \; | h_{\widetilde{\theta}}(x_i) - y_i | \, < \alpha^{\varepsilon'/6} .$ \label{lemma_phase_3_residual}
        \item $\forall j \in N_F^{(k+1)}, \; \forall t \in [\tau_3, \tau_4], \; \sqrt{n \left\lVert D_j^{(k)} \right\rVert} - \alpha^{\varepsilon'/8} < \lVert \widetilde{w}_j \rVert < \sqrt{n \left\lVert D_j^{(k)} \right\rVert} + \alpha^{\varepsilon'/8}$ \label{lemma_phase_3_fitted_norms}
        \item $\forall j \in N_U^{(k+1)}, \; \widetilde{\overline{w}}_j^\top(\tau_4) \overline{D}_j^{(k+1)} > 1 - \alpha^{\varepsilon'/15} .$ \label{lemma_phase_3_unfitted_alignment}
        \item $\forall i \in N_U^{(k+1)}, \; \forall t \in [\tau_1, \tau_4]$,  \; $\left\vert \ell_j^\circ(t) +  \log_{\alpha} ( \lVert \widetilde{w}_j \rVert )\right\vert < \zeta_j(\varepsilon, \varepsilon')$,
        where \;\;\;\;\;\;\;\;\;\;\;\;\;\;\;\;\;\;\;\;\;\;\;\; $ \zeta_j(\varepsilon, \varepsilon') := \frac{\varepsilon}{2} + \left( \left\lVert D_j^{(k)} \right\rVert + \sqrt{\frac{2\mathcal{L}(0)}{n}} \right) \left( \frac{9 \varepsilon}{2 \left\lVert D_{j^\star}^{(k)} \right\rVert} + \frac{\varepsilon'}{\left\lVert D_{j^\star}^{(k)} \right\rVert }  \right)$. \label{lemma_phase_3_unfitted_norms}
    \end{enumerate}
\end{lem}

\begin{proof}

The reader should note it will be convenient to derive the bounds in \Cref{lemma_phase_3_duration} \Cref{lemma_phase_3} as consequences of analyses for the other points, as it is not possible to derive the upper and lower bounds at the same time. 

\begin{subproof}[Proof of \Cref{lemma_phase_3_residual}] For an upper bound on $\tau_4,$ combining the upper bound on $\tau_1$ provided by \Cref{lemma_phase_1_duration} \Cref{lemma_phase_1}, the upper bound on $\tau_2$ provided by \Cref{lemma_phase_2a_duration} \Cref{lemma_phase_2a}, and the upper bound on $\tau_3$ provided by \Cref{lemma_phase_2b_duration} \Cref{lemma_phase_2b}, we conclude that 

\begin{align}
    \tau_4 &< \tau_1+ \frac{3 \varepsilon}{\left\lVert D_{j^\star}^{(k)} \right\rVert}  + \frac{\varepsilon'}{\left\lVert D_{j^\star}^{(k)} \right\rVert } + \frac{3 \varepsilon}{2 \left\lVert D_{j^\star}^{(k)} \right\rVert} \label{phase_3_tau_4_tau_1_max_diff} \\
    &< t_{k+1} + \frac{4 \varepsilon}{\left\lVert D_{j^\star}^{(k)} \right\rVert} + \frac{\varepsilon'}{\left\lVert D_{j^\star}^{(k)} \right\rVert} .
\end{align}
    
As $\varepsilon$ and $\varepsilon'$ are arbitrary, we may select these quantities sufficiently small that $\tau_3+\frac{3 \varepsilon}{2 \left\lVert D_{j^\star}^{(k)} \right\rVert} < \tau_1 +2 \sqrt{\frac{n}{2\mathcal{L}(0)}} c_U $. In this case, we may apply \Cref{pr:uniform_bound_unfitted} to bound the norms of the neurons in the set $N_U^{(k)} \setminus \{ j^\star \}$. Therefore, we have on $[\tau_3, \tau_4]$ the bound $\lVert \widetilde{w}_j \rVert < \alpha^{c_U}$ for all $j \in N_U^{(k)} \setminus \{ j^\star \}$. To determine the loss on data fitted by neuron $j^\star$, this observation combined with \Cref{lemma_phase_2b_alignment} and \Cref{lemma_phase_2b_j_star_norm} \Cref{lemma_phase_2b} implies that for $i \in S_{j^\star}^{(k)}, \;|h_{\widetilde{\theta}}(x_i) - y_i| \lesssim \alpha^{\varepsilon'/4}$. Furthermore, by definition of $\tau_3$, $\left\lVert E^{(k)}(\tau_3) \right\rVert < \alpha^{\varepsilon'}$. From these facts we conclude that for $\alpha$ sufficiently small,

\begin{equation}
    \left\lVert E^{(k+1)} (\tau_3) \right\rVert < \alpha^{\varepsilon' / 5}.
\end{equation}

We now use arguments similar to those employed in the proof of \Cref{lemma_phase_1_residual} \Cref{lemma_phase_1} to control the residual on fitted data. The same reasoning used to obtain \eqref{derivative_of_residual_basic_bound} leads us to conclude 

\begin{equation}
    \frac{1}{2 \log(1/\alpha) } \frac{\df}{\df t} \left\lVert E^{(k+1)} \right\rVert^2 \leq - \frac{1}{n} \sum_{i \in S_F^{(k+1)}} \sum_{j \in N_U^{(k+1)}} A_{i, j} \lVert \widetilde{w}_j \rVert^2 \fD_j^\top \widetilde{\overline{w}}_j \widetilde{\overline{w}}_j^\top x_i x_i^\top E^{(k+1)}.
\end{equation}

We use \cref{pr:uniform_bound_unfitted} to obtain the estimate

\begin{align} 
    \frac{\df}{\df t} \left\lVert E^{(k+1)} \right\rVert &\leq \frac{m}{n} \sqrt{2 \mathcal{L}(0)} \log(1/\alpha) \alpha^{2{c_U}}  \label{phase_3_ode_unfitted_bound} \\
    &< \alpha^{\varepsilon'}, 
\end{align}

where the final inequality is valid for $\alpha$ sufficiently small and differentiability almost everywhere follows from Rademacher's theorem. Integration yields 

\begin{align}
    \left\lVert E^{(k+1)} \right\rVert < \alpha^{\varepsilon'/5} + (\tau_4-\tau_3)\alpha^{\varepsilon'}.
\end{align}

Hence, 

\begin{align}
    \max_{i \in S_F^{(k+1)}} | h_{\widetilde{\theta}^{t}}(x_i) - y_i | &< n \left( \alpha^{\varepsilon'/5} + \varepsilon' \alpha^{\varepsilon'} \right) \\
    &< \alpha^{\varepsilon' / 6} \label{phase_3_norm_residual_bound} ,
\end{align} 

as claimed. 
\end{subproof}






\begin{subproof}{Proof of \Cref{lemma_phase_3_fitted_norms}} We verify that the norms of neurons in the set $N_F^{(k)}$ did not change significantly on the interval $[\tau_3, \tau_4]$, which will lead us to conclude in combination with \Cref{pr:uniform_bound_unfitted} that the condition on neurons' norms did not fail on this interval. \Cref{lemma_phase_2b_fitted_norms} \Cref{lemma_phase_2b} and \Cref{lemma_phase_2b_j_star_norm} provide upper and lower bounds on the norms of neurons in the set $N_F^{(k+1)}$ at time $\tau_3$. In light of these initial conditions and the bound on the loss on data in the set $S_F^{(k+1)}$ given in \eqref{phase_3_residual_defn}, identical workings to those yielding \eqref{phase_1_initial_lower_bound_fitted_norms} give

\begin{equation}
    \sqrt{n \left\lVert D_j^{(k)} \right\rVert } - \alpha^{\varepsilon'/8} < \lVert \widetilde{w}_j \rVert < \sqrt{n \left\lVert D_j^{(k)} \right\rVert } + \alpha^{\varepsilon'/8}
\end{equation}

for all $t \in [\tau_3, \tau_4]$ for $\alpha$ sufficiently small. Since none of the other conditions in the definition of $\tau_4$ break, as a result of our analysis so far, we conclude that $\tau_4 = \tau_3  + \frac{3 \varepsilon}{2 \left\lVert D_{j^\star}^{(k)} \right\rVert}$. This leads to the easy lower bound

\begin{align}
    \tau_4 &> \tau_1+\frac{3 \varepsilon}{2 \left\lVert D_{j^\star}^{(k)} \right\rVert}  \\
    &> t_{k+1},
\end{align}

where the first inequality is by definition of $\tau_2, \tau_3,$ and $\tau_4$ and the second is from the lower bound on $\tau_1$ provided by \Cref{lemma_phase_1_duration} \Cref{lemma_phase_1}
    
\end{subproof}

\begin{subproof}[Proof of \Cref{lemma_phase_3_unfitted_alignment}] We verify that by time $\tau_3+\frac{3 \varepsilon}{2 \left\lVert D_{j^\star}^{(k)} \right\rVert}$, the neurons in the set $N_U^{(k+1)}$ have realigned. We first check that in the interval $[\tau_1, \tau_3]$, no such neurons deactivated on the data in the set $S_j^{(k+1)}$. We recall that $\frac{\df}{\df t} \widetilde{w}_j^\top x_i = - \frac{1}{n} \iind{ w_j^\top x_i > 0} \log(1/\alpha) ( h_{\theta^t}(x_i) - y_i) \, \lVert \widetilde{w}_j \rVert.$ By \eqref{phase_2_unfitted_uniformly_small}, on $[\tau_1, \tau_3]$ we had for every $i \in S_U^{(k+1)}$ that $ h_{\theta^t}(x_i) - y_i \leq 0$ for $\alpha$ sufficiently small. From the alignment at time $\tau_1$ given by \Cref{lemma_phase_1_alignment} \Cref{lemma_phase_1} we conclude that no deactivation occurred. Similarly to the proof of \Cref{induction_basis_case_lemma}, we estimate 

\begin{align}
    \frac{1}{\log(1/\alpha)} \frac{\df}{\df t} \widetilde{\overline{w}}_j^\top D_j^{(k+1)} &=  \left\lVert D_j^{(k+1)} \right\rVert^2 - \left( \widetilde{\overline{w}}^\top_j D_j^{(k+1)} \right)^2 \\
    &\qquad - \frac{1}{n} \sum_{i \in S_j^{(k+1)}} h_{\widetilde{\theta}}(x_i) \left( x_i^\top D_j^{(k+1)} - x^\top_i \widetilde{\overline{w}}_j \widetilde{\overline{w}}^\top_j D_j^{(k+1)} \right) \\ 
    &\qquad - \frac{1}{n} \sum_{i \in S_F^{(k+1)}} ( h_{\widetilde{\theta}}(x_i) - y_i) \left( x_i^\top D_j^{(k+1)} - x^\top_i \widetilde{\overline{w}}_j \widetilde{\overline{w}}^\top_j D_j^{(k+1)} \right)  \\
    &> \left\lVert D_j^{(k+1)} \right\rVert^2 - \left( \widetilde{\overline{w}}^\top_j D_j^{(k+1)} \right)^2 \\
    &\qquad- 2 m \left\lVert D_j^{(k+1)} \right\rVert \alpha^{2 c_U} - 2 \left\lVert D_j^{(k+1)} \right\rVert \alpha^{\varepsilon' / 6} \\
    &> \left\lVert D_j^{(k+1)} \right\rVert^2 - \left( \widetilde{\overline{w}}^\top_j D_j^{(k+1)} \right)^2 - \alpha^{\varepsilon' / 7},
\end{align}

where the final inequality is valid for $\alpha$ sufficiently small. Similar workings to those which led to \eqref{eq:basis_case_alignment_bound_general} give 

\begin{align}
    &\widetilde{\overline{w}}_j^\top \left( \tau_3 + \frac{3 \varepsilon}{2 \left\lVert D_{j^\star}^{(k)} \right\rVert} \right) D_j^{(k+1)} \\
    &\qquad> \left( \left\lVert D_j^{(k+1)} \right\rVert - \alpha^{\varepsilon' / 14} \right) \left( 1 - \exp_\alpha \left( \left( \left\lVert D_j^{(k+1)} \right\rVert - \varepsilon'/14 \right) \frac{3 \varepsilon}{2 \left\lVert D_{j^\star}^{(k)} \right\rVert} \right) \right) \\
    &\qquad> \left\lVert D_j^{(k+1)} \right\rVert - \alpha^{ \varepsilon'/15},
\end{align}

where the final inequality is valid for $\varepsilon'$ selected sufficiently small dependent on $\varepsilon$ (and $\alpha$ sufficiently small in turn).
\end{subproof}

\begin{subproof}[Proof of \Cref{lemma_phase_3_unfitted_norms}]

At this point, we are in a position to estimate the deviation of the norms of the neurons $j \in N_U^{(k+1)}$ from that predicted by \Cref{alg:s.t.s}. A simple Gr{\"o}nwall argument gives that

\begin{align}
    \lVert \widetilde{w}_j(\tau_4) \rVert &\leq \lVert \widetilde{w}_j(\tau_1) \rVert  \alpha^{- (\tau_4-\tau_1) \sqrt{\frac{2 \mathcal{L}(0)}{n}}} \\
    &< \exp_\alpha \left( -\ell_j^\circ(\tau_1) - \frac{\varepsilon}{2} - \sqrt{\frac{2\mathcal{L}(0)}{n}} \left(  \frac{9 \varepsilon}{2 \left\lVert D_{j^\star}^{(k)} \right\rVert} + \frac{\varepsilon'}{\left\lVert D_{j^\star}^{(k)} \right\rVert } \right) \right), \label{phase_3_unfitted_lower_bound} 
\end{align}

where above we have combined \eqref{phase_1_unfitted_norm_upper_bound_tau1} with the upper bound on $\tau_4-\tau_1$ from \eqref{phase_3_tau_4_tau_1_max_diff}. Note that $|\ell_j^\circ(t)-\ell_j^\circ(t')| \leq \left\lVert D_j^{(k)} \right\rVert(t-t')$ for any $t,t' \in (t_k, \infty)$. Using \eqref{phase_3_tau_4_tau_1_max_diff} again, we get

\begin{align}
    |-\ell_j^\circ(\tau_1) + \ell_j^\circ(\tau_4) | &<  \left\lVert D_j^{(k)} \right\rVert | \tau_1 - \tau_4 | \\
    &< \left\lVert D_j^{(k)} \right\rVert \left( \frac{9 \varepsilon}{2 \left\lVert D_{j^\star}^{(k)} \right\rVert} + \frac{\varepsilon'}{\left\lVert D_{j^\star}^{(k)} \right\rVert } \right) . \label{phase_3_abs_diff_ell_js}
\end{align}

Putting \eqref{phase_3_abs_diff_ell_js} back into \eqref{phase_3_unfitted_lower_bound}, we conclude that for any $j \in N_U^{(k+1)}$ and for any $t \in [\tau_1, \tau_4]$,

\begin{align}
    \lVert \widetilde{w}_j(\tau_4) \rVert &< \exp_\alpha \left( -\ell_j^\circ(\tau_4) - \frac{\varepsilon}{2} - \left( \left\lVert D_j^{(k)} \right\rVert + \sqrt{\frac{2\mathcal{L}(0)}{n}} \right) \left( \frac{9 \varepsilon}{2 \left\lVert D_{j^\star}^{(k)} \right\rVert} + \frac{\varepsilon'}{\left\lVert D_{j^\star}^{(k)} \right\rVert }  \right) \right)  . \label{phase_3_unfitted_final_upper_bound} 
\end{align}

For the converse inequality, we argue in an identical fashion to obtain

\begin{align}
    \lVert \widetilde{w}_j(\tau_4) \rVert &> \exp_\alpha \left( -\ell_j^\circ(\tau_4) + \frac{\varepsilon}{2} + \left( \left\lVert D_j^{(k)} \right\rVert + \sqrt{\frac{2\mathcal{L}(0)}{n}} \right) \left( \frac{9 \varepsilon}{2 \left\lVert D_{j^\star}^{(k)} \right\rVert} + \frac{\varepsilon'}{\left\lVert D_{j^\star}^{(k)} \right\rVert}  \right) \right).  \label{phase_3_unfitted_final_lower_bound}
\end{align}

\end{subproof}
This completes the proof of \Cref{lemma_phase_3}.
\end{proof}

We can now provide a result from which \Cref{lem:network_post_saddle} follows as a corollary. We will verify that the alignment, norm, and residual conditions propagate forwards after the jump having applied \Cref{lemma_phase_3}. We will also verify that the norms of neurons in the set $N_U^{(k+1)}$ consequently grow at a rate close to that predicted by \Cref{alg:s.t.s} after the jump. Define for any $\kappa'' \in (t_{k+1}, t_{k+2})$ the time 

\begin{align}
    \tau_5 := \inf \Big\{ t \geq & \tau_4 \; \big\vert \; \exists \, j \in [m] \; \text{s.t.} \; \widetilde{\overline{w}}_j^\top \overline{D}_j^{(k+1)} \leq 1 - \alpha^{\varepsilon'/32}  \\
        & \;\; \mathrm{or} \; \exists \, i \in S_F^{(k+1)} \; \mathrm{s.t.} \; | h_{\widetilde{\theta}}(x_i) - y_i | \, \geq \alpha^{\varepsilon'/8} \label{phase_4_residual_defn}\\
        & \;\; \mathrm{or} \; \exists \, j \in N_F^{(k+1)} \; \mathrm{s.t.} \;\lVert \widetilde{w}_{j} \rVert > 2 \max_{j \in [m]} \left\lVert D_j^{(k+1)} \right\rVert \label{phase_4_lipschitz_defn} \\
        & \;\; \mathrm{or} \; \exists \, j \in N_U^{(k+1)} \; \mathrm{s.t.} \; \lVert \widetilde{w}_j \rVert \geq \alpha^{\varepsilon'} \label{phase_4_norm_defn} \Big\} \bigwedge  \kappa''
\end{align}

This phase is very similar to that described in \Cref{phase_1}, in that it is also a period of slow growth during which the network parameters do not change significantly. Indeed, the structure of the proof is almost identical to that of \Cref{lemma_phase_1}.

\begin{lem} \label{lemma_phase_4}
    For all $\kappa$ and $\kappa''$ such that $\kappa \in (t_k, t_{k+1})$ and $\kappa'' \in (t_{k+1}, t_{k+2})$, there exists $\varepsilon^*$ such that for all $\varepsilon \in (0, \varepsilon^*)$ there exists $\varepsilon'^*$ such that for all $\varepsilon' \in (0, \varepsilon'^*)$ there exists $\alpha^*$ such that for all $\alpha \in (0, \alpha^*)$ the following are true: 

    \begin{enumerate}
        \item $\tau_5 = \kappa''$. \label{lemma_phase_4_duration}
        \item $\forall i \in S_F^{(k+1)}, \, \forall t \in [\tau_4, \tau_5], \; |h_{\widetilde{\theta}}(x_i)-y_i|<\alpha^{\varepsilon'/8}$. \label{lemma_phase_4_residual}
        \item $\forall j \in [m], \, \forall t \in [\tau_4, \tau_5], \; \widetilde{\overline{w}}_j^\top \overline{D}_j^{(k+1)} > 1 - \alpha^{\varepsilon'/32}$ \label{lemma_phase_4_alignment}
        \item $\forall j \in N_F^{(k+1)}, \forall t \in [\tau_4, \tau_5], \; \sqrt{n \left\lVert D_j^{(k)} \right\rVert} - \alpha^{\varepsilon'/10} < \lVert \widetilde{w}_j \rVert < \sqrt{n \left\lVert D_j^{(k)} \right\rVert} + \alpha^{\varepsilon'/10}$. \label{lemma_phase_4_fitted_norms}
        \item $\forall j \in N_U^{(k+1)}, \, \forall t \in [\tau_4, \tau_5], \; \left\vert \ell_j^\circ(t) +  \log_{\alpha} ( \lVert \widetilde{w}_j \rVert )\right\vert < \zeta_j(\varepsilon, \varepsilon') + \varepsilon'$, where $\zeta_j(\varepsilon, \varepsilon')$ is as defined in \Cref{lemma_phase_3}. \label{lemma_phase_4_unfitted_norms}
    \end{enumerate}
\end{lem}

\begin{proof}

Given any $\kappa'' \in (t_{k+1}, t_{k+2})$, \Cref{lemma_phase_3_duration} \Cref{lemma_phase_3} guarantees that we may select $\varepsilon$ and $\varepsilon'$ sufficiently small that $\tau_4 < \kappa''$. Assume that this has been done and that $\alpha$ has been selected sufficiently small that all of the conditions of \Cref{lemma_phase_3} hold. It remains to check that all of these conditions remain stable for $t \in [\tau_4, \kappa'']$. 

\begin{subproof}[Proof of \Cref{lemma_phase_4_residual}]

In a similar manner to workings for pervious phases, we have only to estimate 

\begin{equation}
    \frac{1}{2 \log(1/\alpha) } \frac{\df}{\df t} \left\lVert E^{(k+1)} \right\rVert^2 \leq - \frac{1}{n} \sum_{i \in S_F^{(k+1)}} \sum_{j \in N_U^{(k+1)}} A_{i, j} \lVert \widetilde{w}_j \rVert^2 \fD_j^\top \widetilde{\overline{w}}_j \widetilde{\overline{w}}_j^\top x_i x_i^\top E^{(k+1)}.
\end{equation}

From the uniform bound on $j \in N_U^{(k+1)}$ given by \eqref{phase_4_norm_defn}, we have for $i \in S_U^{(k+1)}$ that on $[\tau_4, \kappa'']$, $ |h_{\widetilde{\theta}}(x_i)| < m \alpha^{2 \varepsilon'}$. Hence, 

\begin{equation}
    \frac{\df}{\df t} \left\lVert E^{(k+1)} \right\rVert \leq \frac{m}{n} \sqrt{2 \mathcal{L}(0)} \log(1/\alpha) \alpha^{2 \varepsilon'}. 
\end{equation}

Integration implies that 

\begin{align}
    \left\lVert E^{(k+1)} \right\rVert &< \alpha^{\varepsilon'/6} + \frac{m}{n} \sqrt{2 \mathcal{L}(0)} \log(1/\alpha) \alpha^{2 \varepsilon'} (\kappa'' - \tau_4) .
\end{align}

Hence, 

\begin{align}
    \max_{i \in S_U^{(k+1)}} | h_{\widetilde{\theta}}(x_i) - y_i| &<  n \left( \alpha^{\varepsilon'/6} + \frac{m}{n} \sqrt{2 \mathcal{L}(0)} \log(1/\alpha) \alpha^{2 \varepsilon'} (\kappa'' - \tau_4) \right) \\
    &< \alpha^{\varepsilon' / 8}
\end{align}

for $\alpha$ sufficiently small, as required.
\end{subproof}

\begin{subproof}[Proof of \Cref{lemma_phase_4_alignment}] We again split this argument into the cases $j = j^\star$, $j \in N_F^{(k+1)} \setminus \{ j^\star \}$, and $j \in N_U^{(k+1)}$. For the case $j \in N_F^{(k+1)} \setminus \{ j^\star \}$, note that \Cref{lemma_phase_1_alignment} \Cref{lemma_phase_1} provides an initial condition for the alignment of these neurons in the directions $D_j^{(k+1)}$ (since for such neurons, $D_j^{(k+1)} = D_j^{(k)}$ by construction). We can reuse arguments similar to those which led to \eqref{phase_1_alignment_derivative_bound_fitted} to obtain

\begin{align}
    \frac{1}{\log(1/\alpha)} \frac{\df}{\df t} \widetilde{\overline{w}}_j^\top D_j^{(k)} &> - \frac{2 \left\lVert D_j^{(k+1)} \right\rVert}{\sqrt{n}} \alpha^{\varepsilon'/8},
\end{align} 

given that $\lVert E^{(k)} \rVert < \alpha^{\varepsilon'/8}$ for $t \in [\tau_1, \tau_5]$. Integrating over this time interval, we get

\begin{align}
    \widetilde{\overline{w}}_j^\top D_j^{(k)} &> \widetilde{\overline{w}}_j^\top(\tau_1) D_j^{(k)} - \frac{2 \left\lVert D_j^{(k+1)} \right\rVert}{\sqrt{n}} \log(1/\alpha) \alpha^{\varepsilon'/8} (\tau_5-\tau_1) \\
    &> \left\lVert D_j^{(k+1)} \right\rVert \left(1-\alpha^{\varepsilon/2} \right) - \frac{2 \left\lVert D_j^{(k+1)} \right\rVert}{\sqrt{n}} \log(1/\alpha) \alpha^{\varepsilon'/8} (\tau_5-\tau_1) \\
    &>  \left\lVert D_j^{(k)} \right\rVert \left(1-\alpha^{\varepsilon' / 9} \right),
\end{align} 

where the final inequality is valid for $\alpha$ sufficiently small. For the case $j = j^\star$, note that \Cref{lemma_phase_2b_alignment} \Cref{lemma_phase_2b} provides an initial condition for the alignment of neuron $j^\star$ in the direction $D_{j^\star}^{(k)}$ (where again we have $D_{j^\star}^{(k+1)} = D_{j^\star}^{(k)}$ by construction). Identical reasoning to the case $j \in N_F^{(k+1)} \setminus \{ j^\star \}$, except integrating from $\tau_3$ to $\tau_5$ with initial conditions $\widetilde{\overline{w}}^\top_j(\tau_3) D_j^{(k+1)} > \left\lVert  D_j^{(k+1)} \right\rVert \left( 1 - \alpha^{\varepsilon'} \right)$, gives the desired lower bound for $\alpha$ sufficiently small. For the case $j \in N_U^{(k+1)}$, we have

\begin{align}
    \frac{1}{\log(1/\alpha)}\frac{\df}{\df t}  \widetilde{\overline{w}}^\top_j D_j^{(k+1)} &= \left\lVert D_j^{(k+1)} \right\rVert^2 - \left( \widetilde{\overline{w}}^\top_j D_j^{(k+1)}\right)^2 \\
    &\qquad - \frac{1}{n} \sum_{ i \in S_j^{(k+1)}} h_{\widetilde{\theta}}(x_i) \left( x_i^\top D_j^{(k+1)} - x_i^\top \widetilde{\overline{w}}_j \widetilde{\overline{w}}^\top_{j} D_j^{(k+1)} \right) \\ 
    &\qquad - \frac{1}{n} \sum_{i \in S_F^{(k+1)}} A_{i,j}( h_{\widetilde{\theta}}(x_i) - y_i) \left( x_i^\top D_j^{(k+1)} - x^\top_i \widetilde{\overline{w}}_j \widetilde{\overline{w}}^\top_{j} D_j^{(k+1)} \right)  \\
    &> \left\lVert D_j^{(k+1)} \right\rVert^2 - \left( \widetilde{\overline{w}}^\top_j D_j^{(k+1)} \right)^2 \\
    &\qquad - 2 \left\lVert D_j^{(k+1)} \right\rVert m \alpha^{129 \varepsilon' / 64} - 2 \left\lVert D_j^{(k+1)} \right\rVert \alpha^{9 \varepsilon' / 64} \\
    &> \left\lVert D_j^{(k+1)} \right\rVert^2 - \left( \widetilde{\overline{w}}^\top_j D_j^{(k+1)} \right)^2 - \alpha^{\varepsilon' / 8}.
\end{align}

By similar arguments to previous phases, we have convergence of alignment monotonically for $t \in [\tau_4, \tau_5]$ to a value lower bounded by $\left\lVert D_j^{(k+1)} \right\rVert \left( 1 - \alpha^{\varepsilon' / 16} \right)$. Combining this with the initial conditions from \Cref{lemma_phase_3_unfitted_alignment} \Cref{lemma_phase_3}, we conclude that 

\begin{equation}
    \widetilde{\overline{w}}^\top_j D_j^{(k+1)} > \left\lVert D_j^{(k+1)} \right\rVert \left( 1 - \alpha^{\varepsilon' / 32} \right)
\end{equation}

for all $t \in [\tau_4, \tau_5]$ provided $\alpha$ is selected sufficiently small.
\end{subproof}

\begin{subproof}{Proof of \Cref{lemma_phase_4_fitted_norms}} We verify that the norms of neurons in the set $N_F^{(k)}$ did not change significantly on the interval $[\tau_4, \tau_5]$, which will lead us to conclude in combination with \eqref{phase_4_norm_defn} that \eqref{phase_4_lipschitz_defn} did not fail on this interval. \Cref{lemma_phase_3_fitted_norms} \Cref{lemma_phase_3} provides upper and lower bounds on the norms of neurons in the set $N_F^{(k+1)}$ at time $\tau_4$. In light of these initial conditions and the bound on the loss on data in the set $S_F^{(k+1)}$ given in \eqref{phase_4_residual_defn}, identical workings to those yielding \eqref{phase_1_initial_lower_bound_fitted_norms} give

\begin{equation}
    \sqrt{n \left\lVert D_j^{(k)} \right\rVert } - \alpha^{\varepsilon'/10} < \lVert \widetilde{w}_j \rVert < \sqrt{n \left\lVert D_j^{(k)} \right\rVert } + \alpha^{\varepsilon'/10}
\end{equation}

for all $t \in [\tau_4, \tau_5]$ for $\alpha$ sufficiently small.
  
\end{subproof}

\begin{subproof}[Proof of \Cref{lemma_phase_4_unfitted_norms,lemma_phase_4_duration}] To estimate the growth of any neuron $j \in N_U^{(k+1)}$, we follow similar arguments to those which yielded \eqref{phase_1_initial_norm_lower_bound} in the proof of \Cref{lemma_phase_1} to obtain

\begin{align}
    \frac{1}{\log(1/\alpha)} \frac{\df}{\df t} \lVert \widetilde{w}_j \rVert &> \left( \left\lVert D_j^{(k+1)} \right\rVert - \alpha^{\varepsilon' / 32} - m \alpha^{2 \varepsilon'} \right) \lVert \widetilde{w}_j \rVert - \sqrt{2} \alpha^{9\varepsilon'/64} \lVert \widetilde{w}_j \rVert \\
    &> \left( \left\lVert D_j^{(k+1)} \right\rVert - \alpha^{\varepsilon' / 33}  \right) \lVert \widetilde{w}_j \rVert.
\end{align}

Gr{\"o}nwall's inequality yields for $\alpha$ sufficiently small the inequalities 

\begin{align}
    \lVert \widetilde{w}_j(t) \rVert &> \lVert \widetilde{w}_j(\tau_4) \rVert \alpha^{- \left\lVert D_j^{(k+1)} \right\rVert (t - \tau_4) + \varepsilon'} \\
    &> \alpha^{- \ell_j^\circ(t) + \zeta_j(\varepsilon, \varepsilon') + \varepsilon' }, 
\end{align}

where in the second inequality we inserted \eqref{phase_3_unfitted_final_lower_bound}. Analogous arguments yield the claimed upper bound 

\begin{equation}
    \lVert \widetilde{w}_j(t) \rVert <  \alpha^{- \ell_j^\circ(t) - \zeta_j(\varepsilon, \varepsilon') - \varepsilon' }, 
\end{equation}

where we have inserted \eqref{phase_3_unfitted_final_upper_bound}. By construction, for every $j \in N_U^{(k+1)}$, $-\ell_j^\circ(\kappa'') > 0$. Therefore, we may select $\varepsilon$ and $\varepsilon'$ sufficiently small that for all such $j$, $\lVert \widetilde{w}_j(\kappa'') \rVert < \alpha^{\varepsilon'}$. For this choice of $\varepsilon$ and $\varepsilon'$, we conclude that necessarily $\tau_5 = \kappa''$, since none of the other conditions in the definition of $\tau_5$ break.
\end{subproof}
This completes the proof of \Cref{lemma_phase_4}.
\end{proof}

\subsection{Final Saddle} \label{final_saddle}

We provide a lemma which is a version of \Cref{lem:network_pre_saddle} for $t>t_p$. Specifically, we prove uniform convergence under the assumption the conditions of \Cref{inductive_hypothesis} hold at some time $\kappa>t_p$.

\begin{lem} \label{lemma_final_saddle}
    Suppose that for some $\kappa \in (t_p, \infty)$, there exists $\delta>0$ such that for all $\xi>0$, there exists $\alpha^*>0$ such that for all $\alpha \in (0,\alpha^*)$:
    
    \begin{enumerate}
        \item  $\forall j \in N_U^{(p)}, \; \left\vert \ell_j^\circ(\kappa) +  \log_{\alpha} ( \lVert \widetilde{w}_j(\kappa) \rVert )\right\vert < \xi .$ \label{final_saddle_unfitted_norms}
        \item $\forall j \in N_F^{(p)}, \; \Big\vert n \left\lVert D_j^{(p)} \right\rVert - \lVert \widetilde{w}_j(\kappa) \rVert^2 \Big\vert < \alpha^{\delta} . $ \label{final_saddle_fitted_norms}
        \item $\forall i \in [n], \; | h_{\widetilde{\theta}(\kappa)}(x_i) - y_i | < \alpha^{\delta} .$  \label{final_saddle_residual}
        \item $\forall j \in [m], \; \widetilde{\overline{w}}^\top_j(\kappa) \overline{D}^{(p)}_j > 1 - \alpha^{\delta} .$ \label{final_saddle_alignment}
    \end{enumerate} 

    Then given any $\kappa' \in [\kappa, \infty)$, $\widetilde{\theta}^\alpha$ converges uniformly to $\theta^\circ$ on $[\kappa, \kappa']$. 
\end{lem}

\begin{proof}

We note that \Cref{final_saddle_residual} \Cref{lemma_final_saddle} and the monotonicity of the gradient flow imply that for all $j \in [m]$ and for all $t \in [\kappa, \infty)$, $\left\lVert \fD_j \right\rVert < \frac{\alpha^{\delta}}{\sqrt{n}}$. To estimate the change in the network parameters from $\kappa$ to $\kappa'$, we have that for each $j \in [m]$,

\begin{align}
    \frac{1}{\log(1/\alpha)} \frac{\df}{\df t} \lVert \widetilde{w}_j \rVert &< \frac{\alpha^{\delta}}{\sqrt{n}} \lVert \widetilde{w}_j \rVert.
\end{align}

Gr{\"o}nwall's inequality gives 

\begin{align}
    \lVert \widetilde{w}_j \rVert &< \lVert \widetilde{w}_j (\kappa) \rVert \exp \left( \log(1/\alpha) \frac{\alpha^\delta}{\sqrt{n}} (t-\kappa) \right) \\
    &< \lVert \widetilde{w}_j (\kappa) \rVert \exp \left( \log(1/\alpha) \frac{\alpha^\delta}{\sqrt{n}} (\kappa'-\kappa) \right) \\
    &< \lVert \widetilde{w}_j (\kappa) \rVert \exp \left( \alpha^{\delta/2} \right) \\
    &< \lVert \widetilde{w}_j (\kappa) \rVert \left( 1 +  \alpha^{\delta/2} \right),
\end{align}

where the last two inequalities are valid for $\alpha$ sufficiently small. For $j \in N_U^{(p)}$, we obtain that for $\alpha$ sufficiently small, $\lVert \widetilde{w}_j \rVert < 2 \exp_\alpha \left(\frac{1}{2} \min_{j \in N_U^{(p)}} -\ell_j^\circ(\kappa) \right)$, by selecting $\xi < \frac{1}{2} \min_{j \in N_U^{(p)}} -\ell_j^\circ(\kappa)$ in \Cref{final_saddle_unfitted_norms} \Cref{lemma_final_saddle}. Noting that for $ \in N_U^{(p)}$, $-\ell_j^\circ(\kappa) > 0$, we conclude that such neurons converge uniformly to $(\mathbf{0})_{j \in N_U^{(p)}}$. Analogous arguments yield the lower bound

\begin{equation}
    \lVert \widetilde{w}_j \rVert > \lVert \widetilde{w}_j (\kappa) \rVert \left( 1 - \alpha^{\delta/2} \right).
\end{equation}

For $j \in N_F^{(p)}$ and $\alpha$ sufficiently small, we use \Cref{final_saddle_fitted_norms} \Cref{lemma_final_saddle} to obtain 

\begin{equation}
    \sqrt{n \left\lVert D_j^{(k)} \right\rVert} - \alpha^{\delta/3} < \lVert \widetilde{w}_j \rVert < \sqrt{n \left\lVert D_j^{(k)} \right\rVert} + \alpha^{\delta/3}. \label{final_saddle_fitted_norms_bound}
\end{equation}

For alignment of neurons $j \in N_F^{(p)}$, we obtain 

\begin{align}
    \frac{1}{\log(1/\alpha)} \frac{\df}{\df t} \widetilde{\overline{w}}_j^\top D_j^{(k)} &> - \frac{2 \left\lVert D_j^{(k)} \right\rVert }{ \sqrt{n} } \alpha^{\delta}.
\end{align}

Simple integration yields 

\begin{align}
    \widetilde{\overline{w}}_j^\top D_j^{(p)} &> \widetilde{\overline{w}}_j^\top(\kappa) D_j^{(p)} - \frac{2 \left\lVert D_j^{(p)} \right\rVert }{ \sqrt{n} } \alpha^{\delta} (t - \kappa) \\
    &> \left\lVert D_j^{(p)} \right\rVert \left( 1 - \alpha^{\delta} \right) - \frac{2 \left\lVert D_j^{(p)} \right\rVert }{ \sqrt{n} } \alpha^{\delta} (\kappa' - \kappa) \\
    &> \left\lVert D_j^{(p)} \right\rVert \left( 1 - \alpha^{\delta/2} \right), \label{final_saddle_alignment_fitted}
\end{align}

where the final inequality is valid for $\alpha$ sufficiently small. Combining \eqref{final_saddle_fitted_norms_bound} and \eqref{final_saddle_alignment_fitted}, we get for $j \in N_F^{(k)}$

\begin{align}
    \Big\lVert \widetilde{a}_j \widetilde{w}_j - n D_{j}^{(p)} \Big\rVert^2 &= \lVert \widetilde{w}_j \rVert^4 + n^2 \left\lVert D_{j}^{(p)} \right\rVert^2 - 2 n \widetilde{a}_j \widetilde{w}_j^\top D_j^{(p)} \\
    &< \left( \sqrt{n \left\lVert D_j^{(p)} \right\rVert} + \alpha^{\delta/3} \right)^4 + n^2 \left\lVert D_j^{(p)} \right\rVert^2 \\
    &\qquad - 2n \left\lVert D_j^{(p)} \right\rVert \left( \sqrt{n \left\lVert D_j^{(p)} \right\rVert} - \alpha^{\delta/3} \right)^2 \left( 1 - \alpha^{\delta/2} \right) \\
    &< \alpha^{\delta/4}, \label{final_saddle_unform_convergence_fitted}
\end{align}

for $\alpha$ sufficiently small. This yields the required uniform convergence on $[\kappa, \kappa']$ for neurons in the set $N_F^{(p)}$.
\end{proof}

\subsection{Proof of \texorpdfstring{\Cref{th:s.t.s}}{Theorem \ref*{th:s.t.s}}} \label{sect:final_proof_thm_2}

In this section, we prove \Cref{lem:network_pre_saddle} and \Cref{lem:network_post_saddle} using the multiphase analysis of the previous subsections. Once these two lemmas are proven, it will be possible to provide a short proof of \Cref{th:s.t.s}. The proof of \Cref{lem:network_pre_saddle} is a fairly direct application of \Cref{lemma_phase_1} (and in the case of the final saddle, the proof is provided in \Cref{lemma_final_saddle}).

\begin{proof}[Proof of \Cref{lem:network_pre_saddle}] 

From \cref{lemma_phase_1_duration} \Cref{lemma_phase_1}, we conclude there exists $\varepsilon$ sufficiently small such that there exists $\alpha$ sufficiently small that $\tau_1 > \kappa'$. For such $\varepsilon$ and $\alpha$, we have by definition of $\tau_1$ that for all $j \in N_U^{(k)}$ and for all $t \in [\kappa, \kappa']$, $\lVert \widetilde{w}_j \rVert<\alpha^{\varepsilon}$. As a consequence, neurons in the set $N_U^{(k)}$ converge uniformly to $(\mathbf{0})_{j \in N_U^{(k)}}$. For the neurons in the set $N_F^{(k)}$, we follow similar workings to those leading to \eqref{final_saddle_unform_convergence_fitted}. The combination of \Cref{lemma_phase_1_alignment} \Cref{lemma_phase_1} and \Cref{lemma_phase_1_fitted_norms} \Cref{lemma_phase_1} imply 

\begin{align}
    \Big\lVert \widetilde{a}_j \widetilde{w}_j - n D_{j}^{(k)} \Big\rVert^2 &= \big\lVert \widetilde{w}_j \big\rVert^4 + n^2 \left\lVert D_{j}^{(k)} \right\rVert^2 - 2 n \widetilde{a}_j \widetilde{w}_j^\top D_j^{(k)} \\
    &< \left( \sqrt{n \left\lVert D_j^{(k)} \right\rVert} + \alpha^{\varepsilon/4} \right)^4 + n^2 \left\lVert D_j^{(k)} \right\rVert^2 \\
    &\qquad - 2n \left\lVert D_j^{(k)} \right\rVert \left( \sqrt{n \left\lVert D_j^{(k)} \right\rVert} - \alpha^{\varepsilon/4} \right)^2 \left( 1 - \alpha^{\varepsilon/2} \right) \\
    &< \alpha^{\varepsilon/5},
\end{align}

for $\alpha$ sufficiently small. Therefore, $(\widetilde{a}_j \widetilde{w}_j)_{j \in N_F^{(k)}}$ converges uniformly to $\left( n D_j^{(k)} \right)_{j \in N_F^{(k)}}$ on $[\kappa, \kappa']$.
\end{proof}

Next, we prove \Cref{lem:network_post_saddle}, which recovers \Cref{inductive_hypothesis} at saddle $k+1$ and allows us to construct the argument for uniform convergence of the whole process. To prove this lemma, we require that given any any $\kappa'' \in (t_{k+1}, t_{k+2})$ (where we will take the convention that $t_{p+1} = \infty$), there exists $\delta''>0$ such that for all $\xi''>0$, there exists $\alpha_*>0$ such that for all $\alpha \in (0,\alpha^*)$:

    \begin{enumerate}
        \item $\forall j \in N_U^{(k+1)}, \; \left\vert \ell_j^\circ(\kappa'') +  \log_{\alpha} ( \lVert \widetilde{w}_j(\kappa'') \rVert )\right\vert < \xi''. $ \label{unfitted_norms_network_post_saddle}
        \item $\forall j \in N_F^{(k+1)}, \; \Big\vert n \left\lVert D_j^{(k+1)} \right\rVert - \lVert \widetilde{w}_j(\kappa'') \rVert^2 \Big\vert < \alpha^{\delta''} $ \label{fitted_norms_network_post_saddle}
        \item $\forall i \in S_F^{(k+1)}, \; | h_{\widetilde{\theta}(\kappa'')}(x_i) - y_i | < \alpha^{\delta''} .$ \label{residual_fitted_network_post_saddle}
        \item $\forall j \in [m], \; \widetilde{\overline{w}}^\top_j(\kappa'') \overline{D}^{(k+1)}_j > 1 - \alpha^{\delta''} .$ \label{alignment_network_post_saddle}
    \end{enumerate} 
    
The proof consists of two applications of \Cref{lemma_phase_4}: one choice of $\varepsilon, \varepsilon'$ in order to get the existence of $\delta''$ such that the alignment, norm, and loss conditions all hold at time $\kappa''$; the other choice of $\varepsilon, \varepsilon'$ in order to get the arbitrary precision of the norms of the unfitted neurons. That the same choice does not work to get all of the points at once is a result of the fact that our analysis ends up linking the alignment, norm, and loss conditions with the precision of the unfitted neurons' norms, so it is impossible to prove \Cref{lem:network_post_saddle} with a single choice of $\varepsilon, \varepsilon'$ (otherwise one would have to have $\delta''$ depending on $\xi''$).

\begin{proof}[Proof of \Cref{lem:network_post_saddle}]

Note that \Cref{lemma_phase_4} immediately guarantees that for $\varepsilon_0$ sufficiently small, for $\varepsilon'_0$ sufficiently small, and for $\alpha$ sufficiently small \Cref{alignment_network_post_saddle}, \Cref{fitted_norms_network_post_saddle}, and  \Cref{residual_fitted_network_post_saddle} immediately above all hold with $\delta'' = \varepsilon'_0/32$. For \Cref{alignment_network_post_saddle} immediately above, we note that \Cref{lemma_phase_4_unfitted_norms} \Cref{lemma_phase_4} states that for $\varepsilon$ sufficiently small one can choose $\varepsilon'$ sufficiently small that for $\alpha$ sufficiently small, $\alpha^{- \ell_j^\circ(\kappa'') + \zeta_j(\varepsilon, \varepsilon') + \varepsilon'} < \lVert \widetilde{w}_j(\kappa'') \rVert <  \alpha^{- \ell_j^\circ(\kappa'') - \zeta_j(\varepsilon, \varepsilon') - \varepsilon'}$. Simply choosing $\varepsilon_1$ and $\varepsilon'_1$ such that $| \zeta_j(\varepsilon_1, \varepsilon'_1) + \varepsilon' |  < \xi'' $ for all $j \in N_U^{(k+1)}$ gives the claim. 
\end{proof}

Finally, we are able to provide a proof of \Cref{th:s.t.s}. This proof consists of iteratively concluding uniform convergence on arbitrary closed intervals contained in $(t_{k}, t_{k+1})$ using \Cref{inductive_hypothesis}, and then using \Cref{lem:network_post_saddle} to apply the same reasoning to the next interval.

\begin{proof}[Proof of \Cref{th:s.t.s}] Let $\mathcal{S} \subset \mathbb{R} \setminus \{ t_0, t_1, \dots, t_p\}$ be a compact set. By compactness, there exists a cover of $\mathcal{S}$ by disjoint closed intervals $\left( [\kappa_k, \kappa'_k] \right)_{k=0}^{p}$ whose endpoints satisfy $\kappa_k, \kappa_k' \in (t_{k}, t_{k+1})$. We note from \Cref{induction_basis_case_lemma} that \Cref{inductive_hypothesis} holds for any $\kappa_0 \in \left( 0, \frac{1}{4} \sqrt{\frac{n}{2\mathcal{L}(0)}}\right)$. From \Cref{lem:network_pre_saddle}, we obtain that for such $\kappa_0$ and for any $\kappa_0' \in (\kappa_0, t_1) $, $\widetilde{\theta}$ converges uniformly to $\theta^\circ =  \mathbf{0}$. Furthermore, we have from \Cref{lem:network_post_saddle} that for any $\kappa_1 \in (t_1, t_2)$, \Cref{inductive_hypothesis} also holds. Iterating this argument, we can use \Cref{lem:network_pre_saddle} to conclude that there exists $\alpha^*_k>0$ such that $\widetilde{\theta}$ converges uniformly to $\theta^\circ$ on $[\kappa_k, \kappa_k']$ with rate $\alpha^{\delta_k}$ for some $\delta_k>0$, followed by \Cref{lem:network_post_saddle} to conclude that \Cref{inductive_hypothesis} holds at each $\kappa_{k+1} \in (t_{k+1}, t_{k+2})$. Since the number of saddles is finite, we may select $\alpha \in (0,\min_{k \in [p]} \alpha_k^*)$ so that each application of \Cref{lem:network_pre_saddle} and \Cref{lem:network_post_saddle} holds. This gives the claimed uniform convergence, the rate of which is $\alpha^{\delta^*}$ where $\delta^* = \min_{k \in [p]} \delta_k^*$.
\end{proof}

\section{Proof of \texorpdfstring{\Cref{thm:implicit_bias}}{Theorem \ref*{thm:implicit_bias}}}

We consider in this section the setting where we have only positive labels and condition all the probabilities on the fact that $a_i>0$ for all $i\in[m]$. The argument generalizes directly to the case of mixed signs, by treating positive and negative labels separately.

Recall that \Cref{thm:implicit_bias} provides a bound on the squared $\ell_2$-norm of $\theta_{\pred}$, which satisfies
\begin{align}
   \frac{1}{2} \lVert \theta_{\pred} \rVert^2 &= \sum_{j \in [m]} n \lVert D_j^{(p)} \rVert_2  \\
    &= \sum_{j \in [m]} \sqrt{\sum_{i \in S_j^{(p)}} y_i^2} \\
    &= \sum_{k \in [p]}\sqrt{\sum_{i \in S_U^{(k)} \cap S_{j_{*}^{(k)}}} y_i^2} .
\end{align}
\Cref{alg:s.t.s} relies on a specific ordering of the neurons, given by $(j_{\star}^{(1)}, \ldots, j_{\star}^{(p)})$, that depends on both the data and initialization scheme. Instead, our proof here considers a fixed ordering $\pi=(j_1, \ldots, j_m)$ that is independent of both data and initialization, and aims at bounding
\begin{equation}\label{eq:fixedpisum}
    \sum_{k \in [m]}\sqrt{\sum_{i \in S_U^{k} \cap S_{j_k}} y_i^2},
\end{equation}
where we define similarly the sets $S_U^{k}$ in a recursive way, as
\begin{gather}
    S_U^0 = [n], \quad S_U^{k+1} = S_U^k \setminus S_{j_k}\\
    S_{j} = \lbrace i \in [n] \mid A_{i, j} = 1 \rbrace,
\end{gather}
i.e., $S_U^k$ represent the data points that remain to be fitted after increment $k$, and $S_{j}$ refers to the data points with which neuron $j$ is positively correlated.

Note that when the ordering $\pi$ coincides with the one induced by $j_\star$ on its first $p$ elements, \cref{eq:fixedpisum} is exactly $\frac{1}{2} \lVert \theta_{\pred} \rVert^2$. We thus aim at bounding this quantity with high probability over all possible orderings by a union bound argument, which then directly translates to a bound of $\frac{1}{2} \lVert \theta_{\pred} \rVert^2$.

Note that when initializing the neurons according to \Cref{ass:init}, and conditioning on the fact that\footnote{In the remaining of the section, we omit the conditioning on $a_i>0$ for ease of exposition.} $a_{i}>0$, $(A_{i,j})_{j \in [m],\, i \in [n]}$ are independent $\mathrm{Bernoulli}\left(\frac{1}{2}\right)$ variables. 
For each $i \in [n]$, we define
\begin{equation}
    Y_i^{k}  = \iind{ i \in S_U^{k} },
\end{equation}
which is the random variable that is equal to~$1$ if the data point~$i$ remains to be fitted after increment~$k$. Obviously, we have
\begin{equation}
    \big| S_U^{k} \big| = \sum_{i=1}^{n} Y_i^{k}.
\end{equation}

For a fixed $\delta \in (0,\frac{1}{2})$, we define for each $k$ the events
\begin{align}
    G_k = \left\{ \bigg\vert \; \sum_{i=1}^{n} Y_i^{k} - \frac{1}{2} \sum_{i=1}^{n} Y_i^{k-1} \; \bigg\vert \leq \delta \sum_{i=1}^{n} Y_i^{k-1} \right\}.
\end{align}

The interpretation of the event~$G_k$---which has high probability, as shown below---is that at iteration~$k$ of our incremental algorithm (for the fixed ordering~$\pi$), the cardinality of the set of unfitted data has almost halved from the previous iteration.

Finally, denote for some $\rho\in(0,\frac{1}{2})$, to be fixed later,
\begin{equation}
    k^\star = \frac{(1-\rho) \log(n)}{\log(2)} .
\end{equation}
For the remainder of the proof, we assume $m \geq k^\star$ to simplify notation and exposition. The upper bounds extend straightforwardly to the case $m < k^\star$ by truncating the estimator considered here.
\begin{prop} \label{prop_half_bound_failure}
    Given any $0<\rho<\frac{1}{2}$ and a fixed deterministic ordering $\pi=(j_1,\ldots, j_m)$, for any $\delta\in\left(0,\frac{\rho \log(2)}{8(1-\rho)}\right)$ and $n$,
    \begin{equation}
        \mathbb{P} \left( \forall k\leq k^{\star}, G_k \right) \geq 1-3\log(n)\exp \left( - \frac{4}{3}n^{\rho/2} \delta^2 \right) .
    \end{equation}
\end{prop}
Note that as $\delta$~approaches zero here, the probability bound degrades to~$0$, but the events~$G_k$ get more restrictive, as $\delta$~controls the relative distance by which the increments $\sum_{i=1}^{n} Y_i^{k}$ can deviate from their conditional expectation.

\begin{proof}
By conditioning
\begin{align}
    \mathbb{P} \left( \forall k\leq k^{\star}, G_k \right)&=  \prod_{k=1}^{\lfloor k^{\star}\rfloor} \mathbb{P} \left( G_k \mid \forall l\leq k-1, G_{l} \right) ,
\end{align}
Our goal is thus for any $k\leq k^{\star}$, to lower bound $\mathbb{P} \left( G_k \mid \forall l\leq k-1, G_{l} \right)$.

Note that $Y_i^k = Y_i^{k-1} \cdot A_{i, j_k}$ by definition. Since the variables $(A_{i, j})_{i,j}$ are independent, $A_{i, j_k}$ is independent from $(A_{i,j_l})_{l\leq k-1}$ for a fixed ordering\footnote{The choice of a fixed ordering is crucial here. If we were to consider the ordering $j_\star$ induced by \Cref{alg:s.t.s}, this independence argument would not hold anymore.} $\pi$. Consequently, 
\begin{equation}
    \sum_{i=1}^{n} Y_i^{k} \; \Bigg\vert \; (A_{i,j_l})_{\substack{i\in[n]\\ l\leq k-1}}\sim \mathrm{Bin}\left( \; \sum_{i=1}^{n} Y_i^{k-1} \; , \frac{1}{2} \right) .
\end{equation}

Using the multiplicative Chernoff bound \citep[see, e.g.,][Corollary 4.6]{mitzenmacher2017probability}, we have for any $\delta\in(0,\frac{1}{2})$:
\begin{align}
    \mathbb{P}\left( \bigg\vert \; \sum_{i=1}^{n} Y_i^{k} - \frac{1}{2} \sum_{i=1}^{n} Y_i^{k-1} \; \bigg\vert \leq \delta \sum_{i=1}^{n} Y_i^{k-1} \;\Bigg\vert \; (A_{i,j_l})_{\substack{i\in[n]\\ l\leq k-1}} \right) &\geq 1 - 2 \exp \left( - \frac{4}{3} \cdot \frac{1}{2} \sum_{i=1}^{n} Y_i^{k-1} \delta^2 \right)
\end{align}
Conditioned on the event $\lbrace \forall l\leq k-1, G_l \rbrace$, we have that
\begin{equation}
    \sum_{i=1}^{n} Y_i^{k-1} \geq n \left( \frac{1}{2} - \delta \right)^{k-1},
\end{equation}
so that the previous inequality becomes, when conditioning on this event,

    \begin{align}
         \mathbb{P}\left( \; G_k \; \mid \forall l\leq k-1, G_l \right)
    &\geq 1 - 2 \exp \left( - \frac{4}{3} \cdot \frac{n}{2} \left( \frac{1}{2} - \delta \right)^{k-1} \delta^2 \right) \\
    &\geq 1 - 2 \exp \left( - \frac{4}{3} n \left( \frac{1}{2} - \delta \right)^{k^\star} \delta^2 \right).
\end{align}

Using that for any $0<x<\frac{1}{2},$ $(1-2x)^{k^\star} \geq 1-2k^{\star}x,$ we get  
\begin{align}
\mathbb{P} \left( \forall k\leq k^{\star}, G_k \right) & = 
    \prod_{k=1}^{\lfloor k^{\star}\rfloor} \mathbb{P} \left( G_k \mid \forall l\leq k-1, G_l\right),\\
    &\geq  1 - 2 k^{\star} \exp \left( - \frac{4}{3} n \left( \frac{1}{2} - \delta \right)^{k^{\star}} \delta^2 \right) .\label{failure_bound}
\end{align}

Now inserting $k^{\star}=\frac{(1-\rho) \log(n)}{\log(2)}$, we obtain  that
\begin{align}
n\left(\frac{1}{2}-\delta\right)^{k^{\star}}&= n 2^{-k^{\star}} \left(1 - 2 \delta \right)^{k^{\star}} \\
    &= n^{1 + \rho - 1 + \frac{(1-\rho) }{\log(2)} \log(1 - 2 \delta)} \\
    &> n^{\rho/2}, \label{exponent_bound}
\end{align}
where we here used the upper bound on~$\delta$ that gives $\frac{(1-\rho) }{\log(2)} \log(1 - 2 \delta)\geq -\frac{\rho}{2}$. It will be useful for later to note that  similar logic allows us to obtain the following inequality, given $\lbrace \forall k\leq k^\star, G_k\rbrace$:
\begin{align}
&n\left(\frac{1}{2}-\delta\right)^{k}\leq\sum_{i=1}^{n} Y_i^{k}  \leq n\left(\frac{1}{2}+\delta\right)^{k} \qquad \text{for any }k\leq k^\star\label{exponent_upper_bound1}\\
\text{and}\qquad &   n^{\rho/2}< \sum_{i=1}^{n} Y_i^{\lfloor k^{\star}\rfloor}  < 2\cdot n^{2\rho} \label{exponent_upper_bound2}.
\end{align}

Putting \cref{exponent_bound} back into \cref{failure_bound}, we obtain

\begin{align}
    \mathbb{P} \left( \forall k\leq k^{\star}, G_k \right) &\geq 1- 2\frac{ (1-\rho) \log(n)}{\log(2)} \exp \left( - \frac{4}{3} n^{\rho/2} \delta^2 \right) \\
    &>1-3\log(n)\exp \left( - \frac{4}{3}n^{\rho/2} \delta^2 \right).
    \qedhere
\end{align}

\end{proof}

\cref{prop_half_bound_failure} demonstrates that with high probability, the cardinality of the set $S_U^{k}$ decreases by roughly one half for each of the first $k^{\star}$ iterations. Note that we here worked with a fixed, arbitrary ordering $\pi$. However, \Cref{alg:s.t.s} selects a specific ordering $\pi$ that depends on the data. \Cref{cor_all_failure} provides a union bound over all the possible orderings, guaranteeing that the event $\lbrace \forall k\leq k^{\star}, G_k\rbrace$ holds with high probability for the ordering $\pi$ returned by \Cref{alg:s.t.s}.

\begin{cor} \label{cor_all_failure}
    Given any $0<\rho<\frac{1}{2}$, $\delta\in\left(0,\frac{\rho \log(2)}{8(1-\rho)}\right)$ and $n\in\mathbb{N}$,

    \begin{equation}
        \mathbb{P} \left( \exists \, \pi , k\leq k^{\star}, \neg G_k\right) \leq 3\log(n)  \exp \left( \frac{ \log(n)}{\log(2)} \log \left( m \right) - \frac{4}{3} n^{\rho/2} \delta^2 \right).
    \end{equation}
\end{cor}

\begin{proof}
\Cref{prop_half_bound_failure} is established for an arbitrary fixed permutation $\pi$. We then apply a union bound over all injective mappings from $[k^{\star}]$ to $[m]$, whose total number is at most $m^{k^\star}$. Thus we directly obtain from \Cref{prop_half_bound_failure}  the bound


\begin{align}
     \mathbb{P} \left( \exists \, \pi , k\leq k^{\star}, \neg G_k\right)&\leq m^{k^\star} \cdot 3\log(n) \exp \left( - \frac{4}{3} n^{\rho/2} \delta^2 \right) \\
    &= 3\log(n) \exp \left( k^{\star} \log(m) - \frac{4}{3}  n^{\rho/2} \delta^2 \right) \\
    &= 3\log(n)  \exp \left( \frac{ (1-\rho) \log(n)}{\log(2)} \log \left( m \right) - \frac{4}{3} n^{\rho/2} \delta^2 \right). 
    \qedhere
\end{align}


\end{proof}

In the following, we work under the event $\lbrace \forall k\leq k^{\star}, G_k\rbrace$ for the considered permutation $\pi$ and bound the squared $\ell_2$-norm of the estimator returned by \Cref{alg:s.t.s} by splitting it into two terms: that induced by the first $k^{\star}$ iterations of the algorithm, and the remaining ones. 

\begin{prop} \label{prop_chi_bound_failure}
    If for all $i\in[n]$, $|y_i|\leq y_{\max}$, then 
\begin{equation}
\forall k\leq k^{\star}, G_k\implies  \sum_{k\leq k^{\star}} \sqrt{\sum_{i \in S_U^{k} \cap S_{j_k}} y_i^2} \leq \frac{\sqrt{\frac{1}{2} + \delta}}{1-\sqrt{\frac{1}{2} +\delta}} y_{\max} \sqrt{n} .
    \end{equation}
\end{prop}

\begin{proof}
    The bound $|y_i|\leq M$ directly implies that 
    \begin{align}
        \sum_{k \leq k^{\star}} \sqrt{\sum_{i \in S_U^{k} \cap S_{j_k}} y_i^2} \leq  y_{\max} \sum_{k \leq k^{\star}} \sqrt{\big|S_U^{k} \cap S_{j_k}\big|}.
    \end{align}
    Note that $\lbrace \forall k\leq k^{\star}, G_k\rbrace$ implies that for any $k\leq k^{\star}$, $|S_U^{k} \cap S_{j_k}| \leq n \left( \frac{1}{2} + \delta \right)^k$ thanks to \cref{exponent_upper_bound1}. So that if it holds,
\begin{align}
    \sum_{k \leq k^{\star}} \sqrt{\sum_{i \in S_U^{k} \cap S_{j_k}} y_i^2} & \leq  y_{\max} \sqrt{n} \sum_{k =1}^{\lfloor k^{\star}\rfloor}\left( \frac{1}{2} + \delta \right)^{k/2}\\
    & \leq y_{\max} \sqrt{n} \cdot \frac{\sqrt{\frac{1}{2} + \delta}}{1-\sqrt{\frac{1}{2} +\delta}}.
    \qedhere
\end{align}
\end{proof}

Finally, we bound the squared $\ell_2$-norm of the remaining increments. 
\begin{prop} \label{prop_tail_temrs}
    If for all $i\in[n]$, $|y_i|\leq y_{\max}$, then 
    \begin{equation}
     \forall k\leq k^{\star}, G_k \implies \sum_{k > k^{\star}} \sqrt{\sum_{k \in S_U^{k} \cap S_{j_k}} y_i^2} < 2 y_{\max} n^{2\rho} 
    \end{equation}
\end{prop}

\begin{proof}
By subadditivity of the square root, we simply have
\begin{align}
    \sum_{k > k^{\star}} \sqrt{\sum_{i \in S_U^{k} \cap S_{j_k}} y_i^2}  &\leq \sum_{i \in S_U^{\lfloor k^{\star} \rfloor}} | y_i | \\
    &\leq  y_{\max} |S_U^{\lfloor k^{\star}\rfloor}|.
\end{align}
Now, thanks to \cref{exponent_upper_bound2}, we recall that $ \lbrace \forall k\leq k^{\star}, G_k \rbrace$ implies that $|S_U^{\lfloor k^{\star}\rfloor}|<2\cdot n^{2\rho}$, which concludes the proof.
\end{proof}

Combining \cref{prop_chi_bound_failure,prop_tail_temrs,cor_all_failure}, we conclude taking $\rho=1/4$ and $\delta=\frac{1}{40}$ that for any $n$, 
\begin{equation}
    \sum_{j \in [m]} n \lVert D_j^{(p)} \rVert  \leq 5 y_{\max}\sqrt{n} 
\end{equation}
 with probability at least $1 - 3\log(n)  \exp \left( \frac{ \log(n)}{\log(2)} \log \left( m \right) - \frac{4}{3\cdot 40^2} n^{1/8}  \right)$. 


This concludes the proof when considering only positive labels and neurons. \Cref{thm:implicit_bias} directly extends to the general case, by separating positive and negative cases.

\section{Auxiliary results}

We provide in this section auxiliary results that are useful for our analysis, as well as \texorpdfstring{\Cref{pr:alg_assumptions}}{Theorem \ref*{pr:alg_assumptions}}. 

\subsection{Proof of \texorpdfstring{\Cref{pr:alg_assumptions}}{Theorem \ref*{pr:alg_assumptions}}}

\begin{subproof}[Proof of \Cref{pr:alg_assumptions_1}]
 The proof is a simple union bound. Indeed, observe that 1) $(A_{i,j})_{j\in[m]}$ are independent i.i.d. Bernoulli$\left(\frac{1}{4}\right)$ variables. Then, the probability for a row $i$ to be zero is exactly $\left(\frac{3}{4}\right)^m$, so that the probability that any row is zero is smaller than $n\left(\frac{3}{4}\right)^m$.

 2) Conditionally on $a_j$, $(A_{i,j})_{i\in[n]}$ are independent variables, each of them being either the deterministic variable $0$ if $i\not\in I_{s_j}$, or a Bernoulli$\left(\frac{1}{2}\right)$ $i\in I_{s_j}$. Since $|I_{s_j}|\geq \min(n_-,n_+)$ almost surely by definition, we then have that a column $j$ is zero with probability at most $\left(\frac{1}{2}\right)^{\min(n_-,n_+)}$. So that the probability for any column to be zero is bounded by $m\left(\frac{1}{2}\right)^{\min(n_-,n_+)}$.

 3) Note that two non-zero columns $j$ and $j'$ are equal if and only if both $\sgn(a_j)=\sgn(a_{j'})$ and $\mathds{1}(w_j^\top x_i) = \mathds{1}(w_{j'}^\top x_i)$ for all $i\in I_{s_j}$. All these variables are independent Bernoulli$\left(\frac{1}{2}\right)$, so that a union bound argument again yields that the probability for two non-zero columns to be equal is at most $\frac{m(m-1)}{2}\left(\frac{1}{2}\right)^{\min(n_-,n_+)+1}$.

 A union bound on these three events and summing the different terms then yields \Cref{pr:alg_assumptions_1} of \Cref{pr:alg_assumptions}.
\end{subproof}

\begin{subproof}[Proof of \Cref{pr:alg_assumptions_2}] 
Similarly to the proof of \Cref{thm:implicit_bias}, we here fix an arbitrary ordering $\pi = (j_1, \ldots, j_m)$, and define an alternative \Cref{alg:s.t.s} with the fixed ordering $\pi$, i.e.,
\begin{gather}
    S_U^0 = [n], \quad S_U^{k+1} = S_U^k \setminus S_{j_k}\\
    S_{j} = \lbrace i \in [n] \mid A_{i, j} = 1 \rbrace.
\end{gather}
Similarly, we also define the increment times as $t'_0=0$, and for any $k\geq 0$:
\begin{gather}
    t'_{k+1} = \max(t'_{k} - \frac{\ell'_{j_{k+1}}(t'_k)}{\|D_{j_{k+1}}^k\|}, t'_k)\\
    \text{where }\quad D_j^k = \frac{1}{n}\sum_{i\in S_U^k \cap S_{j}} y_i x_i, \\
    \ell'_j(0)=-1 \text{ and }\quad \ell'_{j}(t) = \ell'_j(t'_k) + (t-t'_k)\|D_j^k\| \text{ for any } t\in[t'_k, t'_{k+1}],
\end{gather}
where $\frac{\ell'_{j_{k+1}}(t'_k)}{\|D_{j_{k+1}}^k\|}=-\infty$ when  $\|D_{j_{k+1}}^k\|=0$ by convention. Note that when $\pi$ coincides with $j_\star$ on its first $p$ elements, we have the equality $(t_0, \ldots, t_p)=(t'_0, t'_p)$, as well as $\ell'_j(t) = \ell_j^\circ(t)$ on $[0,t_p]$ and $D_j^k=D_j^{(k)}$ for $k\leq p$. 

From there, we can see all $t'_k$, $\ell'_j(t)$ and $\|D_j^k\|$ as functions of $(y_1,\ldots,y_n)$. Consider two neurons $j,j^{\dagger}$. Since we assume \cref{ass:A.j.star}~\ref{ass:A.j.star:A} holds, the columns of the masking matrix $A_{\cdot,j}$ and $A_{\cdot,j^{\dagger}}$ are both non-zero and distinct; and also consider $k$ such that $j_l\not\in\lbrace j, j^{\dagger}\rbrace$ for all $l\leq k$, i.e., both neurons $j$ and $j^{\dagger}$ have not been selected by the alternative algorithm in the first $k$ increments. From there, we want to show that the following system of inequalities
\begin{gather}\label{eq:equalityprop2}
     -\frac{\ell'_{j}(t'_k)}{\|D_{j}^k\|} = -\frac{\ell'_{j^{\dagger}}(t'_k)}{\|D_{j^{\dagger}}^k\|},\\
    \|D_{j}^k\| > 0 \quad \text{and} \quad \max(\ell'_{j^{\dagger}}(t'_k),\ell'_{j}(t'_k)) < 0 \label{eq:ineqconstraintprop2}
\end{gather}
 only holds on a measure zero set for the values of $(y_1, \ldots, y_n)$. Indeed, since the two masking columns are different, there exists some $i\in[n]$ such that, w.l.o.g., $A_{i,j}=1$ and $A_{i,j^{\dagger}}=0$. We then have two cases.

Either $i\in S_U^k$, i.e., $i$ remains to be fitted at iteration $k$. In that case, note that all $(t'_0, \ldots, t'_k)$, $\ell'_{j^{\dagger}}$ and $(D_{j^{\dagger}}^0, \ldots, D_{j^{\dagger}}^k)$ do not depend on the value $y_i$. On the other hand, $-\frac{\ell'_{j}(t'_k)}{\|D_{j}^k\|}$ decreases (strictly) with~$|y_i|$, when \cref{eq:ineqconstraintprop2} holds. Indeed, note that
\begin{gather}
    -\ell'_{j}(t'_k) = 1 - \sum_{l<k}(t'_{l+1}-t'_l)\|D_j^l\|\quad
    \text{and} \quad  \|D_j^l\| = \sqrt{\frac{1}{n}\sum_{i'\in S_U^l\cap S_j} y_{i'}^2}.
\end{gather}
so that $-\ell'_j(t'_k)$ decreases with $|y_i|$ and is positive, while $\|D_j^k\|$ increases with it. In consequence, $-\frac{\ell'_{j}(t'_k)}{\|D_{j}^k\|}$ is a decreasing function of~$|y_i|$, so that for fixed values $(y_{i'})_{i'\neq i}$, there is at most one value of~$|y_i|$ such that the system of \cref{eq:equalityprop2,eq:ineqconstraintprop2} holds. This implies over all labels that the system of \cref{eq:equalityprop2,eq:ineqconstraintprop2} can only be satisfied on a zero measure set.

The other possible case is when for all $i$ such that $A_{i,j}\neq A_{i,j^{\dagger}}$, $i\not\in S_U^k$. In other words, this means that $S_U^k\cap S_j = S_U^k\cap S_{j^{\dagger}}$. By definition of the quantities of interest, if the system of \cref{eq:equalityprop2,eq:ineqconstraintprop2} holds for $k$, the following system also holds 
\begin{gather}\label{eq:equalitybisprop2}
     \ell'_{j}(t'_{\overline{k}}) = \ell'_{j^{\dagger}}(t'_{\overline{k}}),\\
    \|D_{j}^{\overline{k}}\| > 0 \quad \text{and} \quad \max(\ell'_{j^{\dagger}}(t'_{\overline{k}}),\ell'_{j}(t'_{\overline{k}}))< 0 ,\label{eq:constraintbisprop2}
\end{gather}
where $\overline{k}\leq k$ is the largest integer such that $S_U^{\overline{k}} \cap S_j \neq S_U^{\overline{k}}\cap S_{j^{\dagger}}$---indeed, for any $\overline{k}<l\leq k$, we would have $D_j^l = D_{j^{\dagger}}^l$. Now note that we could have chosen $i$ such that $A_{i,j}\neq A_{i,j^{\dagger}}$ and it is the last such $i$ to be fitted by the algorithm. In that case, we necessarily have $i\in S_U^{\overline{k}}$ and we are in a case very similar to the one treated above. We can then show with a similar study---which is even simpler here, as we only have to consider the numerator of the previous fraction---that the system of \cref{eq:equalitybisprop2,eq:constraintbisprop2} only holds for a zero measure set.

We thus have proven that the system of \cref{eq:equalityprop2,eq:ineqconstraintprop2} only holds on a zero measure set for the values of $(y_1, \ldots, y_n)$, for a fixed ordering~$\pi$, iteration~$k$, and neurons $j,j^{\dagger}$ not yet fitted at iteration~$k$. By doing a finite union, it remains true that this system is only satisfied on a zero measure set for all choices of ordering, iterations and neuron couples. In particular, it holds for some ordering that coincides with the one returned by \Cref{alg:s.t.s} on the first~$p$ iterations, i.e., $(j_\star^{(1)}, \ldots, j_\star^{(p)})$. For this specific choice of ordering, we observe that the case of a non-unique minimizer of the loop in \Cref{alg:s.t.s} corresponds to a couple satisfying the system of \cref{eq:equalityprop2,eq:ineqconstraintprop2}, such that both neurons have not fitted yet, which concludes the proof of  \Cref{pr:alg_assumptions_2} of \Cref{pr:alg_assumptions}.
\end{subproof}

\subsection{Other auxiliary results}

The following result estimates the normalized component of $\widetilde{w}_j$ in the direction $x_i$ under the assumption that this neuron is aligned in direction $D_j^{(k)}$.  

\begin{prop} \label{alignment_observation}
    For any $j \in [m]$ and $k \in [p]$, if $\widetilde{\overline{w}}_j^\top \overline{D}_j^{(k)} > 1 - \alpha$, then for all $i \in [n]$,
    
    \begin{equation}
        \left\vert \widetilde{\overline{w}}_j^\top x_i - \frac{\iind{y_i \in S_j^{(k)}} y_i}{n \left\lVert D_j^{(k)} \right\rVert} \right\vert < \sqrt{2}\alpha^{1/2}. 
    \end{equation}
\end{prop}

\begin{proof}
    Let $u=\sum_i u_ix_i$ and $v=\sum_i v_ix_i$ be the representations of two unit vectors in the orthonormal basis given by $(x_i)_{i=1}^d$ such that $u^\top v > 1 - \alpha$. We have for any $i_0$
    \begin{align}
        ( u_{i_0} - v_{i_0} )^2 &\leq \sum_{i} ( u_i - v_i )^2 \\
        &= \sum_{i} \left( u_i^2 + v_i^2 - 2 u_i v_i \right) \\
        &< 2 \alpha.
    \end{align}
    Taking square roots yields $ | u_{i_0} - v_{i_0} | < \sqrt{2} \alpha^{1/2}$. Recalling that 
    
\begin{equation}
    \overline{D}_j^{(k)} = \frac{1}{n \left\lVert D_j^{(k)} \right\rVert} \sum_{i \in [n]} \iind{y_i \in S_j^{(k)}} y_i x_i
\end{equation} 
gives the claim. 
\end{proof}

The following result applies the preceding proposition to deduce an inequality which will prove useful in our analysis.

\begin{cor} \label{alignment_corollary}
    If $\widetilde{\overline{w}}_j^\top \overline{D}_j^{(k)} > 1 - \alpha$, then 
    \begin{equation}
        \left\vert x_i^\top D_j^{(k)} - (x^\top_i \widetilde{\overline{w}}_j)(\widetilde{\overline{w}}^\top_j D_j^{(k)}) \right\vert < 2 \left\lVert D_j^{(k)} \right\rVert \alpha^{1/2}
    \end{equation}
    for $\alpha$ sufficiently small.
\end{cor}

\begin{proof}
    We simply plug in the bounds from \cref{alignment_observation} to obtain 
    \begin{align}
        -\sqrt{2} \left\lVert D_j^{(k)} \right\rVert \alpha^{1/2} &< x_i^\top D_j^{(k)} - (x^\top_i \widetilde{\overline{w}}_j)(\widetilde{\overline{w}}^\top_j D_j^{(k)}) \\
        &< \frac{y_i}{n}\alpha + \sqrt{2} \left\lVert D_j^{(k)} \right\rVert \left( \alpha^{1/2} - \alpha^{3/2} \right),
    \end{align}
    from which the claim follows.
\end{proof}

\begin{lem} \label{diff_ineq_lemma}
    Suppose that the following hold:
    \begin{enumerate}
        \item $F \in \mathcal{C}^1([T_1, T_2])$ and $F'$ is Lipschitz continuous. 
        \item $F \geq 0$ and $F(T_1)=0$.
        \item $F''(t) \leq pF(t) + q$ for almost every $t \in [T_1, T_2]$, where $p, q > 0$.
        
    \end{enumerate}
    
    Then for every $t \in [T_1, T_2]$, 
    
    \begin{equation}
        F'(t) \leq \left( \frac{q}{\sqrt{p}} + F'(T_1) \right) \exp(\sqrt{p}(t-T_1)) .    
    \end{equation}
    
\end{lem}

\begin{proof} Let $G = F + q/p$. Then $G$ obeys for almost every $t \in [T_1, T_2]$ the differential inequality $G'' < p G$. We compute using this inequality that 

\begin{equation}
    \frac{\df}{\df t} \left( G' + \sqrt{p}G \right) \leq pG + \sqrt{p} G' = \sqrt{p} \left( \sqrt{p} G + G'\right).
\end{equation}

Since $F'$ is Lipschitz continuous, both $F'$ and $F$ are in fact absolutely continuous. Therefore, $\sqrt{p}G + G'$ is absolutely continuous. We may apply Gr{\"o}nwall's inequality to see that

\begin{equation}
    \sqrt{p} G + G' \leq \exp \left( \sqrt{p}(t-T_1) \right) \left( \sqrt{p} G(T_1) + G'(T_1) \right) .
\end{equation}

Noting that $G \geq F$, the imposition of the condition that $F \geq 0$ gives

\begin{equation}
    G' \leq \exp \left( \sqrt{p}(t-T_1) \right) \left( \sqrt{p} G(T_1) + G'(T_1) \right) .
\end{equation}

Substituting $F$ into the above and using the initial condition $F(T_1) = 0$, we get 

\begin{equation}
    F'(t) \leq \left( \frac{q}{\sqrt{p}} + F'(T_1) \right) \exp(\sqrt{p}(t-T_1)). 
\end{equation}

as claimed.
\end{proof}

\end{document}